\documentclass[journal]{IEEEtran}

\usepackage{graphicx}
\usepackage{amsmath}
\usepackage{amssymb}
\usepackage{amsfonts}
\usepackage{subfigure}
\usepackage[ruled,vlined,linesnumbered]{algorithm2e}

\SetCommentSty{mycommfont}
\usepackage{url}
\usepackage{cite}
\usepackage{amsthm}
\newtheorem{lemma}{Lemma}
\newtheorem{theorem}{Theorem}
\newtheorem{assumption}{Assumption}

\usepackage{float}

\newcommand{\Identity}{{\rm I\kern-.2em l}}
\newcommand{\Expect}{\mathbb{E}}

\usepackage{xcolor}
\usepackage{hyperref}

\allowdisplaybreaks

\begin{document}

\title{Model Pruning Enables Efficient Federated Learning on Edge Devices}

\author{Yuang~Jiang,
        Shiqiang~Wang,
        V\'ictor~Valls,
        Bong~Jun~Ko,
        Wei-Han~Lee,
        Kin~K.~Leung,
        Leandros~Tassiulas
\thanks{
    Y. Jiang, V. Valls, and L. Tassiulas are with Yale University, New Haven, CT, USA.
    
    S. Wang and W.-H. Lee are with IBM T. J. Watson Research Center, Yorktown Heights, NY, USA.
    
    B. J. Ko was with Stanford Institute for Human-Centered Artificial Intelligence (HAI), Stanford, CA, USA when contributing to this work.
    
    K. K. Leung is with Imperial College London, UK.
    
    Contact authors: Y. Jiang (\mbox{yuang.jiang@yale.edu}) and S. Wang (\mbox{wangshiq@us.ibm.com})
    
    Accepted for publication in IEEE Transactions on Neural Networks and Learning Systems (TNNLS)
}
}

\maketitle

\begin{abstract}
Federated learning (FL) allows model training from local data collected by edge/mobile devices while preserving data privacy, which has wide applicability to image and vision applications. A challenge is that client devices in FL usually have much more limited computation and communication resources compared to servers in a datacenter. To overcome this challenge, we propose \emph{PruneFL} -- a novel FL approach with adaptive and distributed parameter pruning, which adapts the model size during FL to reduce both communication and computation overhead and minimize the overall training time, while maintaining a similar accuracy as the original model. PruneFL includes initial pruning at a selected client and further pruning as part of the FL process. The model size is adapted during this process, which includes maximizing the approximate empirical risk reduction divided by the time of one FL round. Our experiments with various datasets on edge devices (e.g., Raspberry Pi) show that: (i) we significantly reduce the training time compared to conventional FL and various other pruning-based methods; (ii) the pruned model with automatically determined size converges to an accuracy that is very similar to the original model, and it is also a lottery ticket of the original model.
\end{abstract}


\IEEEpeerreviewmaketitle

\section{Introduction}\label{sec:intro}

\IEEEPARstart{T}{he} past decade has seen a rapid development of machine learning algorithms and applications, particularly in the area of deep neural networks (DNNs)~\cite{Goodfellow-et-al-2016}. However, a huge volume of training data is usually required to train accurate models for complex tasks such as image classification and computer vision. Due to limits in data privacy regulations and communication bandwidth, it is usually infeasible to transmit and store all training data at a central location. To address this problem, \emph{federated learning} (FL) has emerged as a promising approach of distributed model training from decentralized data~\cite{mcmahan2017communication,li2019federated,park2019wireless,yang2019federated,kairouz2019advances}. In a typical FL system, data is collected by client devices (e.g., cameras) at the network edge; the training process includes local model updates using each client's own data and the fusion of all clients' models typically through a server. In this way, the raw data remains local in clients.

Client devices in FL are usually much more resource-constrained than server machines in a datacenter, in terms of computation power, communication bandwidth, memory and storage size, etc. Training DNNs that can include over millions of parameters (weights) on such resource-limited edge devices can take prohibitively long and consume a large amount of energy.
Therefore, a natural question is: \emph{how can we perform FL efficiently so that a model is trained within a reasonable amount of time and energy?}

Some progress has been made towards this direction recently using model/gradient compression techniques, where instead of training the original model with full parameter vector, either a small model is extracted from the original model for training or a compressed parameter vector (or its gradient) is transmitted in the fusion stage~\cite{konevcny2016federated,caldas2018expanding,han2020adaptive,xu2019elfish}. However, the former approach may reduce the accuracy of the final model in undesirable ways, whereas the latter approach only reduces the communication overhead and does not generate a small model for efficient computation. Furthermore, how to adapt the compressed model size for the most efficient training remains a largely unexplored area, which is a challenging problem due to unpredictable training dynamics and the need of obtaining a good solution in a short time with minimal overhead.

To overcome these problems, we propose a new FL paradigm called \emph{PruneFL}, which includes \emph{adaptive and distributed parameter pruning} as part of the FL procedure.  
We make the following key contributions.

\textbf{Distributed pruning.}  PruneFL includes initial pruning at a selected client followed by further distributed pruning that is intertwined with the standard FL procedure. Our experimental results show that this method outperforms alternative approaches that either  prunes at a single client only or directly involves multiple FL clients for pruning, especially when the clients have heterogeneous data statistics and computational power.

\textbf{Adaptive pruning.} PruneFL continuously ``tracks'' a model that is small enough for efficient transmission and computation with low memory footprint, while maintaining useful connections and their parameters so that the model converges to a similar accuracy as the original model. The importance of model parameters evolves during training, so our method continuously updates which parameters to keep and the corresponding model size. The update follows an objective of minimizing the time of reaching intermediate and final loss values. Each FL round operates on a small pruned model, which is efficient. A small model is also obtained at any time during and after the FL process for efficient inference on edge devices, which is a lottery ticket~\cite{frankle2018the} of the original model as we show experimentally. 

\textbf{Implementation.} We implement FL with model pruning on real edge devices, where we extend a deep learning framework to support efficient sparse matrix computation. Our code is available at: \url{https://github.com/jiangyuang/PruneFL}

\section{Related Work}\label{sec:relatedwork}

\textbf{Neural network pruning.}
To reduce the complexity of neural network models, different ways of parameter pruning were proposed in the literature. Early work considered approximation using second-order Taylor expansion~\cite{lecun1990optimal}. However, the computation of Hessian matrix has high complexity which is infeasible for modern DNNs.  In recent years, magnitude-based pruning has become popular~\cite{han2015learning}, where parameters with small enough magnitudes are removed from the network. A finding that suggests a network that is pruned by magnitude consists of an optimal substructure of the original network, known as ``lottery ticket hypothesis'', was presented in~\cite{frankle2018the,morcos2019one}. It shows that directly training the pruned network can reach a similar accuracy as pruning a pre-trained original network. 

In addition to the above approaches that train until convergence before the next pruning step, there are iterative pruning methods where the model is pruned after every few steps of training~\cite{narang2017exploring,zhu2017prune}. There are also one-shot pruning approaches including SynFlow~\cite{tanaka2020synflow} that prunes the model at model initialization (before training), and SNIP~\cite{lee2018snip} that prunes the model using first training round's gradient information. A dynamic pruning approach that allows the network to grow and shrink during training was proposed in~\cite{Lin2020Dynamic}. Besides these unstructured pruning methods, structured pruning was also studied~\cite{anwar2017structured}, which, however, often requires specific network architectures and does not conform to the lottery ticket hypothesis. The lottery ticket is useful for retraining a pruned model on a different yet similar dataset~\cite{morcos2019one}. The use of pruning for efficient model training was discussed in~\cite{lym2019prunetrain}, where the optimal choice of pruning rate (or final model size) remained unstudied.

These existing pruning techniques consider the centralized setting with full access to training data, which is fundamentally different from our PruneFL that works with decentralized datasets at local clients. Furthermore, the automatic adaptation of model size has not been studied before.

\textbf{Efficient federated learning.}
The first FL method is known as federated averaging (\mbox{FedAvg})~\cite{mcmahan2017communication}, where each ``round'' of training includes multiple local gradient computation steps on each client's local data, followed by a parameter averaging step through a server. This method can be shown to converge in various settings including when the data at different clients are non-identically distributed (non-IID)~\cite{yu2019parallel,Li2020On,wang2020local}. 

To improve the communication efficiency of FL, methods for optimizing the communication frequency were studied~\cite{wang2019MLSys,wang2018edge,karimireddy2019scaffold}. An approach of parameter averaging using structured, sketched, and quantized updates was introduced in~\cite{konevcny2016federated}, which belongs to the broader area of gradient compression/sparsification~\cite{pmlr-v97-karimireddy19a,alistarh2018convergence,shi2019convergence,peng2018linear,du2018efficient,li2020lotteryfl}. These techniques usually consider a fixed degree of sparsity or compression that needs to be configured as a hyperparameter. An online learning approach that determines a near-optimal gradient sparsity was proposed in~\cite{han2020adaptive}, which includes exploration steps that may slow down the training initially. This body of work does not address computation efficiency.

To reduce both communication and computation costs, efficient FL techniques using lossy compression and dropout were developed \cite{caldas2018expanding,xu2019elfish}, where the final model still has the original size and hence providing no benefit for efficient inference after the model is trained. 
Moreover, because the main goal of pruning is to remove less important weights from the model, it is orthogonal to other acceleration methods such as quantization~\cite{quantization_gupta2015deep}, low-rank decomposition~\cite{jaderberg2014lowrank}, etc., and pruning can be applied together with these other methods. In addition, since our approach considers acceleration for both training and inference, methods that accelerate inference only, e.g., knowledge distillation~\cite{hinton2015distilling}, runtime neural pruning~\cite{lin2017runtime},
and DNN partitioning and offloading~\cite{hu2019dnnpartition}, do not serve our purpose. There are other distributed training methods, such as split learning~\cite{vepakomma2018splitlearning}, which are beyond our scope since we focus on FL in this paper.

Furthermore, most existing studies on FL are based on simulation. Only a few recent papers considered implementation on real embedded devices~\cite{wang2018edge,xu2019elfish}, but they do not include parameter pruning.

\textbf{Novelty of our work.} The uniqueness of PruneFL is that we jointly address communication and computation efficiency for both training and inference phases, by extending FedAvg with minimal extra overhead. Our two-stage distributed pruning method is designed to address both data (statistical) and device (system) heterogeneity including non-IID data distribution. 
Our adaptive pruning method is uniquely based on gradient information, which does not require sharing clients' local data, so that existing privacy preservation and secure aggregation~\cite{bonawitz2017practical} methods for FL can be directly applied to the gradient. Thus, our approach does not introduce extra privacy concerns.

\textbf{Roadmap.} The remainder of this paper is organized as follows. Section~\ref{sec:preliminary} provides preliminaries of FL and model pruning. Section~\ref{sec:proposed_approach} presents the proposed PruneFL approach and its analysis. Implementation challenges are discussed in Section~\ref{sec:implementation}. Section~\ref{sec:experiments} presents the experimental setup and results. Section~\ref{sec:conclusion} draws conclusion.

\begin{figure*}
    \centering
    \includegraphics[scale=0.125]{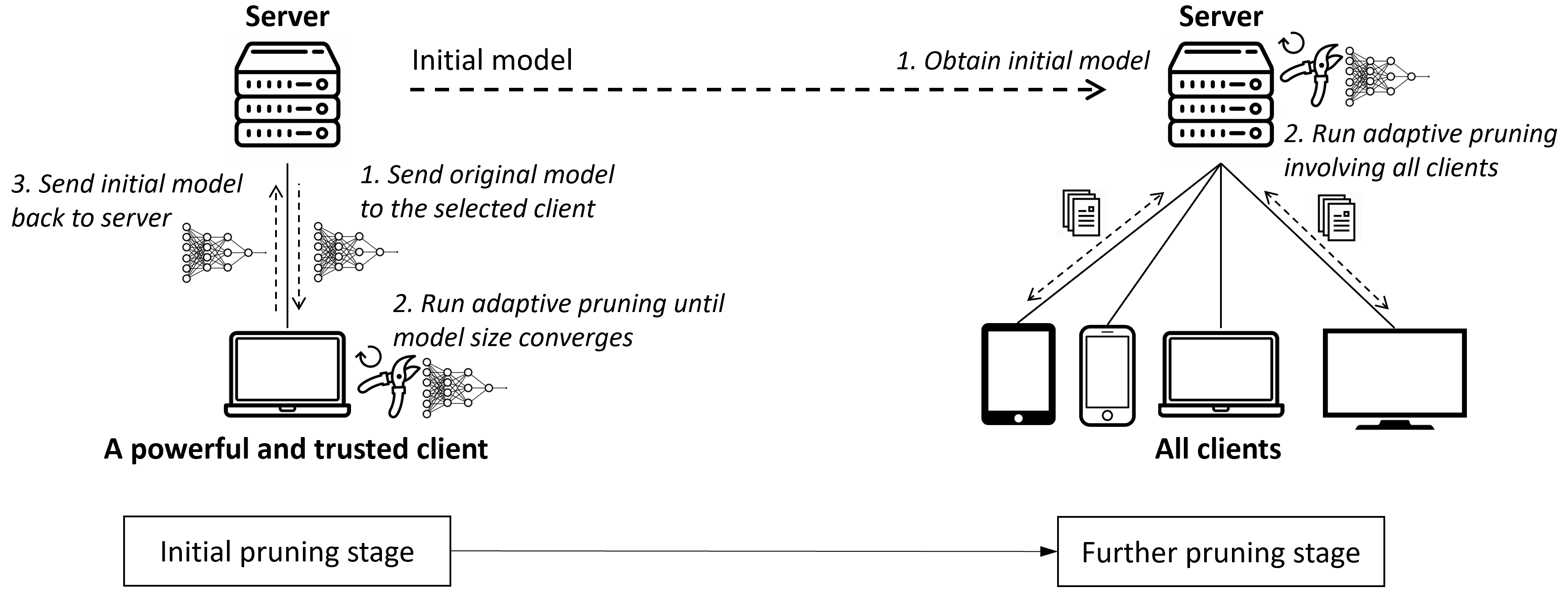}
    \caption{Illustration and flowchart of PruneFL.}
    \vspace{-3mm}
    \label{fig:prunefl_illu}
\end{figure*}

\section{Preliminaries}
\label{sec:preliminary}

\textbf{Federated learning.}
We consider an FL system with $N$ clients. Each client $n \in [N] := \{1,2,...,N\}$ has a local empirical risk $F_n(\mathbf{w}) := \frac{1}{D_n} \sum_{i \in \mathcal{D}_n} f_i(\mathbf{w})$ defined on its local dataset $\mathcal{D}_n$ ($D_n := \left| \mathcal{D}_n \right|$) for model parameter vector 
$\mathbf{w}$, where $f_i(\mathbf{w})$ is the loss function (e.g., cross-entropy, mean square error, etc.) that captures the difference between the model output and the desired output of data sample $i$. The system tries to find a parameter $\mathbf{w}$ that minimizes the global empirical risk:
\begin{equation}
    \min_\mathbf{w} \,\,\, F(\mathbf{w}) := \sum_{n \in [N]} \, p_n F_n(\mathbf{w})
\label{eq:globalEmpiricalRiskFromLocalRisk}
\end{equation}
where $p_n > 0$ are weights such that $\sum_{n \in [N]} p_n = 1$. For example, if $\mathcal{D}_n \cap \mathcal{D}_{n'} = \emptyset$ for $n \neq n'$ and $p_n = {D_n}/{D}$ with $\mathcal{D} := \bigcup_n \mathcal{D}_n$ and $D := |\mathcal{D}|$,  we have $F(\mathbf{w}) = \frac{1}{D} \sum_{i \in \mathcal{D}} f_i(\mathbf{w})$. Other ways of configuring $p_n$ may also be used to account for fairness and other objectives~\cite{Li2020Fair}.

In FL, each client $n$ has a local parameter $\mathbf{w}_n(k)$ in iteration $k$. The aggregation of these local parameters is defined as $\mathbf{w}(k) := \sum_{n \in [N]} p_n \mathbf{w}_n(k)$.
The FL procedure usually involves multiple updates of $\mathbf{w}_n(k)$ using stochastic gradient descent (SGD) on the local empirical risk $F_n(\mathbf{w}_n(k))$ computed by every client $n$, followed by a parameter fusion step that involves the server collecting clients' local parameters $\{\mathbf{w}_n(k) : \forall n\in[N]\}$ and computing the aggregated parameter $\mathbf{w}(k)$. After parameter fusion, the local parameters $\{\mathbf{w}_n(k) : \forall n\in[N]\}$ are all set to be equal to the aggregated parameter $\mathbf{w}(k)$.

In the following, we call this procedure of multiple local SGD iterations followed by a fusion step a \emph{round}. We use $I$ to denote the number of local SGD iterations in each round. The main notations in this paper are listed in Table~\ref{tab:main_notations}.

{\renewcommand{\arraystretch}{1.35}
\begin{table}[t]
\caption{Main notations.}
\label{tab:main_notations}
\centering
\begin{tabular}{c||c}
\hline
Notation & Definition \\ \hline \hline
$\odot$ & element-wise product of two vectors \\\hline
$n, N$ & client index, total number of clients \\\hline
$k, K$ & iteration index, total number of iterations \\\hline
$I$ & number of local iterations\\\hline
$p_n$ & weight for client $n$ ($p_n>0$, $\forall n$, and $\sum_n p_n = 1$)\\\hline
$\mathbf{m}(k)$ & weight mask in iteration $k$ (universal for all clients)\\\hline
$\mathbf{w}_n(k)$ & client $n$'s parameter in iteration $k$ \\\hline
${\mathbf{w}}(k)$ & $\mathbf{w}(k) := \sum_{n=1}^N p_n \mathbf{w}_n(k)$ \\\hline
$\mathbf{w}'_n(k)$, ${\mathbf{w}}'(k)$ & $\mathbf{w}'_n(k) = \mathbf{w}_n(k) \odot \mathbf{m}(k)$, $\mathbf{w}'(k) = \mathbf{w}(k) \odot \mathbf{m}(k)$\\\hline
$\mathbf{g}_n (\mathbf{w})$ & client $n$'s stochastic gradient with parameter $\mathbf{w}$ \\\hline
$\nabla F_n(\mathbf{w})$ & client $n$'s expected gradient with parameter $\mathbf{w}$\\\hline
$\nabla F(\mathbf{w})$ & $\nabla F(\mathbf{w}) = \sum_{n=1}^N p_n F_n(\mathbf{w})$\\\hline
\end{tabular}
\end{table}
}

It is possible that each round only involves a subset of clients, to avoid excessive delay caused by waiting for all the clients~\cite{bonawitz2019towards}. It has been shown that FedAvg converges even with random client participation, although the convergence rate is related to the degree of such randomness~\cite{Li2020On}.

\textbf{Model pruning.}
In the iterative training and pruning approaches  for the centralized machine learning setting, the model is first trained using SGD for a given number of iterations~\cite{han2015learning,narang2017exploring,zhu2017prune}. Then, a certain percentage (referred to as the pruning rate) of weights that have smallest absolute values within each layer is removed (set to zero). This training and pruning process is repeated until a desired model size is reached. The benefit of this approach is that the training and pruning occurs at the same time, so that a trained model with a desired (small) size can be obtained in the end. However, existing pruning techniques require the availability of training data at a central location, which is not applicable to FL.

\section{PruneFL}
\label{sec:proposed_approach}
Our proposed PruneFL approach includes two stages: initial pruning at a selected client and further pruning involving both the server and clients during the FL process. The initial pruning can be done with biased data at a single client that has a relatively high computational capability, and the further pruning stage will ``remove'' the bias and refine the model. We use \textit{adaptive pruning} in both stages. The illustration of the overall procedure is presented in Fig.~\ref{fig:prunefl_illu}. In the following, we introduce the two pruning stages (Section~\ref{sec:distributed-pruning}) and adaptive pruning (Section~\ref{sec:adaptive_pruning}).

\subsection{Two-stage Distributed Pruning}
\label{sec:distributed-pruning}

\textbf{Initial pruning at a selected client.}
Before FL starts, the system selects a single client to prune the model using its local data. This is important for two reasons. First, it allows us to start the FL process with a small model, which can significantly reduce the computation and communication time of each FL round. Second, when clients have heterogeneous computational capabilities, the selected client for initial pruning can be one that is powerful and trusted, so that the time required for initial pruning is short. We apply the adaptive pruning procedure that we describe in Section~\ref{sec:adaptive_pruning}, where we adjust the original model iteratively while training the model on the selected client's local dataset, until the model size remains almost unchanged.

\textbf{Further pruning during FL process.}
The model produced by initial pruning may not be optimal, because it is obtained based on data at a single client. However, it is a good starting point for the FL process involving all clients.
During FL, we perform further adaptive pruning together with the standard FedAvg procedure, where the model can either grow or shrink depending on which way makes the training most efficient. In this stage, data from all participating clients are involved.

\subsection{Adaptive Pruning}\label{sec:adaptive_pruning}
\begin{algorithm}[t]
\caption{Adaptive pruning}
\label{alg:adaptive_pruning}
\For{$k = 0 \ldots, K - 1$}{
    Initialize the set of importance measure on each client: $\mathcal{Z}_n \leftarrow \emptyset, \forall n$;\\
    \For{each client $n$, in parallel}{
    Compute stochastic gradient $\mathbf{g}_n(\mathbf{w}'_n(k)) := \mathbf{g}_n(\mathbf{w}_n(k) \odot \mathbf{m}(k))$;\\
    Update local parameters: $\mathbf{w}_n(k+1) \leftarrow \mathbf{w}'_n(k) - \eta \mathbf{g}_n(\mathbf{w}_n'(k)) \odot \mathbf{m}(k)$;\\
    Add importance measure $\mathbf{z}_n$ to $\mathcal{Z}_n$:
    $\mathbf{z}_n:=\mathbf{g}_n(\mathbf{w}_n'(k)) \odot \mathbf{g}_n(\mathbf{w}_n'(k))$;
    $\mathcal{Z}_n \leftarrow \mathcal{Z}_n \cup \mathbf{z}_n$;
    }
    
    \If{$I \mid k+1$}{
    Each client $n$ sends $\mathbf{w}_n(k+1)$ to the server;\\
    Server aggregates the parameters from each client: $\mathbf{w}(k+1) \leftarrow \sum_{n=1}^N p_n \mathbf{w}_n(k+1)$;\\
    \If{$k+1$ is reconfiguration iteration}{
    Each client sends the averaged importance measure $\overline{\mathbf{z}}_n:= \Big( \sum_{\mathbf{z}_n\in \mathcal{Z}_n} \mathbf{z}_n \Big) / |\mathcal{Z}_n|$ to the server;\\
    Server aggregates the received importance measure: $\mathbf{z} \leftarrow \sum_{n=1}^N p_n \overline{\mathbf{z}}_n$;\\
    Reconfigure using Algorithm~\ref{alg:greedySearchNewLinear}: $\mathbf{w}'(k+1), \mathbf{m}(k + 1) \leftarrow \text{reconfigure}(\mathbf{w}(k+1), \mathbf{z})$;\\
    Reset: $\mathcal{Z}_n \leftarrow \emptyset, \forall n$;\\
    }
    \Else{
    No reconfiguration: $\mathbf{w}'(k+1) \leftarrow \mathbf{w}(k+1)$;
    }
    Server sends new parameters to each client: $\mathbf{w}'_n(k+1) \leftarrow \mathbf{w}'(k+1), \forall n$;
    }
}
\end{algorithm}
For our adaptive pruning method, the notion of pruning broadly includes both removing and adding back parameters. Hence, we also refer to such pruning operations as \emph{reconfiguration}. We reconfigure the model at a given interval of multiple iterations. For initial pruning, the reconfiguration interval can be any number of local iterations at the selected client. For further pruning, reconfiguration is done at the server after receiving parameter updates from clients (i.e., at the boundary between two rounds), and the reconfiguration interval in this case is always an integer multiple of the number of iterations (i.e., $I$) in each round.

\textbf{Definitions.} Let $k$ denote the iteration index, $\mathbf{g}_n(\mathbf{w}(k))$ denote the stochastic gradient of $F_n(\mathbf{w}(k))$ evaluated at $\mathbf{w}(k)$ and computed on the full parameter space on client $n$. 
Also, let $\mathbf{m}_\mathbf{w}(k)$ denote a mask vector that is zero if the corresponding component in $\mathbf{w}(k)$ is pruned and one if not pruned, and $\odot$ denote the element-wise product.

In each reconfiguration step, adaptive pruning finds an optimal set of remaining (i.e., not pruned) model parameters. %
Then, parameters are pruned or added back accordingly, and the resulting model and mask are used for training until the next reconfiguration step. This procedure is illustrated in Algorithm~\ref{alg:adaptive_pruning}, where $a \mid b$ denotes that $a$ divides $b$, i.e., $b$ is an integer multiple of $a$. More details are given in Appendix~\ref{appendix:implementation_details}.

Our goal is to find the subnetwork that learns the ``fastest''. We do so by estimating the empirical risk reduction divided by the time required for completing an FL round, for any given subset of parameters chosen to be pruned. Note that at the beginning of each round (after averaging the clients' parameters), all clients start with the same parameter vector $\mathbf{w}$, i.e., $\mathbf{w}_n(k) = \mathbf{w}(k)$ for all $n$ if a new round starts at iteration $k$ (see Section~\ref{sec:preliminary}). For approximation purpose, we first consider the change of empirical risk after \textit{one SGD iteration} starting with a \textit{common} parameter $\mathbf{w}(k)$, for both initial and further pruning stages. The full FL procedure will be considered later in Theorem~\ref{theorem:convergence}.

When the model is reconfigured at the end of iteration $k$, parameter update in the next iteration will be done on the reconfigured parameter $\mathbf{w}'(k)$, so we have an SGD update step that follows:
\begin{align}
\mathbf{w}(k+1) &= \mathbf{w}'(k) - \eta \sum_{n=1}^N p_n \mathbf{g}_n(\mathbf{w}'(k)) \odot \mathbf{m}_{\mathbf{w}'}(k) \nonumber \\
& = \mathbf{w}'(k) - \eta \mathbf{g}(\mathbf{w}'(k)) \odot \mathbf{m}(k)
\label{eq:SGDUpdate}
\end{align}
where $\eta$ is the learning rate. For simplicity, we define $\mathbf{g}(\mathbf{w}'(k)) := \sum_{n=1}^N p_n \mathbf{g}_n(\mathbf{w}'(k))$, and we omit the subscript $\mathbf{w}'$ of $\mathbf{m}$ in the following when it is clear from the context. 

Let $\mathcal{M}$ denote the index set of components that are not pruned, which corresponds to the indices of all non-zero values of the mask $\mathbf{m}(k)$.

\textbf{Empirical risk reduction.} 
To analyze the empirical risk reduction, we use a first-order approximation, which is a common practice in the literature \cite{lee2018snip,wang2019picking,molchanov2016pruning}.
We have
\begin{align}
    & {F(\mathbf{w}(k+1))} \nonumber\\
    & \quad\approx F(\mathbf{w}'(k)) + \langle\nabla F(\mathbf{w}'(k)) , \mathbf{w}(k+1) - \mathbf{w}'(k)\rangle \label{eq:risk-reduction-derivation1} \\
    & \quad = F(\mathbf{w}'(k)) - \eta \langle\nabla F(\mathbf{w}'(k)) , \mathbf{g}(\mathbf{w}'(k)) \odot \mathbf{m}(k)\rangle \label{eq:risk-reduction-derivation2} \\
    & \quad\approx F(\mathbf{w}(k)) - \eta \Vert \mathbf{g}(\mathbf{w}'(k)) \odot \mathbf{m}(k) \Vert^2 \label{eq:risk-reduction-derivation3}
\end{align}
where $\langle\cdot , \cdot\rangle$ is the inner product, (\ref{eq:risk-reduction-derivation1}) is from Taylor expansion, (\ref{eq:risk-reduction-derivation2}) is because of (\ref{eq:SGDUpdate}), and (\ref{eq:risk-reduction-derivation3}) is obtained by using the stochastic gradient to approximate the actual gradient, i.e., $\mathbf{g}(\mathbf{w}'(k)) \approx \nabla F(\mathbf{w}'(k))$. Then, the approximate decrease of empirical risk after the SGD step (\ref{eq:SGDUpdate}) is:
\begin{align}
F(\mathbf{w}'(k)) - F(\mathbf{w}(k+1)) &\approx \eta \Vert \mathbf{g}(\mathbf{w}'(k)) \odot \mathbf{m}(k) \Vert^2 \nonumber \\ 
& \propto \Vert \mathbf{g}(\mathbf{w}'(k)) \odot \mathbf{m}(k) \Vert^2 \nonumber \\ 
& = \sum_{j \in \mathcal{M}} g_j^2 =:\Delta(\mathcal{M})
\label{eq:riskDecayApproximation}
\end{align}
where $g_j$ is the $j$-th component of $\mathbf{g}(\mathbf{w}'(k))$ and we define the set function $\Delta(\mathcal{M})$ in the last line. The learning rate $\eta$ is omitted since it is independent to the relative importance between parameter components, no matter if it is constant or varying. We use $\Delta(\mathcal{M})$ as the \emph{approximate risk reduction}, where we ignore the proportionality coefficient because our optimization problem is independent of the coefficient. As $\Delta(\mathcal{M})$ is defined as the sum of $g^2_j$ in (\ref{eq:riskDecayApproximation}), we use $g^2_j$ as the \textit{importance measure} for the $j$-th component of the parameter vector.

\emph{Remark.} During the further pruning stage, the stochastic gradient $\mathbf{g}(\mathbf{w}'(k))$ is the aggregated stochastic gradient from clients in FL. Since clients cannot compute $\mathbf{g}(\mathbf{w}'(k))$ before receiving $\mathbf{w}'(k)$ from the server, they compute $\mathbf{g}(\mathbf{w}(k))$ and we use $\mathbf{g}(\mathbf{w}'(k)) \approx \mathbf{g}(\mathbf{w}(k))$, both of which are denoted by $\mathbf{g}(\mathbf{w}'(k))$ with components $\{g_j\}$ in the following. The additional overhead for clients to compute and transmit gradients on the full parameter space in a reconfiguration is small because pruning is done once in many FL rounds (the interval between two reconfigurations is $50$ rounds in our experiments). Further details are given in Appendix~\ref{appendix:implementation_details}.

\textbf{Time of one FL round.}
We define the (approximate) time of one FL round when the model has remaining parameters $\mathcal{M}$ as a set function $T(\mathcal{M}) := c + \sum_{j \in \mathcal{M}} t_j$, where $c \geq 0$ is a fixed constant and $t_j > 0$ is the time corresponding to the $j$-th parameter component. Note that this is a linear function which is sufficient according to our empirical observations (see Appendix~\ref{appendix:validate}). In particular, the quantity $t_j$ has a value that can be dependent on the neural network layer, and $c$ captures a constant system overhead.
From our experiments, we observed that $t_j$ remains the same for all $j$ that belong to the same neural network layer. Therefore, we can estimate the quantities $\{t_j\}$ and $c$ by measuring the time of one FL round for a small subset of different model sizes, before the overall pruning and FL procedure starts.
An extension to the general case with non-linear $T(\mathcal{M})$ is also discussed in Appendix~\ref{appendix:non-linear-T}.

\textbf{Optimization of reconfiguration.} We would like to find the set of remaining parameters $\mathcal{M}$ that maximizes the empirical risk reduction per unit training time. However, $\Delta(\mathcal{M})$ only captures the risk reduction in the \textit{next} SGD step when starting from the reconfigured parameter vector $\mathbf{w}'(k)$, as defined in (\ref{eq:riskDecayApproximation}). It does not capture the change in empirical risk when using $\mathbf{w}'(k)$ instead of the original parameter vector $\mathbf{w}(k)$ before reconfiguration. In other words, in addition to maximizing $\Gamma(\mathcal{M}):= \frac{\Delta(\mathcal{M})}{T(\mathcal{M})}$, we also need to ensure that $F(\mathbf{w}'(k)) \approx F(\mathbf{w}(k))$.

To ensure $F(\mathbf{w}'(k)) \approx F(\mathbf{w}(k))$ after reconfiguration, we define an index set $\overline{\mathcal{P}}$ to denote the parameters that \textit{are not allowed to} be pruned. Usually, $\overline{\mathcal{P}}$ includes parameters whose magnitudes are larger than a certain threshold, because pruning them can cause $F(\mathbf{w}'(k))$ to become much larger than $F(\mathbf{w}(k))$.
Among the remaining parameters that \textit{can} be pruned (or added back if they are already pruned before), denoted by $\mathcal{P}$, we find which of them to prune to maximize $\Gamma(\mathcal{M})$. This yields the following optimization problem:
\begin{equation}
    \max_\mathcal{A \subseteq P} \quad \Gamma \left( \mathcal{A} \cup \overline{\mathcal{P}} \right)
    \label{eq:optimizationProblem}
\end{equation}
where $\mathcal{A}$ is the set of parameters in $\mathcal{P}$ that remain (i.e., are not pruned). The final set of remaining parameters is then $\mathcal{M} = \mathcal{A} \cup \overline{\mathcal{P}}$. Note that $\mathcal{P} \cup \overline{\mathcal{P}}$ is the set of all parameters in the original model.

The algorithm for solving (\ref{eq:optimizationProblem}) is given in
Algorithm \ref{alg:greedySearchNewLinear}, where sorting is in non-increasing order and $\mathcal{S}$ is an ordered set that includes the sorted indices.
In essence, this algorithm sorts the ratios of components in the sums of $\Delta(\mathcal{M})$ and $T(\mathcal{M})$. When the individual ratio ${g_j^2}/{t_j}$ is larger than the current overall ratio $\Gamma$, then adding $j$ to $\mathcal{A}$ increases $\Gamma$. The bottleneck of this algorithm is the sorting operation. Hence, the overall time complexity of this algorithm is $O(|\mathcal{P}|\log |\mathcal{P}|)$.

\begin{algorithm}[t]
\caption{Solving (\ref{eq:optimizationProblem})}
\label{alg:greedySearchNewLinear}
\SetKwInOut{inputdata}{Input}
\SetKwInOut{outputdata}{Output}
\inputdata{importance measure $g^2_j$ and time coefficient $t_j$, for each parameter index $j$}
\outputdata{the optimal subset of parameters $\mathcal{A}$}
$\mathcal{A} \leftarrow \emptyset$;

$\mathcal{S} \leftarrow {\arg\textrm{sort}}_{j \in \mathcal{P}} \,\, \frac{g_j^2}{t_j}$ \tcp*{ordered set}

\For{$j\in \mathcal{S}$}{

\If{$\frac{g_j^2}{t_j} \geq \Gamma \left( \mathcal{A} \cup \overline{\mathcal{P}} \right)$}{
$\mathcal{A} \leftarrow \mathcal{A} \cup \{ j \}$;
}
\Else{\textbf{break};}
}
\textbf{return} $\mathcal{A}$ \tcp*{final result}
\end{algorithm}

\begin{theorem}
We have $\Gamma \left( \mathcal{A} \cup \overline{\mathcal{P}} \right) \geq \Gamma \left( \mathcal{A}' \cup \overline{\mathcal{P}} \right)$, where $\mathcal{A}$ from Algorithm~\ref{alg:greedySearchNewLinear} and $\mathcal{A}'$ is any subset of $\mathcal{P}$ with $\mathcal{A}' \neq \mathcal{A}$.
\label{theorem:linearT}
\end{theorem}
Theorem~\ref{theorem:linearT} shows that the result obtained from our Algorithm~\ref{alg:greedySearchNewLinear} is a global optimal solution to (\ref{eq:optimizationProblem}).

\textbf{Convergence of adaptive pruning.}
As adaptive pruning can both increase and decrease the model size over time, a natural question is whether the model parameter vector will converge to a fixed value in a regular FL procedure with $I\geq 1$ local SGD iterations in each round. We study this problem in the following.

We first make the following minimal set of assumptions that are common in the literature~\cite{yu2019parallel,Li2020On}.

\begin{assumption}\label{assumption:convergence}
\text{}
\begin{enumerate}
    \item[(a)] \textbf{Smoothness}:
    $$
    \left\Vert \nabla F_n(\mathbf{w}_1) - \nabla F_n(\mathbf{w}_2) \right\Vert \leq \beta \left\Vert \mathbf{w}_1 - \mathbf{w}_2 \right\Vert, \forall n, \mathbf{w}_1, \mathbf{w}_2 ,
    $$
    where $\beta$ is a positive constant.
    \item[(b)] \textbf{Lipschitzness}:
    $$
    \left\Vert F(\mathbf{w}_1) - F(\mathbf{w}_2) \right\Vert \leq L \left\Vert \mathbf{w}_1 - \mathbf{w}_2 \right\Vert, \forall \mathbf{w}_1, \mathbf{w}_2 \,\,,
    $$
    where $L$ is a positive constant.
    \item[(c)] \textbf{Unbiasedness}:
    $$
    \Expect \left[ \mathbf{g}_n(\mathbf{w}) \right] = \nabla F_n(\mathbf{w}), \forall n, \mathbf{w}\,\,.
    $$
    \item[(d)] \textbf{Bounded variance}:
    $$
    \Expect \left\Vert \mathbf{g}_n (\mathbf{w}) - \nabla F_n(\mathbf{w}) \right\Vert^2 \leq \sigma^2, \forall n, \mathbf{w}\,\,.
    $$
    \item[(e)] \textbf{Bounded divergence}:
    $$
    \Big\Vert \nabla F(\mathbf{w}) - \nabla F_n(\mathbf{w}) \Big\Vert^2 \leq \epsilon^2 , \forall n,\mathbf{w} \,\,,
    $$
    where $F(\mathbf{w}) := \sum_{n=1}^N p_n F_n (\mathbf{w})$ as defined in (\ref{eq:globalEmpiricalRiskFromLocalRisk}). 
    \item[(f)] \textbf{Time independence in SGD}:
    The stochastic gradients obtained in different iterations are independent from each other.
    \item[(g)] \textbf{Client independence}:
    The stochastic gradients obtained from different clients are always independent from each other, even in the same iteration. 
\end{enumerate}

\end{assumption}

\begin{theorem}
\label{theorem:convergence}

When Assumption \ref{assumption:convergence} holds, and $\eta \leq \frac{1}{2\sqrt{6} I \beta}$, we have
\begin{align}
&\frac{1}{K} \sum_{k=0}^{K-1} \Expect \big\Vert \nabla F({\mathbf{w}}'(k)) \odot \mathbf{m}_{\mathbf{w}'}(k) \big\Vert^2 \nonumber \\ 
&\leq \frac{2(F_0 - F^*)}{\eta K} + \alpha \eta \beta \sigma^2 + 4\beta^2 \big( (1- \alpha ) I\sigma^2 + 3I^2\epsilon^2 \big) \eta^2 \nonumber\\
& \quad\quad + \frac{2 L}{\eta K} \sum_{k=0}^{K-1} \Expect \left\| {\mathbf{w}}(k) - {\mathbf{w}}'(k) \right\|\,\,,
\label{eq:convergence}
\end{align}
where $\alpha:= \sum_{n=1}^N p_n^2$, $F_0 := F(\mathbf{w}(0))$, $F^* := \min_\mathbf{w}F(\mathbf{w})$, and $I$ is the number of iterations per round.
\end{theorem}
Under some conditions, this convergence can be bounded asymptotically by $\mathcal{O} \left(\frac{1}{\sqrt{N K}} \right) + \mathcal{O} \left( \frac{1}{K} \right)$, achieving linear speedup\footnote{The dominant term is $\mathcal{O}\left(\frac{1}{\sqrt{N K}} \right)$ when $K$ is sufficiently large. The notion of linear speedup means that, to reach the same error bound, the total number of rounds $K$ can proportionally decrease as $N$ increases.} with the number of clients $N$~\cite{peng2018linear,yu2019linear} for sufficiently large $K$. See Appendix~\ref{appendix:thm2_proof} for more details.

If we do not reconfigure in an iteration $k$, we have $\mathbf{w}(k) = \mathbf{w}'(k)$.
In the right-hand side (RHS) of (\ref{eq:convergence}), the first two terms go to zero as $K \rightarrow \infty$. The last term is related to how well $\mathbf{w}'$ approximates $\mathbf{w}$ after pruning. To ensure that the sum in the last term grows slower than $\sqrt{K}$, the number of \textit{non-zero} prunable parameters (which belong to $\mathcal{P}$) should decrease over time. Note that we consider all zero parameters to be prunable and they also belong to $\mathcal{P}$, thus the size of $\mathcal{P}$ itself may not decrease over time. 
This convergence result shows that the gradient components corresponding to the remaining (i.e., not pruned) parameters vanishes over time, which suggests that we will get a ``stable'' parameter vector in the end, because when the gradient norm is small, the change of parameters in each iteration is also small. In addition to gradient convergence on the subspace after pruning as suggested in Theorem~\ref{theorem:convergence}, our experiments show that our pruned model also converges to an accuracy close to that of the full-sized model.

\textbf{Tracking a small model.}
By choosing the size of $\overline{\mathcal{P}}$ properly over time, our adaptive pruning algorithm can keep reducing the model size as long as such reduction does not adversely impact further training. Intuitively, the model that we obtain from this process is one that has a small size while maintaining full ``trainablity'' in future iterations. Parameter components for which the corresponding gradient components remain zero (or close to zero) will be pruned.

In cases where a target maximum model size should be reached at convergence (e.g., for efficient inference later), we can also enforce a maximum size constraint in each reconfiguration that starts with the full size and gradually decreases to the target size as training progresses, which allows the model to train quickly in initial rounds while converging to the target size in the end.

\section{Implementation}\label{sec:implementation}

\subsection{Using Sparse Matrices}
Although the benefit of model pruning in terms of computation is constantly mentioned in the literature from a theoretical point of view \cite{han2015learning}, most existing implementations substitute sparse parameters by applying binary masks to dense parameters. Applying masks increases the overhead of computation, instead of reducing it. We implement sparse matrices for model pruning, and we show its efficacy in our experiments. We use dense matrices for full-sized models, and sparse matrices for weights in both convolutional and fully-connected layers in pruned models.

\subsection{Complexity Analysis}\label{sec:complexity_analysis}

\textbf{Storage, memory, and communication.}
We implement two types of storage for sparse matrices: bitmap and value-index tuple. Bitmap uses one extra bit to indicate whether the specific value is zero. For $32$-bit floating point parameter components, bitmap incurs $1/32$ extra storage and communication overhead. Value-index tuple stores the values and both row and column indices of all non-zero entries. In our implementation, we use $16$-bit integers to store row and column indices and $32$-bit floating point numbers to store parameter values. Since each parameter component is associated with a row index and a column index, the storage and communication overhead doubles compared to storing the values only. We dynamically choose between the two ways of storage, and thus, the ratio of the sparse parameter size to the dense parameter size is $\min\big\{2\times d, \frac{1}{32}+d\big\}$, where $d$ is the model's \textit{density} (percentage of non-zero parameters). This ratio is further optimized when the matrix sparsity pattern is fixed (in most FL rounds, see Appendix~\ref{appendix:implementation_details}). In this case, there is no extra cost since only values of the non-zero entries need to be exchanged.

\textbf{Computation.}
Because dense matrix multiplication is extremely optimized, sparse matrices will show advantage in computation time only when the matrix is below a certain density, where this density threshold depends on specific hardware and software implementations. 
In our implementation, we choose either dense or sparse representation depending on which one is more efficient. The complexity (computation time) of the matrix multiplication between a sparse matrix $\mathbf{S}$ and a dense matrix $\mathbf{D}$ is linear to the number of non-zero entries in $\mathbf{S}$ (assuming $\mathbf{D}$ is fixed).

\subsection{Implementation Challenges}\label{sec:implementation-challenge}
As of today, well-known machine learning frameworks have limited support of sparse matrix computation. For instance, in PyTorch version 1.6.0, the persistent storage of a matrix in sparse form takes $5\times$ space compared to its dense form; the computations on sparse matrices are slow; and sparse matrices are not supported for the kernels in convolutional layers, etc. To benefit from using sparse matrices in real systems, we extend the PyTorch library by implementing a more efficient sparse storage, and the support for sparse convolutional kernels. We only partially improve backward passes due to implementation limitations (more details in Appendix~\ref{appendix:gradient_comp}). This problem, however, can be improved in the future by implementing and further optimizing efficient sparse matrix multiplication on low-level software, as well as developing specific hardware for this purpose. Nevertheless, the novelty in our implementation is that we use sparse matrices in both fully-connected and convolutional layers in the pruned model.

\section{Experiments}\label{sec:experiments}

In this section, we present the experimental setup and results.

{\renewcommand{\arraystretch}{1.3}
\begin{table*}[t]
\caption{Evaluation configurations (C.S. stands for client selection; LR stands for learning rate).\vspace{-0.05in}}
\label{tab:hyperparam}
\centering
\footnotesize
\vskip 1mm
\begin{tabular}{ccccc}
\hline
Dataset & FEMNIST & CIFAR-10 & ImageNet-100 & CelebA\\ \hline \hline
\noalign{\vskip 0.5mm}  
SGD params in round $r$ & LR = $0.25$& LR = $0.1\cdot 0.5^{\frac{r}{10000}}$& LR = $0.05\cdot 0.5^{\lfloor \frac{r}{1000} \rfloor\cdot 0.1}$ & LR = $0.2$\\ \hline
\noalign{\vskip 0.5mm}
Fraction of non-zero prunable parameters in round $r$ & $0.3\cdot 0.5^{\frac{r}{10000}}$ & $0.3\cdot 0.5^{\frac{r}{10000}}$& $0.3\cdot 0.5^{\frac{r}{10000}}$ & $0.3\cdot 0.5^{\frac{r}{10000}}$\\ \hline

Number of data samples used in initial pruning & 200 & 200 & 500 & 500\\ \hline

Number of clients (Non-C.S., C.S.) & 10, 193 & 10, 100 & 10, 100 & 10, 934\\ \hline

Mini-batch size, local iterations in each round & 20, 5 & 20, 5 & 20, 5 & 20, 5\\ \hline %

Reconfiguration & every 50 rounds & every 50 rounds & every 50 rounds & every 50 rounds\\ \hline

Total number of FL rounds & 10,000 & 10,000 & 20,000 & 1,000 \\ \hline

Evaluation & \begin{tabular}[c]{@{}c@{}} prototype (Pi 4), \\ simulation (Pi 4)\end{tabular} & simulation (Pi 4)& simulation (Android VM) & simulation (Pi 4)\\ \hline

\end{tabular}
\end{table*}
}

\textbf{Datasets.} We evaluate PruneFL on four image classification tasks: 
\begin{enumerate}
    \item[(a)] Conv-2 model on FEMNIST~\cite{LEAF},
    \item[(b)] VGG-11 model~\cite{simonyan2014very} on CIFAR-10~\cite{krizhevsky2009cifar},
    \item[(c)] ResNet-18 model~\cite{he2016deep} on ImageNet-100~\cite{deng2009imagenet},
    \item[(d)] MobileNetV3-Small model~\cite{howard2019searching} on CelebA~\cite{LEAF},
\end{enumerate}
all of which represent typical FL tasks. Due to practical considerations of edge devices' training time and storage capacity, we select data corresponding to $193$ writers for FEMNIST, and the first $100$ classes of the ImageNet dataset (referred to as ImageNet-100). We adapt some layers in VGG-11, ResNet-18 and MobileNetV3-Small to match with the number of output labels in our datasets.

When using full client participation, because we only have 10 clients in total, for FEMNIST, we partition all the 193 writers' images into 10 clients (the first 9 clients each has 19 writers' images and the last client has 22 writers' images).
For CelebA, we partition all the 9,343 persons' images into 10 clients (the first 9 clients each has 934 persons' images and the last client has 937 persons' images). Note that such partitioning is still non-IID.

\textbf{Model architectures.} The architecture details are presented in Table~\ref{tab:model_arch} in the appendix. VGG-11 ResNet-18, and MobileNetV3-Small are well-known architectures, and we directly acquire Conv-2 from its original work~\cite{LEAF}.

\textbf{Platform.}
To study the performance of our proposed approach, we conduct experiments in (i) a real edge computing prototype, where a personal computer serves as both the server and a client, and the other clients are Raspberry Pi devices, and (ii) a simulated setting with multiple clients and a server, where computation and communication times are obtained from measurements involving either Raspberry Pi devices or Android phones.

Unless otherwise specified, the prototype system includes  nine Raspberry Pi (version 4, with 2~GB RAM, 32~GB SD card) devices as clients and a personal computer without GPU as both a client and the server (totaling 10 clients). Three of the Raspberry Pis use wireless connections and the remaining six use wired connections. The communication speed is stable and is approximately 1.4 MB/s. 
The simulated system uses the same setting as in the prototype. We use time measurements from Raspberry Pis, except for the ImageNet-100 dataset where we replace the computation time by measurements from Android virtual machine (VM). 

We consider FL with full client participation in the main paper and present results with random client selection~\cite{bonawitz2019towards} in Appendix~\ref{appendix:additional_results}. The results are similar. \mbox{FEMNIST} and \mbox{CelebA} data are partitioned into clients in a non-IID manner according to writer/person identity, and CIFAR-10 and ImageNet-100 are partitioned into clients in an IID manner.

\textbf{Baselines.} We compare the test accuracy vs.\ \textit{time} curve of PruneFL with five baselines: (i) conventional FL~\cite{mcmahan2017communication}, (ii) iterative pruning~\cite{han2015learning}, (iii) online learning~\cite{han2020adaptive}, (iv) SNIP~\cite{lee2018snip}, and (v) SynFlow~\cite{tanaka2020synflow}. Because iterative pruning and SNIP cannot automatically determine the model size, we consider an enhanced version of these baselines that obtain the same model size as PruneFL at convergence. Additional baselines are also considered in Section~\ref{sec:time_reduction}.

Since our experiments try to minimize the training time using pruning, there is no direct way of comparing with the baselines that either are not specifically designed for pruning (the online learning baseline) or do not adapt the pruned model size (all other baselines). We compare with the baselines as follows. In every round, the online learning approach produces a model size for the next round, and we adjust the model accordingly while keeping each layer's density the same. To compare with SNIP, after the first round, we let SNIP prune the original model in a one-shot manner to the same density as the final model found by our adaptive pruning method, and keep the architecture afterwards. Similarly, to compare with SynFlow, we let SynFlow prune the model (before training) to the same density as the final model found by our adaptive pruning method, and keep the architecture afterwards. To compare with iterative pruning, we let the model be pruned with a fixed rate for $20$ times (at an equal interval) in the first half of the total number of rounds, such that the remaining number of parameter components equals that of the model found by our adaptive pruning method, and the pruning rate is equal across layers. See Fig.~\ref{fig:model_size} for the illustration of the baseline settings.

\textbf{Pruning configurations.} 
The initial pruning stage is done on the personal computer client. We end the initial pruning stage either when the model size is ``stable'', or when it exceeds certain maximum number of iterations. We consider the model size as ``stable'' when its relative change is below $10\%$ for $5$ consecutive reconfigurations.

For adaptive pruning, to ensure convergence of the last term on the RHS of (\ref{eq:convergence}) in Theorem~\ref{theorem:convergence}, we exponentially decrease the number of \textit{non-zero} prunable parameters in $\mathcal{P}$ over rounds. We note that $\mathcal{P}$ includes both zero and non-zero parameters, hence the size of $\mathcal{P}$ itself may not decrease. 
For a given size of $\mathcal{P}$, the $|\mathcal{P}|$ parameters with the smallest magnitude belong to $\mathcal{P}$ that can be pruned (or added back), and the rest belong to $\overline{\mathcal{P}}$ that cannot be pruned.

Biases (if any) in the DNNs are not pruned. In ResNet-18, BatchNorm layers and downsampling layers are not pruned since the number of parameters in such layers is negligible compared to the size of convolutional and fully-connected layers.

\textbf{Lottery ticket analysis.} To verify whether the final model from adaptive pruning is a lottery ticket~\cite{frankle2018the,morcos2019one}, we reinitialize this converged model using the original random seed, and compare its accuracy vs. \textit{round} curve with (i) conventional FL, (ii) random reinitialization (same architecture as the lottery ticket but initialized with a different random seed), (iii) SNIP, and (iv) SynFlow.

\textbf{Hyperparameters.} The hyperparameters above are chosen empirically only with coarse tuning by experience. We observe that our and other methods are insensitive to these hyperparameters. Hence, we do not perform fine tuning on any parameter.

The detailed evaluation configurations are given in Table~\ref{tab:hyperparam}.

\begin{figure}[t]
    \centering
    \includegraphics[width=0.95\linewidth]{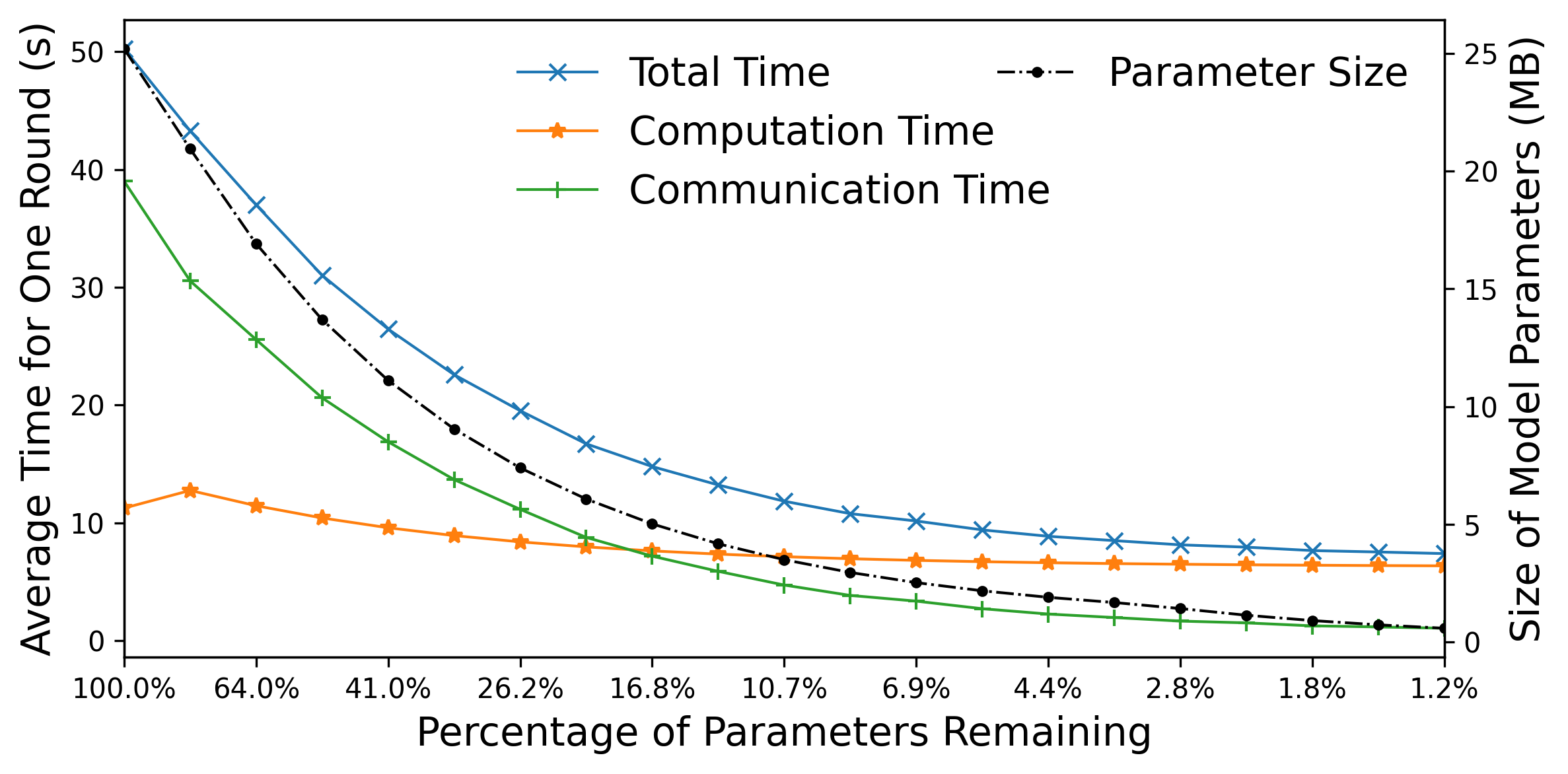}
    \vspace{-3mm}
    \caption{Training time on Raspberry Pi 4 (FEMNIST).}
    \vspace{-3mm}
    \label{fig:femnist_tm}
\end{figure}

\begin{figure}
  \centering
  \includegraphics[width=0.95\linewidth]{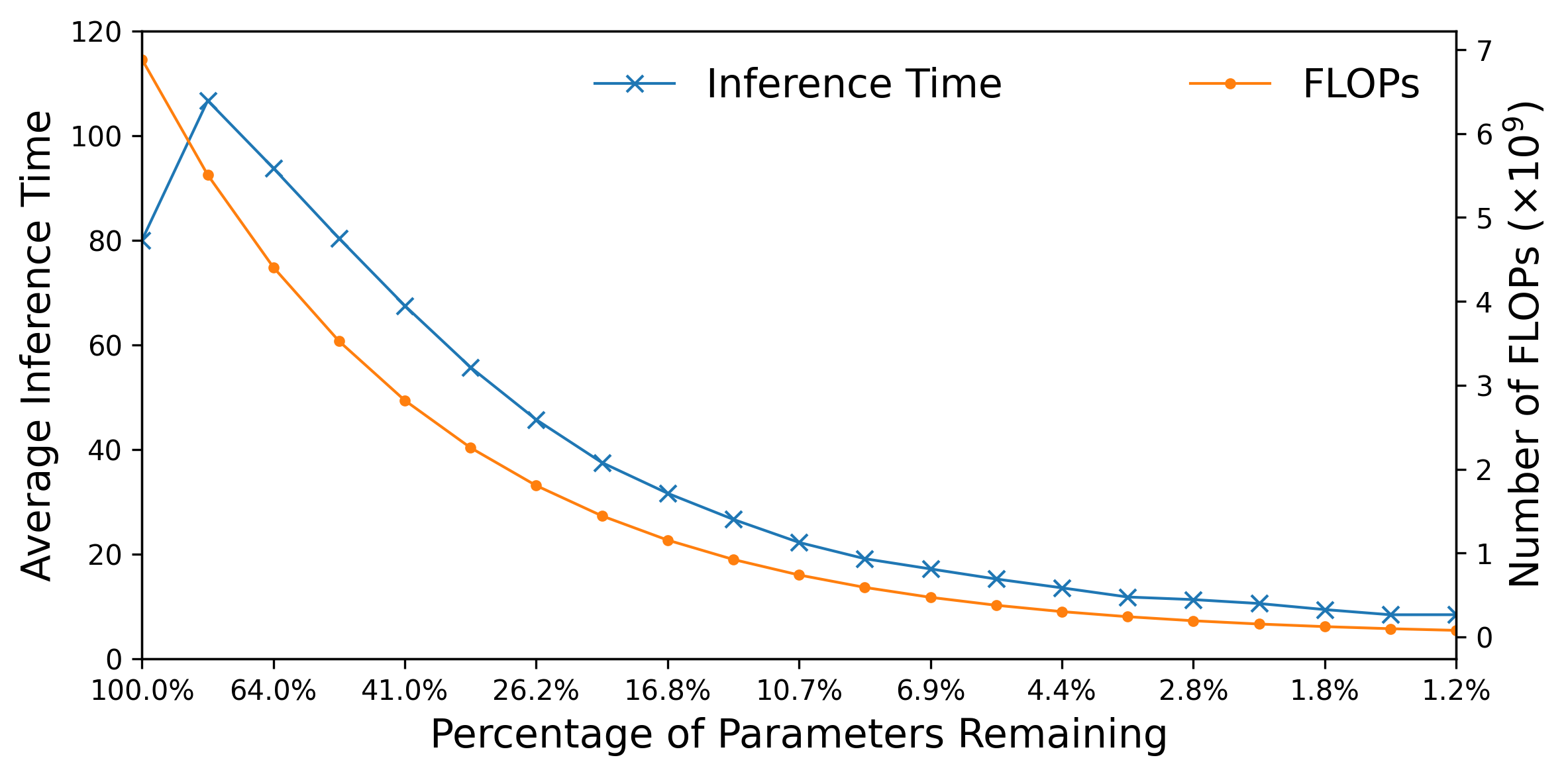}
  \vspace{-3mm}
  \caption{Inference time on Raspberry Pi 4 (FEMNIST).}
  \vspace{-3mm}
  \label{fig:femnist_tm_inf}
\end{figure}

\begin{figure}[t]
    \centering
    \includegraphics[width=0.95\linewidth]{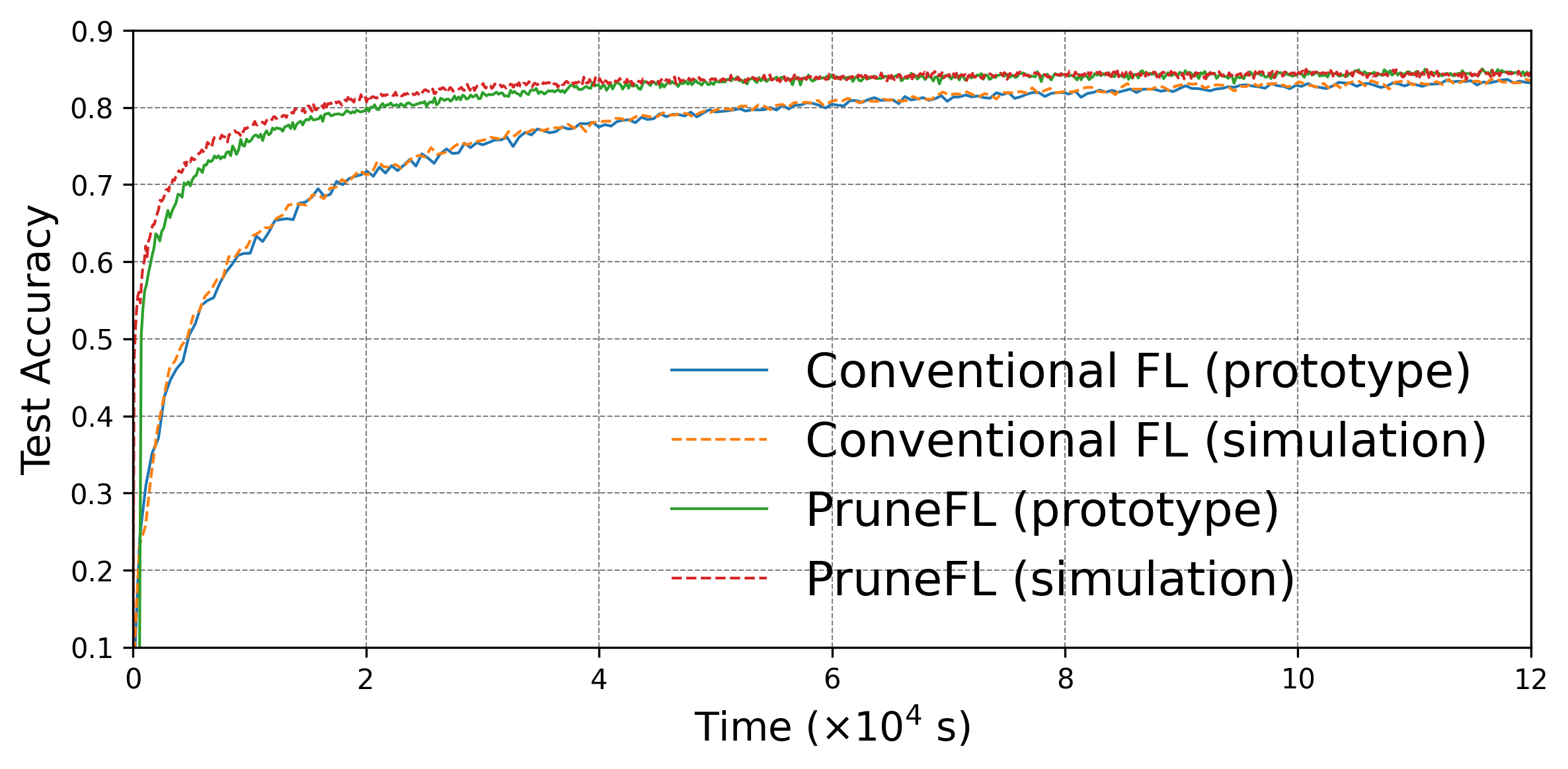}
    \vspace{-3mm}
    \caption{Comparing conventional FL and PruneFL with both prototype and simulation results (FEMNIST).}
    \vspace{-3mm}
    \label{fig:femnist_exp}
\end{figure}

\begin{figure*}[t]
\centering
\subfigure{
\hspace*{2.5mm}\includegraphics[height=4.9mm]{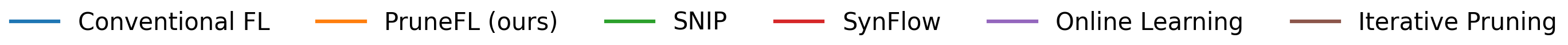}
}
\addtocounter{subfigure}{-1}
\vspace{-4.5mm}
\\
\subfigure[Conv-2 on FEMNIST]{
    \includegraphics[height=3.075cm]{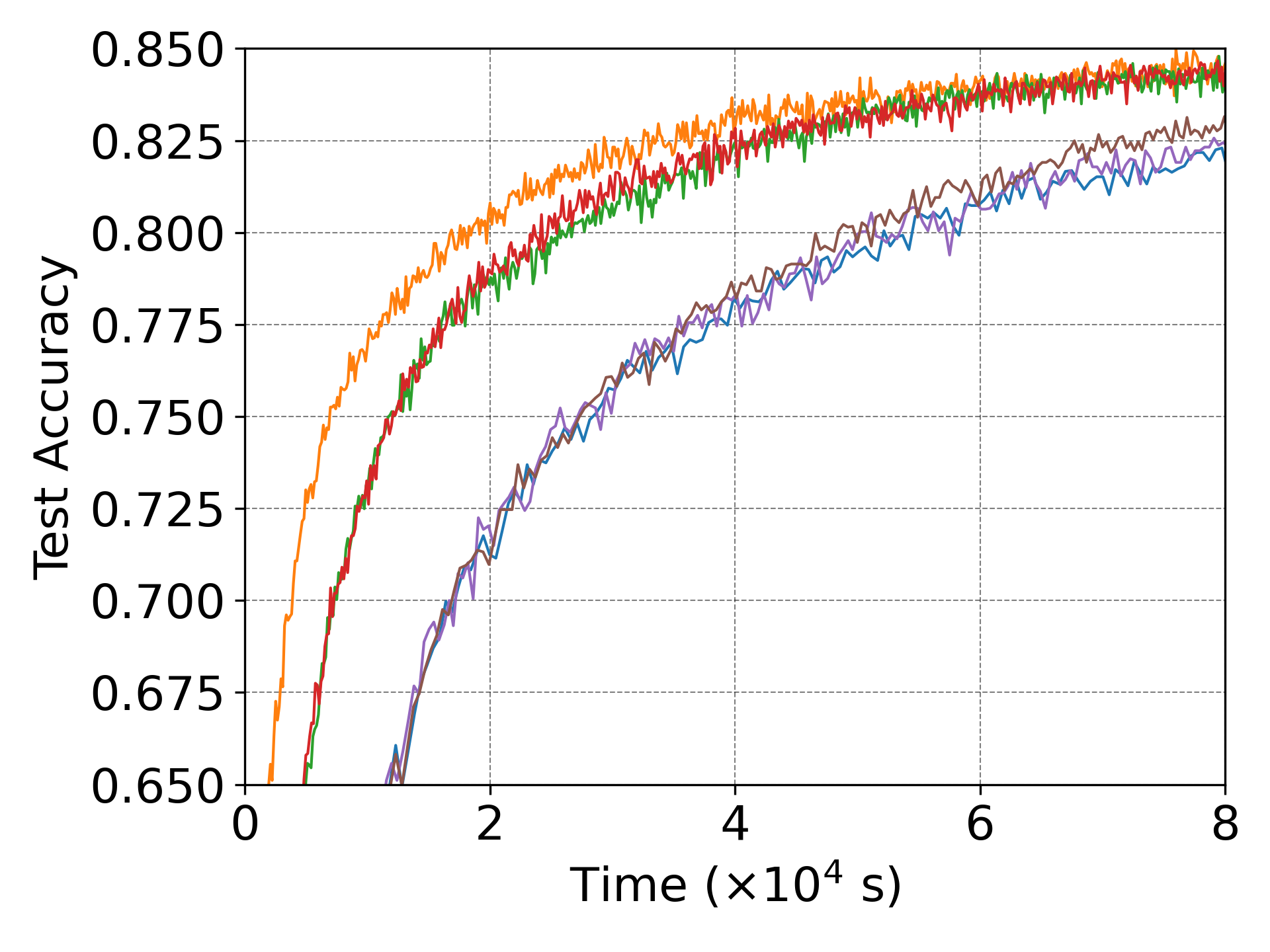}
    \label{fig:femnist_training}
}
\subfigure[VGG-11 on CIFAR-10]{
    \includegraphics[height=3.075cm]{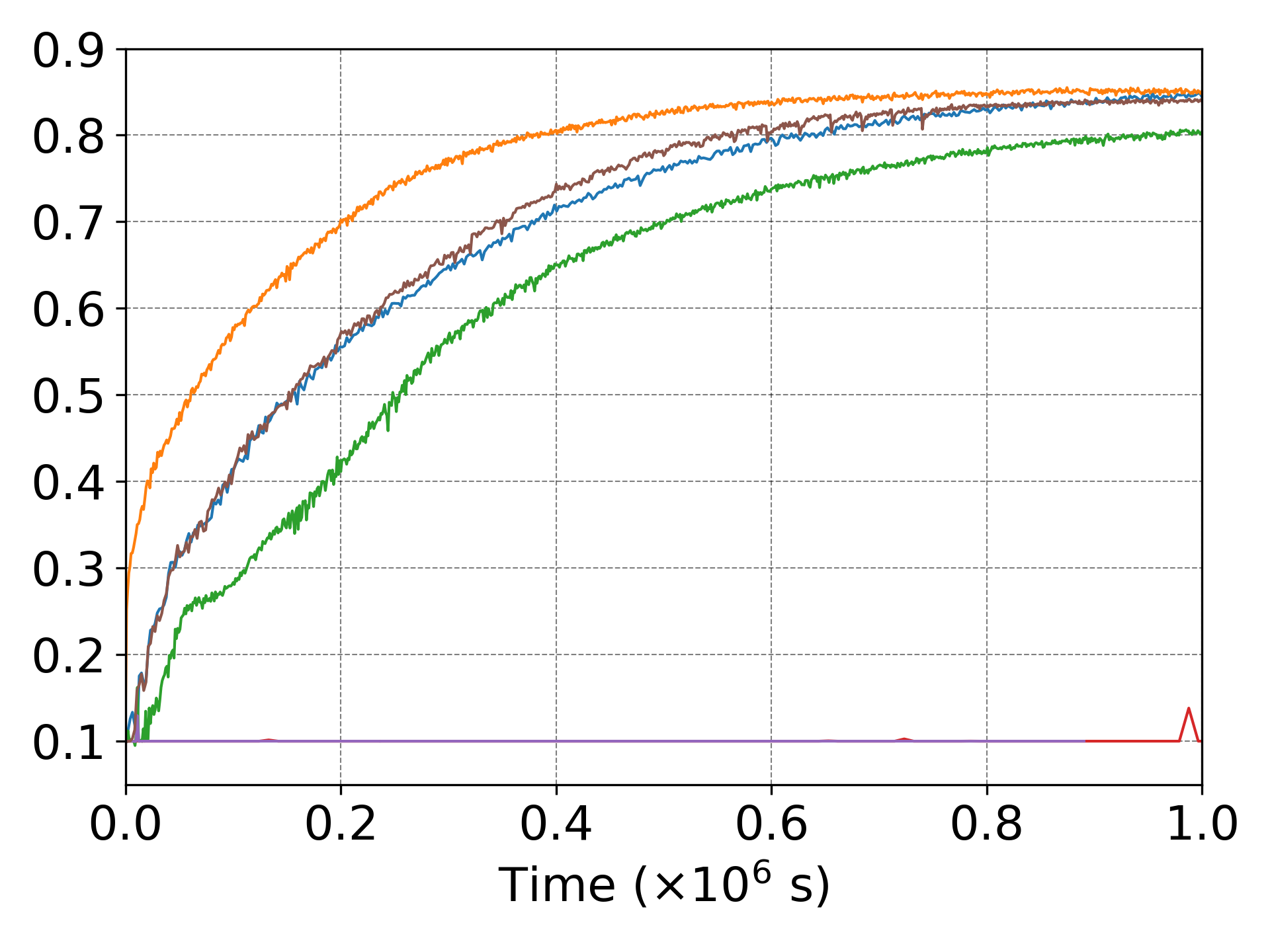}
    \label{fig:cifar10_training}
}
\subfigure[ResNet-18 on ImageNet-100]{
    \includegraphics[height=3.075cm]{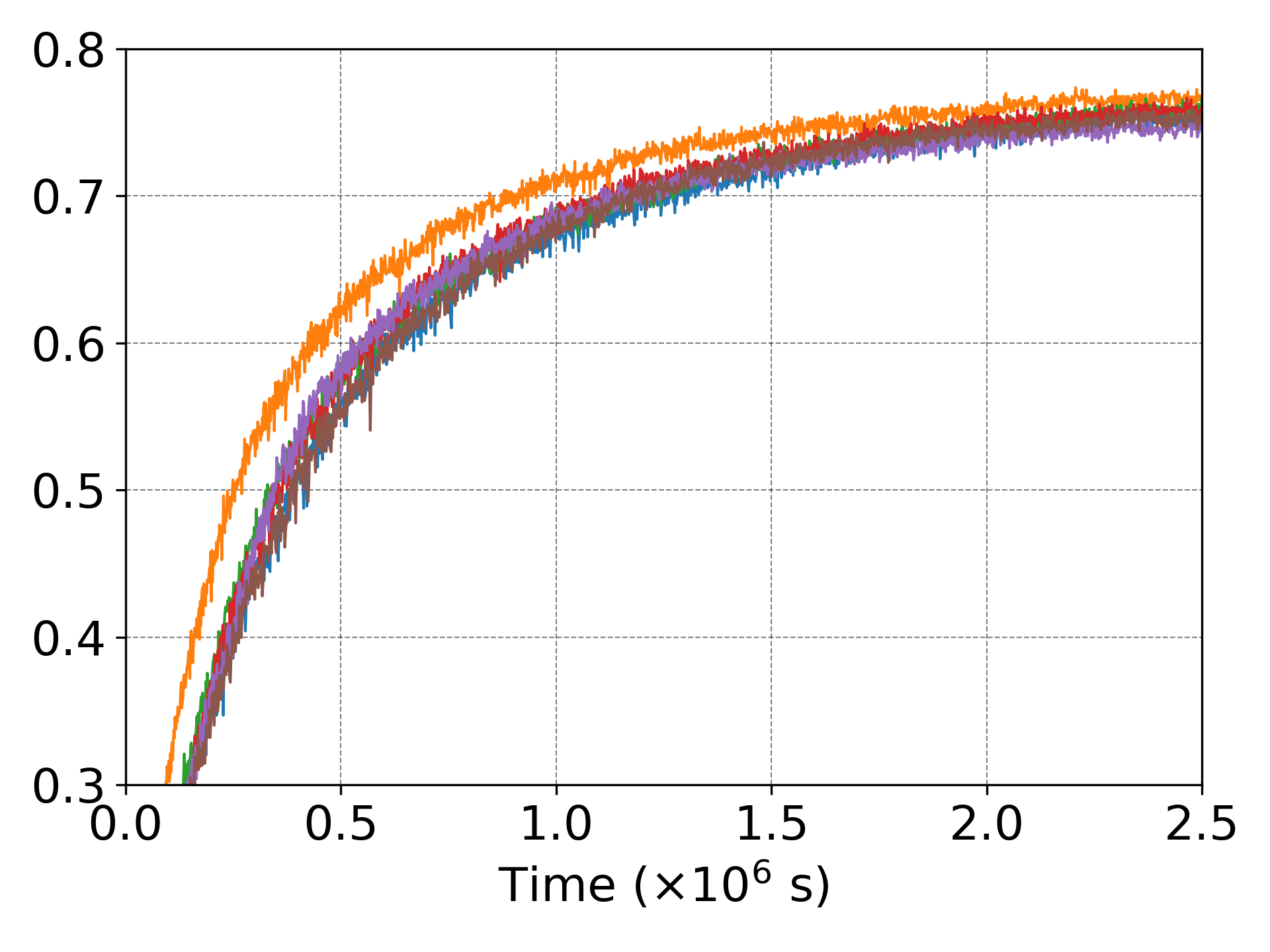}
    \label{fig:imagenet_training}
}
\subfigure[MobileNetV3-Small on CelebA]{
    \includegraphics[height=3.075cm]{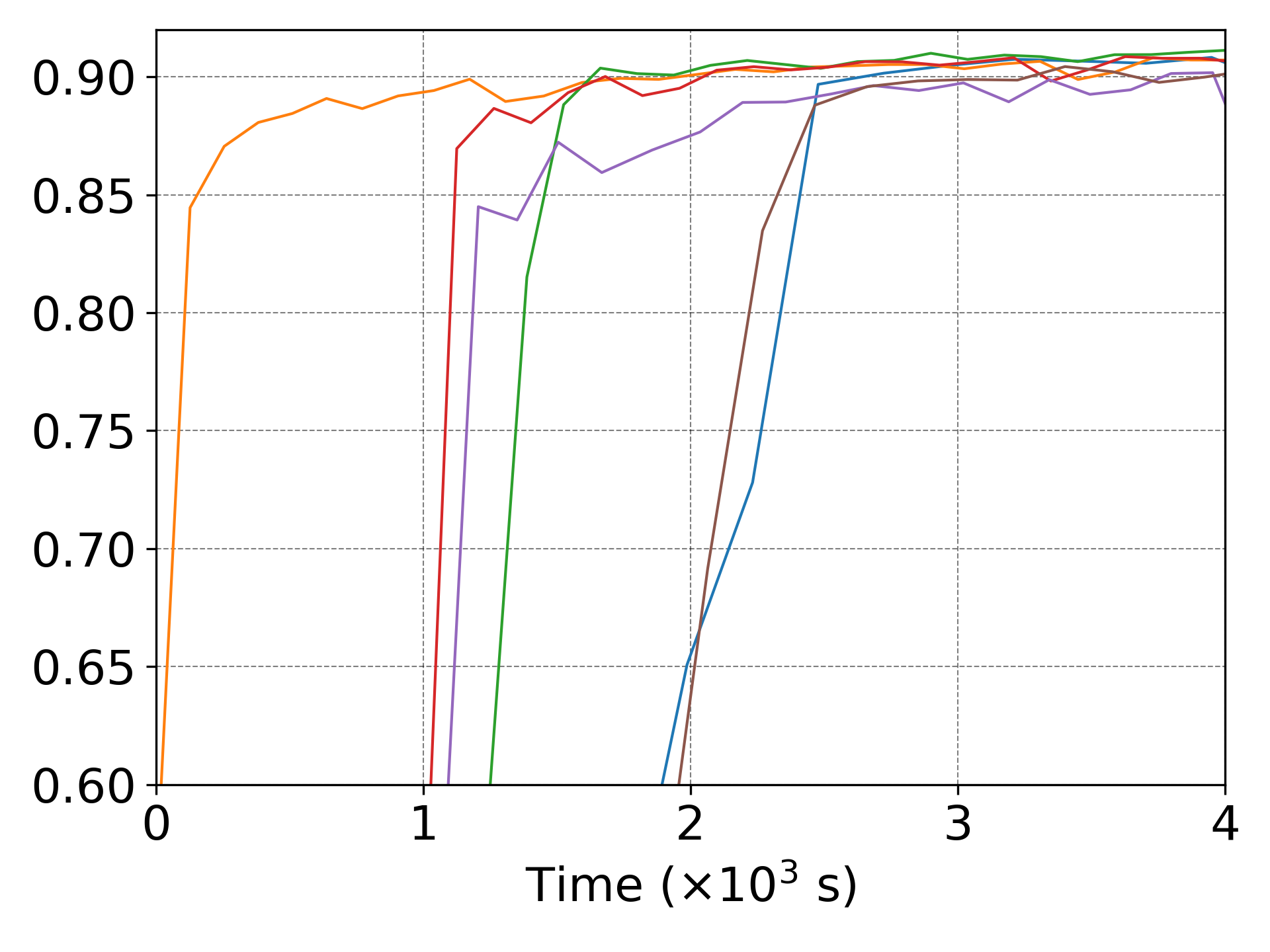}
    \label{fig:celeba_training}
}
\vspace{-2mm}
\caption{Test accuracy vs. time results of four datasets.}
\vspace{-4mm}
\label{fig:training}
\end{figure*}

\subsection{Time Measurement}\label{sec:time_measurement}

We present the time measurements of one FL round on the prototype system to show the effectiveness of model pruning on edge devices. We implement the full-sized Conv-2 model in dense form as well as the pruned models in sparse form at different densities, and measure the average elapsed time of FL on these pruned models involving both the server and clients over 10 rounds.

Fig.~\ref{fig:femnist_tm} shows the average total time, computation time, and communication time in one round as we vary the model density. Note that the model is in dense form at $100\%$ on the $x$-axis and sparse form elsewhere. We also plot the actual size of the parameters that are exchanged between server and clients in this figure for one FL round.

\textbf{Computation time.} We see from Fig.~\ref{fig:femnist_tm} that as the model density decreases, the computation time (for five local iterations) decreases from $11.24$~seconds per round to $6.34$~seconds per round. This reduction in computation time is moderate since our implementation of sparse computation only partially improves backward passes (see Section~\ref{sec:implementation}). Additionally, we plot in Fig.~\ref{fig:femnist_tm_inf} the total inference time and the number of floating-point operations (FLOPs) for $200$ data samples (see Appendix~\ref{appendix:flops_comp} for details of FLOPs computation). The inference time result shows a similar trends as in Fig.~\ref{fig:femnist_tm}, and the number of FLOPs keeps decreasing as we reduce the model size.

\textbf{Communication time.} Our implementation of sparse matrices reduces the storage requirement significantly (see Section~\ref{sec:implementation}). Compared to computation time, the decrease in the communication time is more noticeable. It drops from $35.88$~seconds per round to $1.04$~seconds per round. 

\textbf{Enabling FL on low-power edge device.} In addition, we ran experiments with the LeNet-300-100~\cite{lecun1998gradient} architecture, and we observed that when training the MNIST~\cite{lecun1998gradient} dataset on the full-sized, dense-form LeNet-300-100 model on Raspberry Pi \textit{version 3} (with 1~GB RAM, 32~GB SD card), the system dies during the first mini-batch due to resource exhaustion, while the models in sparse form can be trained. Thus, our approach of using sparse models enables model training on low-power edge devices, which is otherwise impossible on Raspberry Pi 3 in this experiment due to the device's resource limitation.

\subsection{Training Cost Reduction}\label{sec:time_reduction}

In the following, we study PruneFL's cost reduction in terms of both time and FLOPs for training.

\begin{figure}
  \centering
  \includegraphics[width=0.95\linewidth]{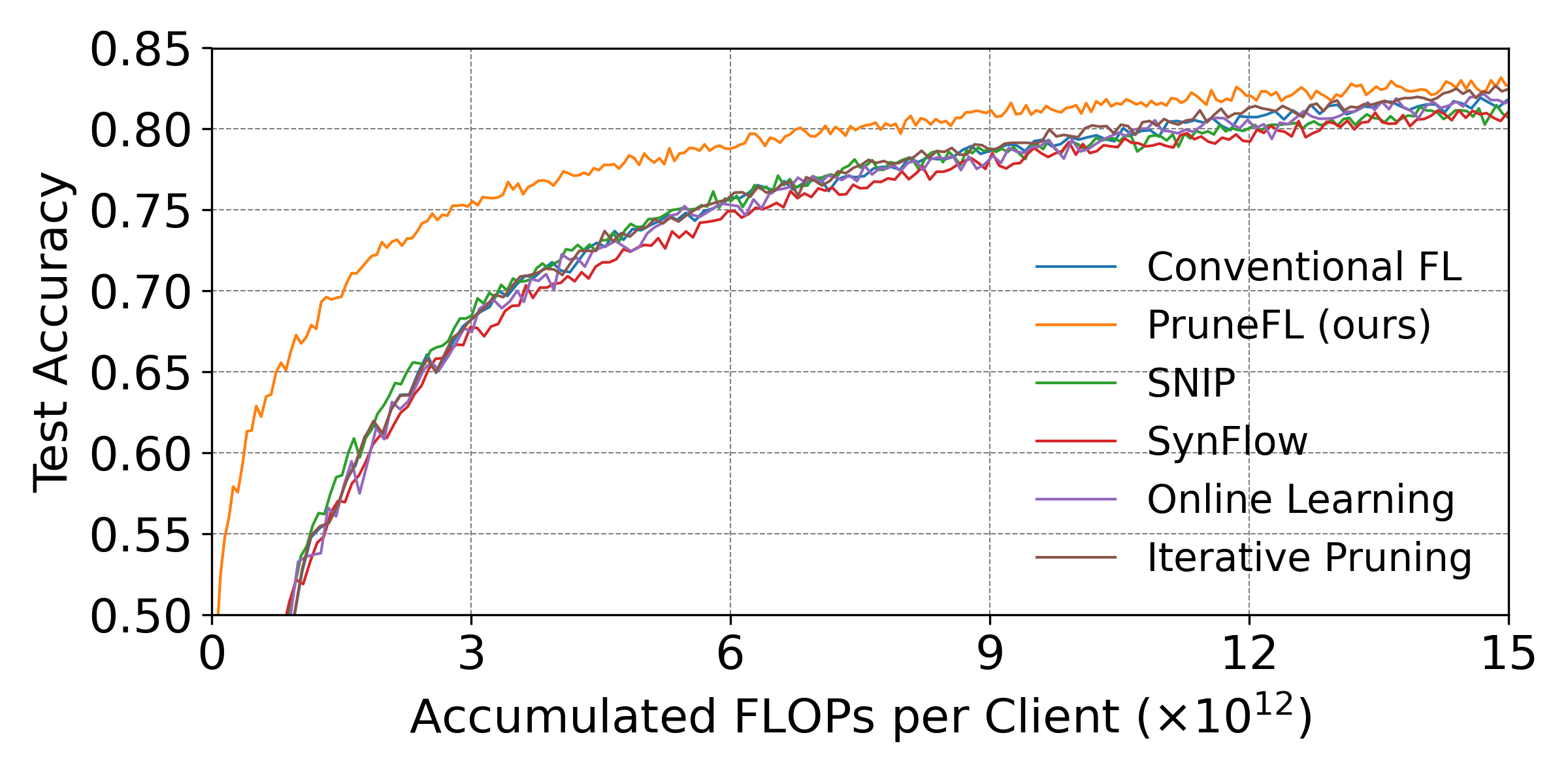}
  \vspace{-3mm}
  \caption{Test accuracy vs. accumulated FLOPs per client (FEMNIST).}
  \vspace{-3mm}
  \label{fig:femnist_flops}
\end{figure}

\begin{figure*}[t]
\centering
\subfigure{
\hspace*{2.5mm}\includegraphics[height=4.9mm]{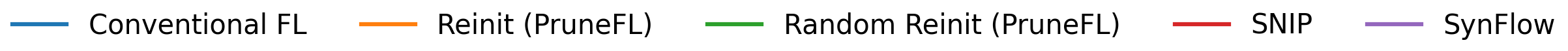}
}
\addtocounter{subfigure}{-1}
\vspace{-4.5mm}
\\
\subfigure[Conv-2 on FEMNIST]{
    \includegraphics[height=3.075cm]{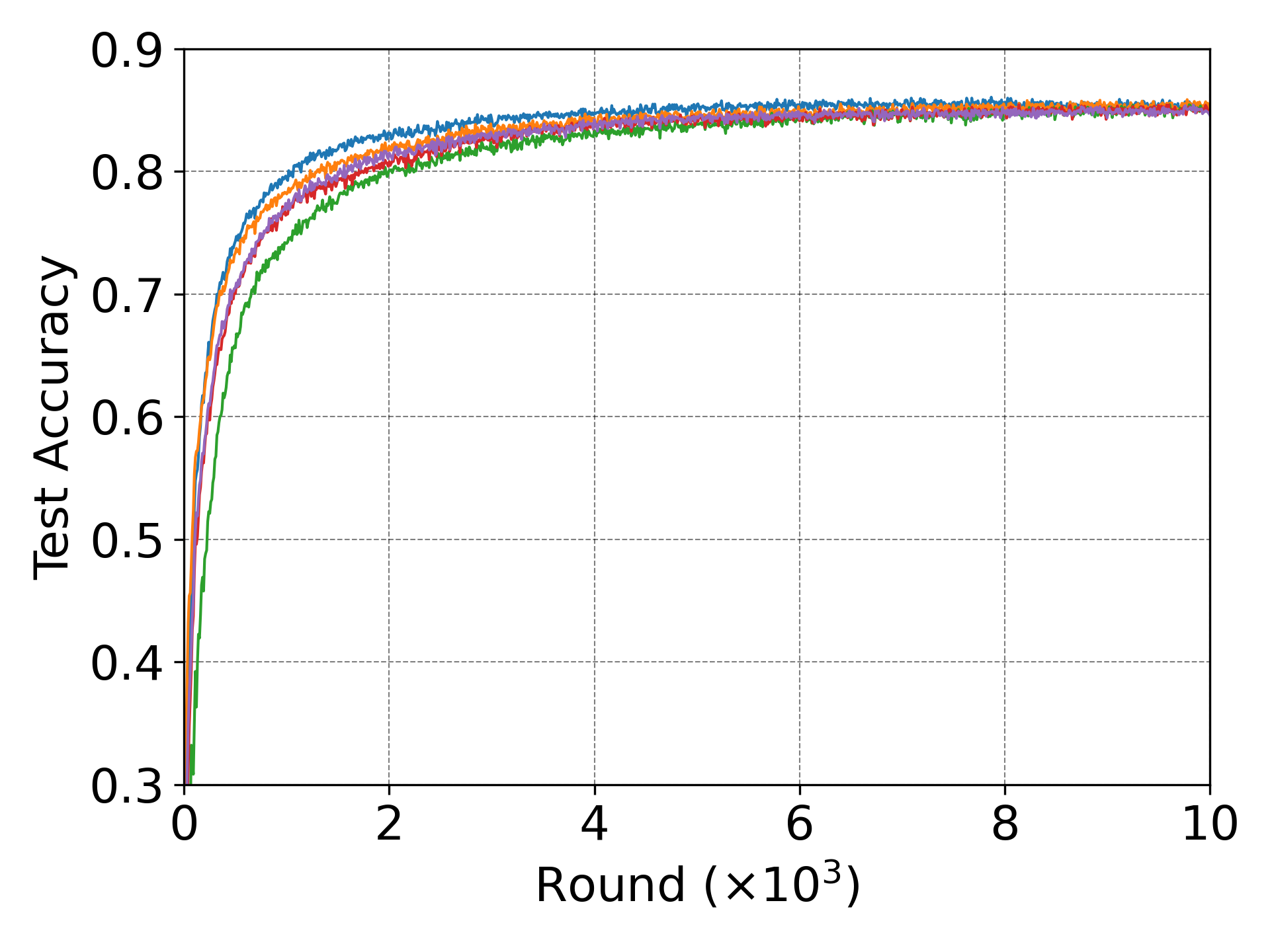}
    \label{fig:femnist_lt}
}
\subfigure[VGG-11 on CIFAR-10]{
    \includegraphics[height=3.075cm]{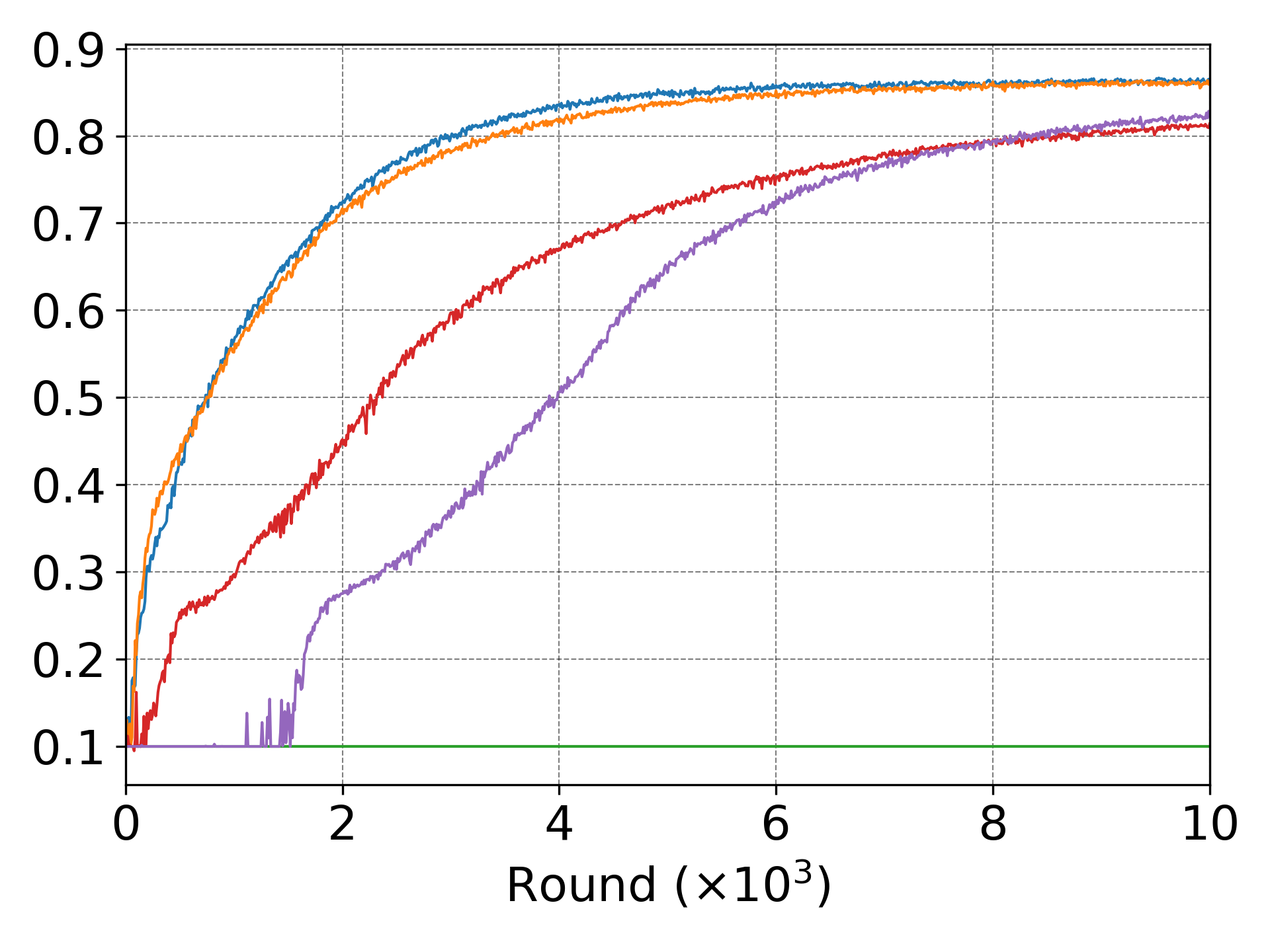}
    \label{fig:cifar10_lt}
}
\subfigure[ResNet-18 on ImageNet-100]{
    \includegraphics[height=3.075cm]{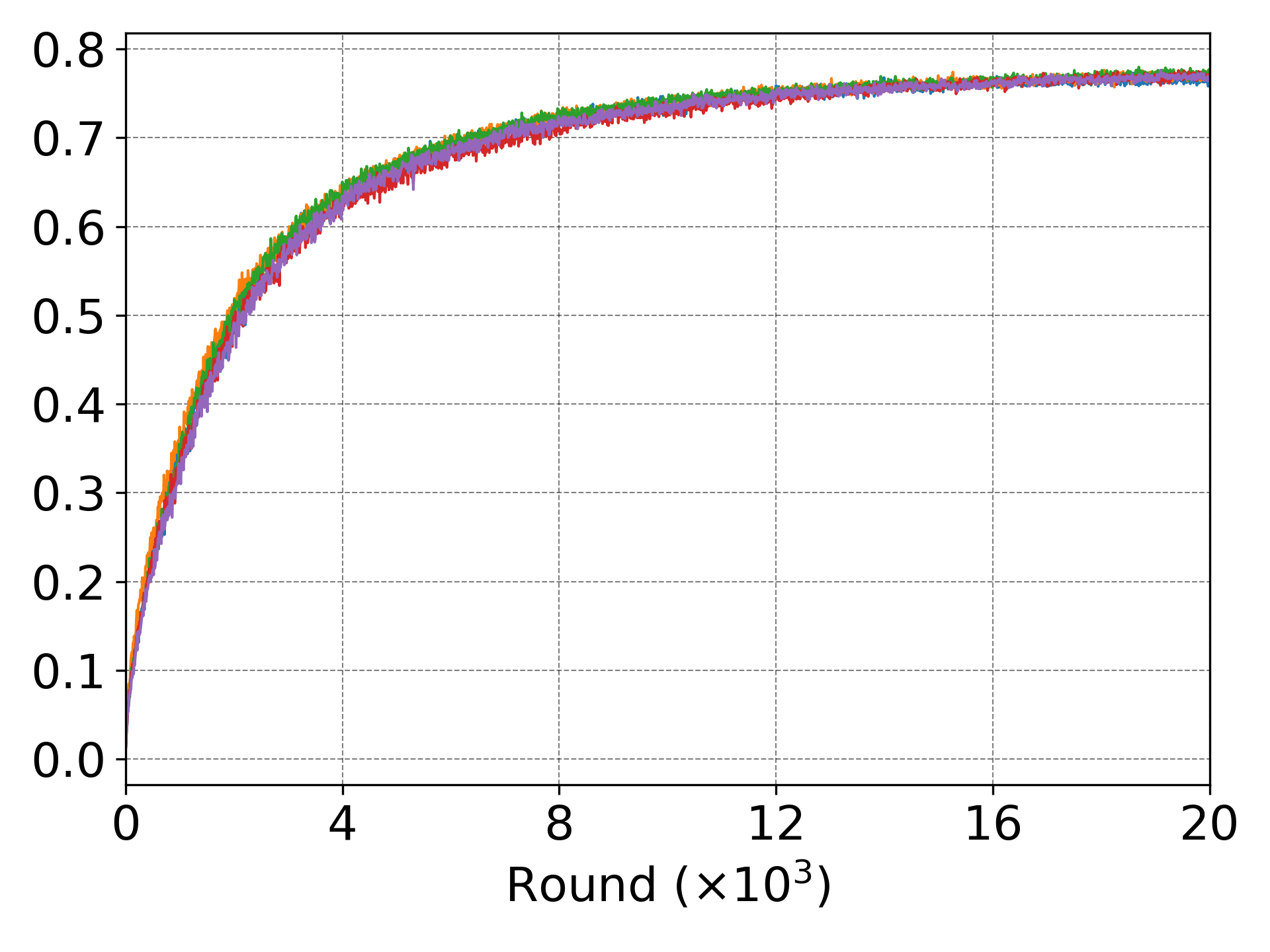}
    \label{fig:imagenet_lt}
}
\subfigure[MobileNetV3-Small on CelebA]{
    \includegraphics[height=3.075cm]{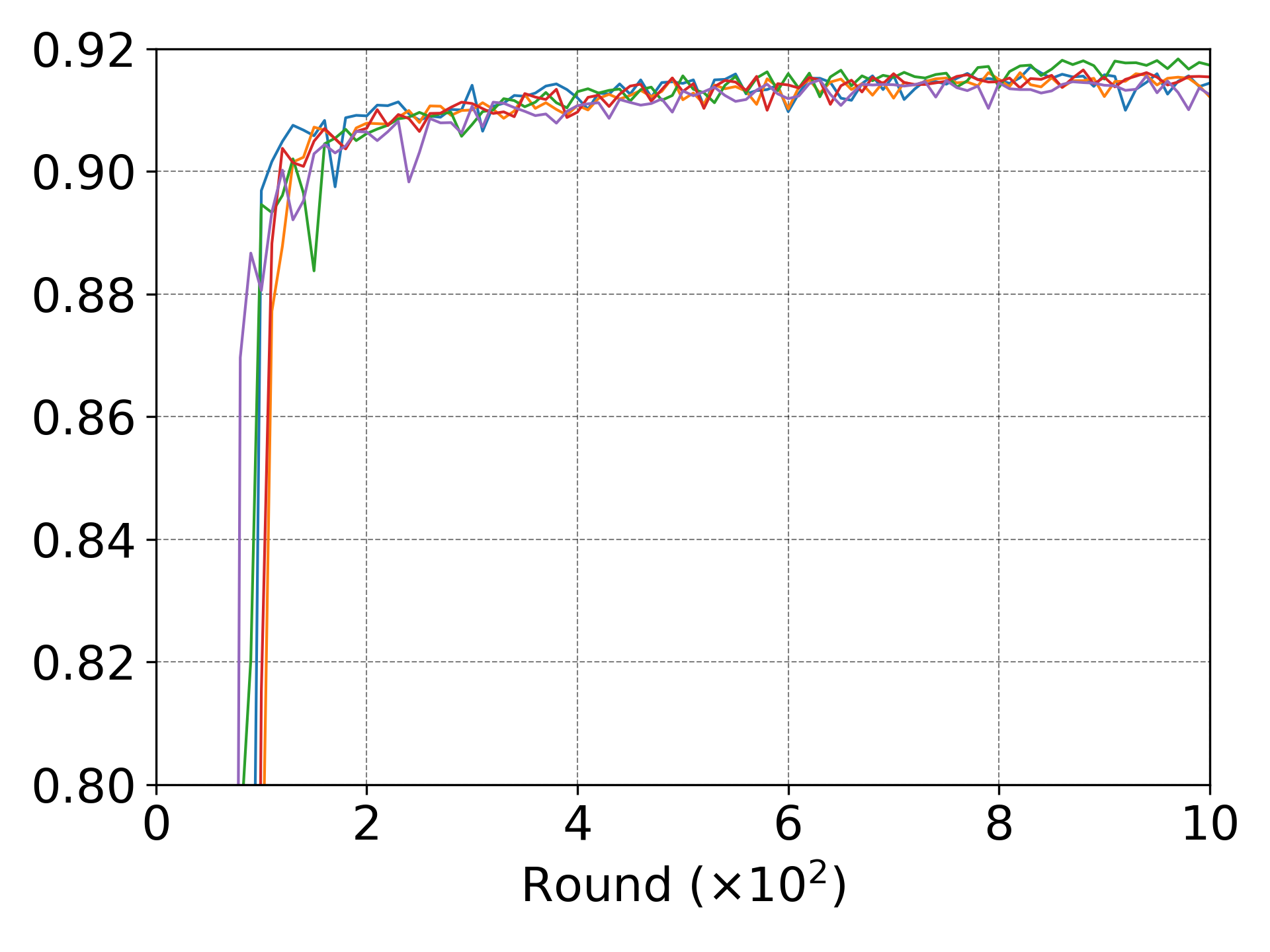}
    \label{fig:celeba_lt}
}
\vspace{-2mm}
\caption{Lottery ticket results of four datasets.}
\vspace{-4mm}
\label{fig:lottery}
\end{figure*}

\begin{figure}[t]
  \centering
  \includegraphics[width=0.95\linewidth]{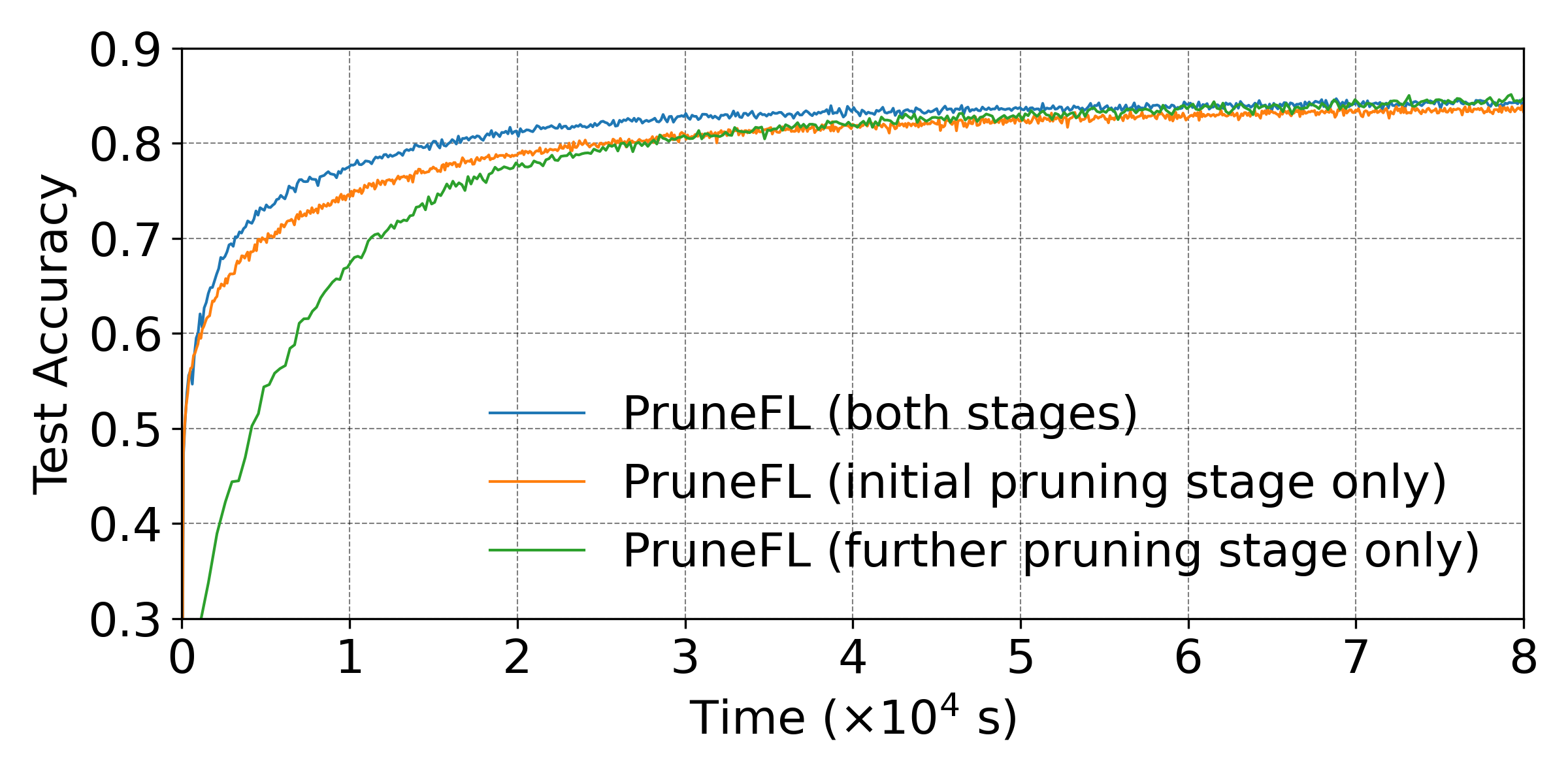}
  \vspace{-3mm}
  \caption{PruneFL with only one pruning stage (FEMNIST).}
  \vspace{-3mm}
  \label{fig:femnist_one_stage}
\end{figure}

\begin{figure*}[t!]
\centering
\subfigure{
\hspace*{2.5mm}\includegraphics[height=4.9mm]{figs/legend1.png}
}
\addtocounter{subfigure}{-1}
\vspace{-4.5mm}
\\
\subfigure[Conv-2 on FEMNIST]{
    \includegraphics[height=3.075cm]{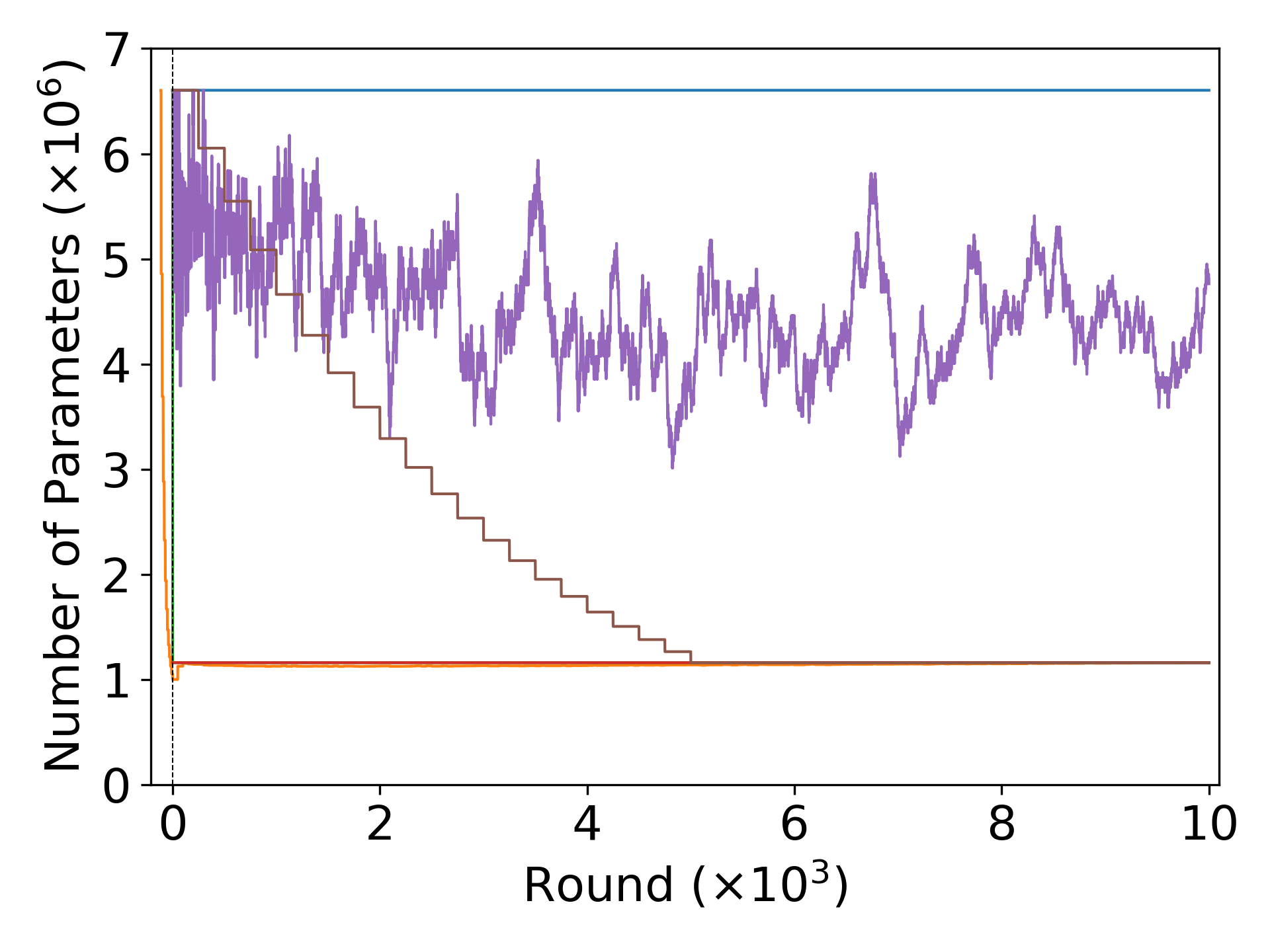}
    \label{fig:femnist_ms}
}
\subfigure[VGG-11 on CIFAR-10]{
    \includegraphics[height=3.075cm]{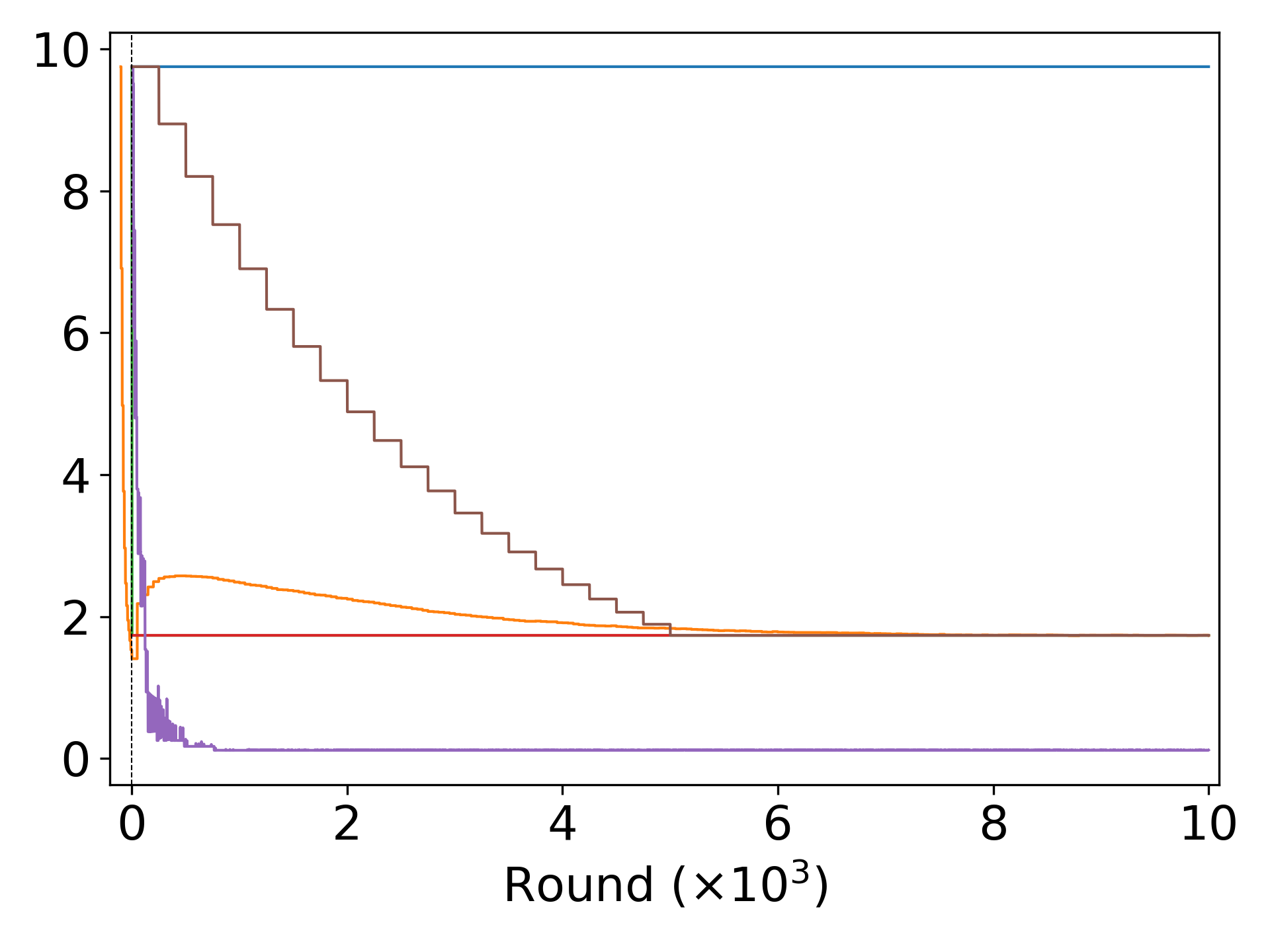}
    \label{fig:cifar10_ms}
}
\subfigure[ResNet-18 on ImageNet-100]{
    \includegraphics[height=3.075cm]{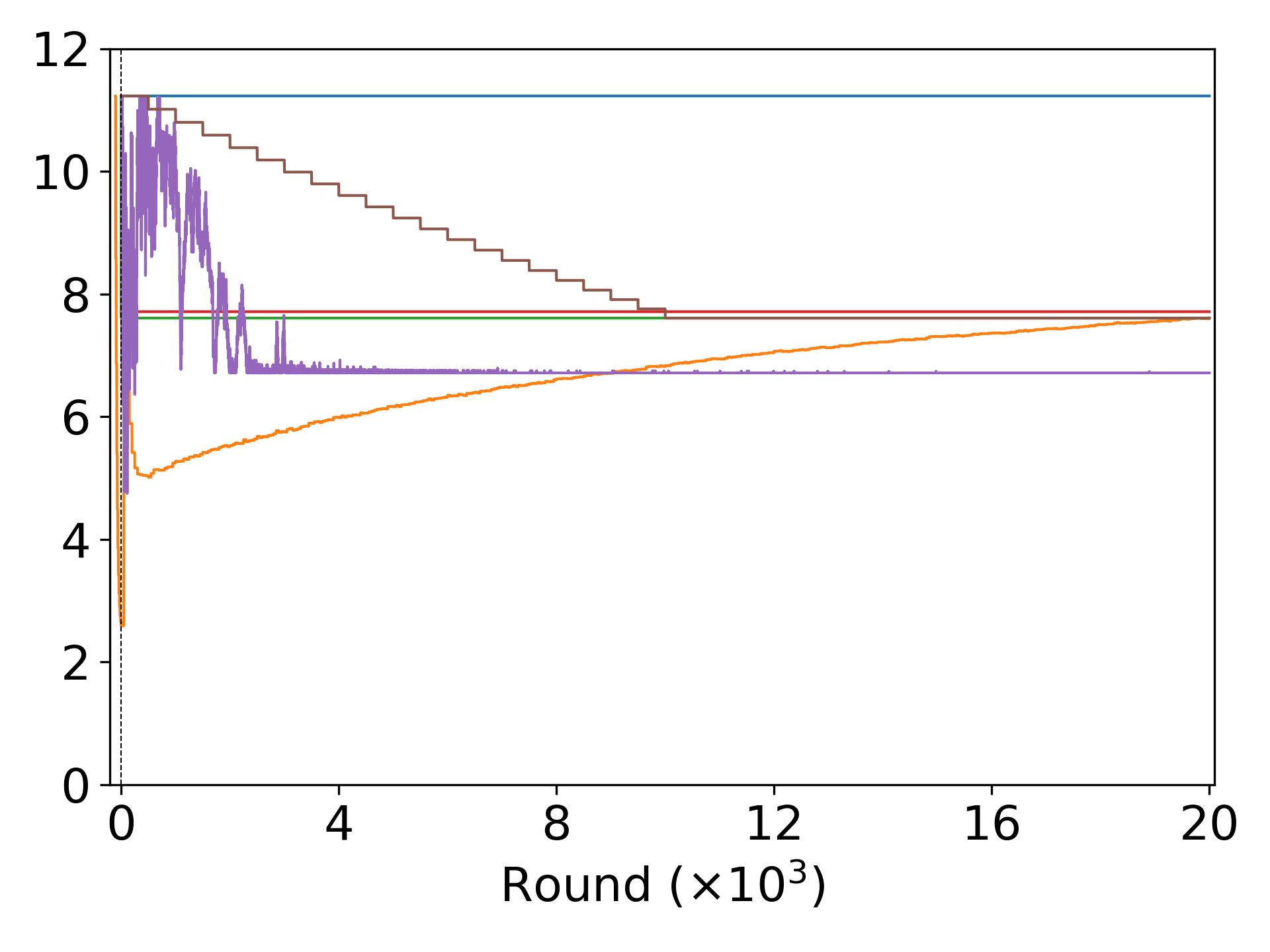}
    \label{fig:imagenet_ms}
}
\subfigure[MobileNetV3-Small on CelebA]{
    \includegraphics[height=3.075cm]{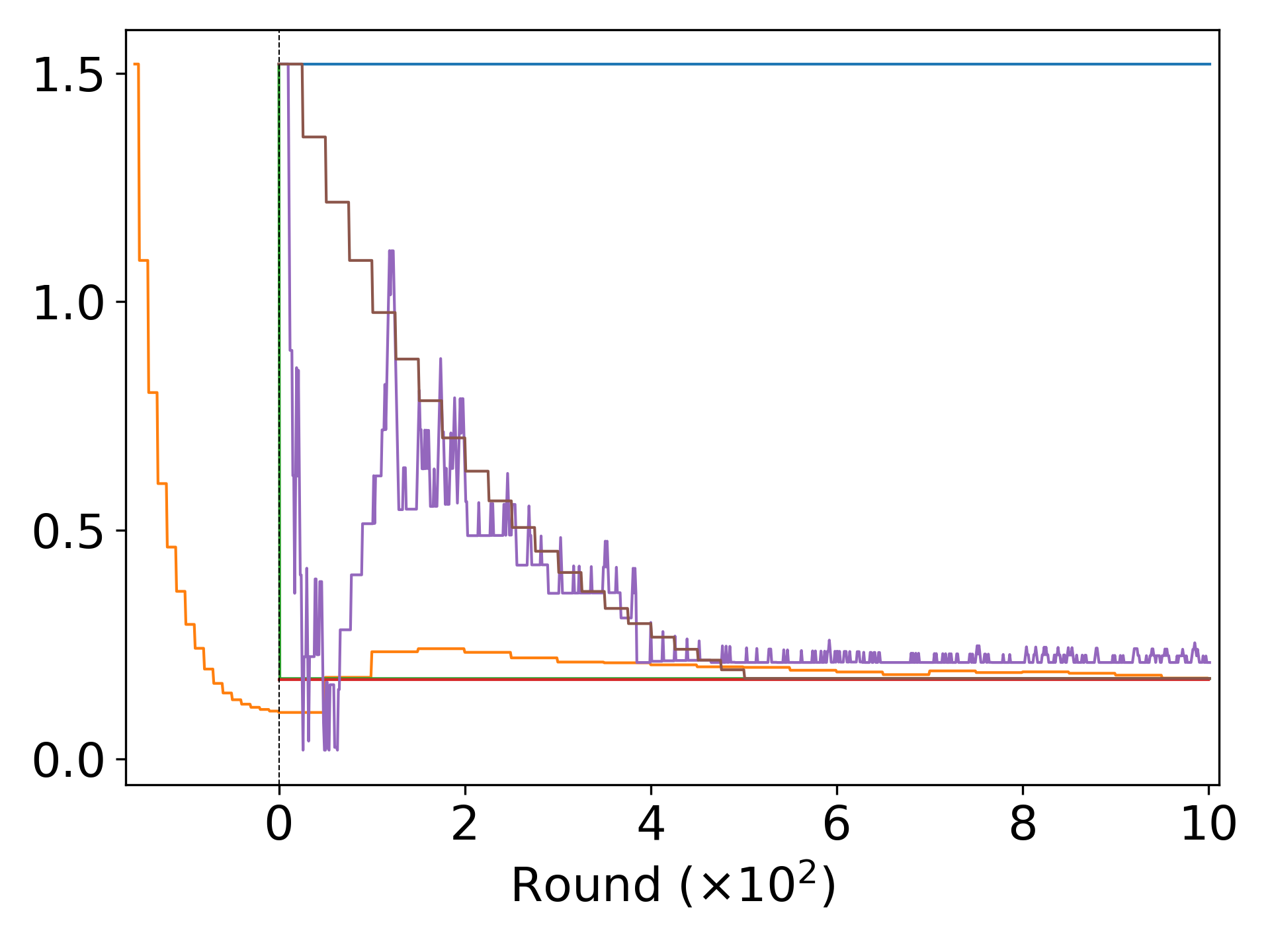}
    \label{fig:celeba_ms}
}
\vspace{-1mm}
\caption{Number of parameters vs. round for four datasets.}
\vspace{-4mm}
\label{fig:model_size}
\end{figure*}

\textbf{Comparing conventional FL and PruneFL.} 
Fig.~\ref{fig:femnist_exp} shows the test accuracy vs. time results on both the prototype and simulated systems, for Conv-2 on FEMNIST. The time for initial pruning of PruneFL is included in this figure, which is negligible (it takes less than $500$ seconds) compared to the further pruning stage. 
We see that PruneFL outperforms conventional FL by a significant margin.
Since the prototype and simulation results match closely, we present the simulation results in subsequent experiments due to their excessive training time on the prototype system.

\textbf{Training time reduction.}
In Fig.~\ref{fig:training}, we compare the test accuracy vs. time results for all datasets, models, and baselines. Clearly, PruneFL demonstrates a consistent advantage on training speed over baselines. Moreover, PruneFL always converges to similar accuracy achieved by conventional FL (see Appendix~\ref{appendix:convergence_acc}). Other methods may have suboptimal performance, e.g., SNIP does not converge to conventional FL's accuracy with CIFAR-10. We also observe that some approaches such as online learning and SynFlow in Fig.~\ref{fig:cifar10_training} always stay at the random guess accuracy. The reason could be that such approaches are unstable and can get stuck in local optimal points at the beginning of training.

\textbf{Training FLOPs reduction.}
Although we observe that the training time on Raspberry Pis is relatively consistent across different models and tasks, there are still factors that can affect the training time (e.g., environment temperature). To further validate our approach's advantage on accelerating training, we present the results on test accuracy vs. accumulated FLOPs per client for FEMNIST in Fig.~\ref{fig:femnist_flops}. We find that this result shares the same characteristics with Fig.~\ref{fig:femnist_training} in terms of acceleration (we present only one set of results here due to the similarity).

\textbf{Time and FLOPs to reach target accuracy.} Table~\ref{tab:target_acc} lists the time and accumulated FLOPs per client that an algorithm first reaches a certain accuracy with FEMNIST. PruneFL takes less than $1/3$ of time compared to conventional FL to reach $80\%$ accuracy, and it also saves more than $33\%$ of time (more than 2 hours) compared to SNIP and SynFlow. The savings of FLOPs are similar.

{\renewcommand{\arraystretch}{1.1}
\begin{table}[t]
\caption{Time and accumulated FLOPs per client to reach target accuracy (FEMNIST).}
\centering
\label{tab:target_acc}
{
    \begin{tabular}{ccc}
    \hline
    Approach       &  \begin{tabular}[c]{@{}c@{}} Time (FLOPs) to reach \\ 70\% accuracy \end{tabular}  &  \begin{tabular}[c]{@{}c@{}} Time (FLOPs) to reach \\ 80\% accuracy\end{tabular}      \\  \hline \hline
    
    Conventional FL & 17,929~s ($3.5$ TFLOPs)  & 52,153~s ($10.5$ TFLOPs)  \\ 
    \textbf{PruneFL (ours)} & \textbf{3,187~s ($1.6$ TFLOPs)} & \textbf{15,009~s ($6.8$ TFLOPs)}  \\
    SNIP & 6,801~s ($3.3$ TFLOPs) & 22,467~s ($11.7$ TFLOPs) \\
    SynFlow & 7,132~s ($3.6$ TFLOPs) & 22,327~s ($12.3$ TFLOPs) \\
    Online & 18,042~s ($3.5$ TFLOPs) & 54,593~s ($10.7$ TFLOPs) \\
    Iterative & 17,495~s ($3.5$ TFLOPs) & 46,521~s ($10.1$ TFLOPs) \\
    \hline
    \end{tabular}
}
\end{table}
}

\textbf{Comparing with additional baselines.} To avoid bottlenecks, our algorithm and implementation ensures that all components in PruneFL, including communication, computation, and reconfiguration, are orchestrated and inexpensive. For this reason, some approaches in the literature that are not specifically designed for the edge computing environment with low-power devices may perform poorly if applied to our system setup, as we illustrate next.

Considering computation time, PruneTrain~\cite{lym2019prunetrain} applies regularization on every input and output channel in every layer. When the same regularization is applied to our system, we find that the computation time (using FEMNIST and Conv-2) takes $17.65$~seconds per round, which is a $57\%$ increase compared to PruneFL. 

Considering communication time, dynamic pruning with feedback (DPF)~\cite{Lin2020Dynamic} maintains a full-sized model, and clients have to upload full-sized gradients to the server (but only download a subset of model parameters) in every round. Thus, assuming unit model size and  model density $d$, the communication cost per round, including both uploading and downloading, is $1+d$. In comparison, clients in PruneFL only upload the full-sized model to the server at a reconfiguration round (every $50$ rounds in our experiments), and always exchange pruned models otherwise. This gives an average cost of $\frac{(1+d) + 2\times 49d}{50}=0.02+1.96d$ including both uploading and downloading. For instance, when the model density is $10\%$, DPF incurs $5.1\times$ communication cost compared to PruneFL. Finally, our reconfiguration algorithm (Algorithm~\ref{alg:greedySearchNewLinear}) runs in quasi-linear time, making it possible to be implemented on edge devices.

\subsection{Finding a Lottery Ticket}\label{sec:finding_lt}
Unlike some existing pruning techniques such as SNIP~\cite{lee2018snip}, dynamic pruning~\cite{Lin2020Dynamic}, and SynFlow~\cite{tanaka2020synflow}, PruneFL finds a lottery ticket (although not necessarily the smallest). In Fig.~\ref{fig:lottery}, FL with the reinitialized pruned model obtained from PruneFL learns comparably fast as FL with the original model, in terms of test accuracy vs. \textit{round}, confirming that they are lottery tickets. In Fig.~\ref{fig:cifar10_lt}, the Random Reinit curve stays at the random guess accuracy. This is not surprising since the lottery ticket, i.e., the final pruned model found by PruneFL, needs to be reinitialized to their original values to learn comparably fast as the full-sized model~\cite{frankle2018the}. When reinitialized with different values, the training of the ``lottery ticket'' can be suboptimal. In this experiment, it is stuck at the beginning of training.

Fig.~\ref{fig:femnist_one_stage} compares the test accuracy vs. round curves of PruneFL with alternative methods that either only includes initial pruning (at a single client) or only includes further pruning (during FL). It shows that the model obtained from the initial pruning stage does not converge to the optimal accuracy, and only performing further pruning without initial pruning results in a slower learning speed. PruneFL with both stages avoids drawbacks from methods that includes only one stage. Furthermore, the model obtained from initial pruning is not a lottery ticket of the original model, while PruneFL with only further pruning or both pruning stages find a lottery ticket. 

Therefore, one can view PruneFL as a two-stage procedure to find a lottery ticket of the given model, which is in line with our claims in Section~\ref{sec:proposed_approach}. The ability that we can find lottery tickets is useful when we need to retrain a pruned model on slightly different but similar datasets~\cite{morcos2019one}.

\subsection{Model Size Adaptation}\label{sec:size_adaptation}

An illustration of the change in model size is shown in Fig.~\ref{fig:model_size}. The small negative part of the $x$-axis shows the initial pruning stage which is unique to PruneFL. Since there is no notion of ``round'' in the initial pruning stage, we consider five local iterations in this stage as one round, which is consistent with our FL setting. Conventional FL always keeps the full model size. The model sizes provided by online learning is unstable. It fluctuates in initial rounds due to its exploration. SNIP and SynFlow prune the initial model to the target size in a one-shot manner at the beginning of training. Iterative pruning gradually reduces the model size until reaching the target. 
We notice that PruneFL also discovers the degree of overparameterization. Empirically, Conv-2, an overparameterized model for FEMNIST, converges to a small density ($13.4\%$), while ResNet-18, an underparameterized model for ImageNet-100, converges to a density of around $67.7\%$.

It is worth mentioning that finding a proper target density for pruning is non-trivial. Usually, foresight pruning methods such as \cite{lee2018snip} and \cite{wang2019picking} prune the model to the (manually selected) density before training. Fig.~\ref{fig:femnist_snip_diff_size} shows two cases where we use SNIP to prune Conv-2 (with FEMNIST) to 30\% and 1\%, the training speed becomes slower, and if the density is too small (1\%), the sparse model cannot converge to the same accuracy as the original model. In comparison, PruneFL automatically determines a proper density.

\begin{figure}
\centering
  \includegraphics[width=0.95\linewidth]{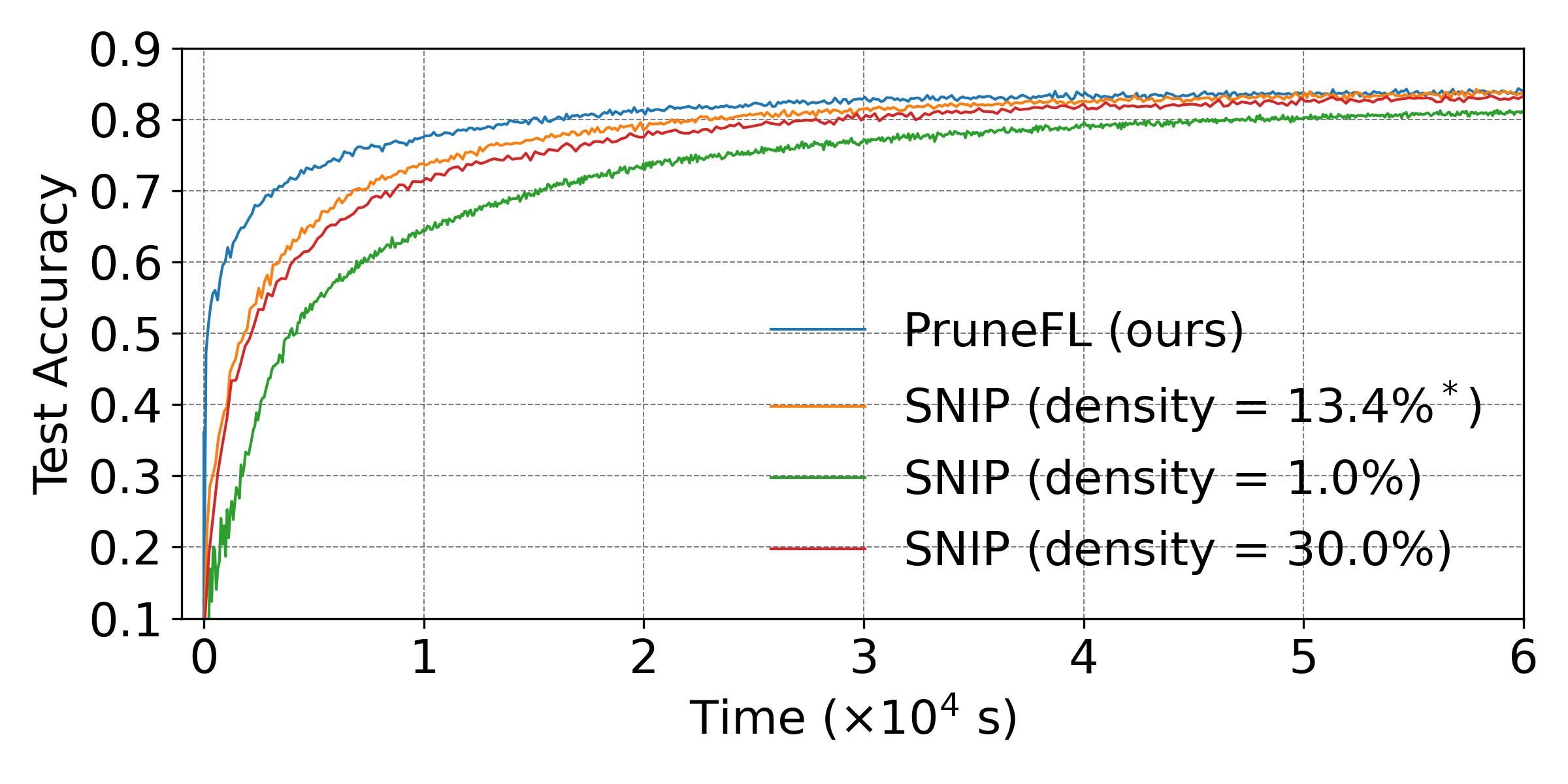}
  \vspace{-3mm}
  \caption{SNIP with different densities (FEMNIST).}
  \vspace{-3mm}
  \label{fig:femnist_snip_diff_size}
\end{figure}

\begin{figure}[t]
  \centering
  \includegraphics[width=\linewidth]{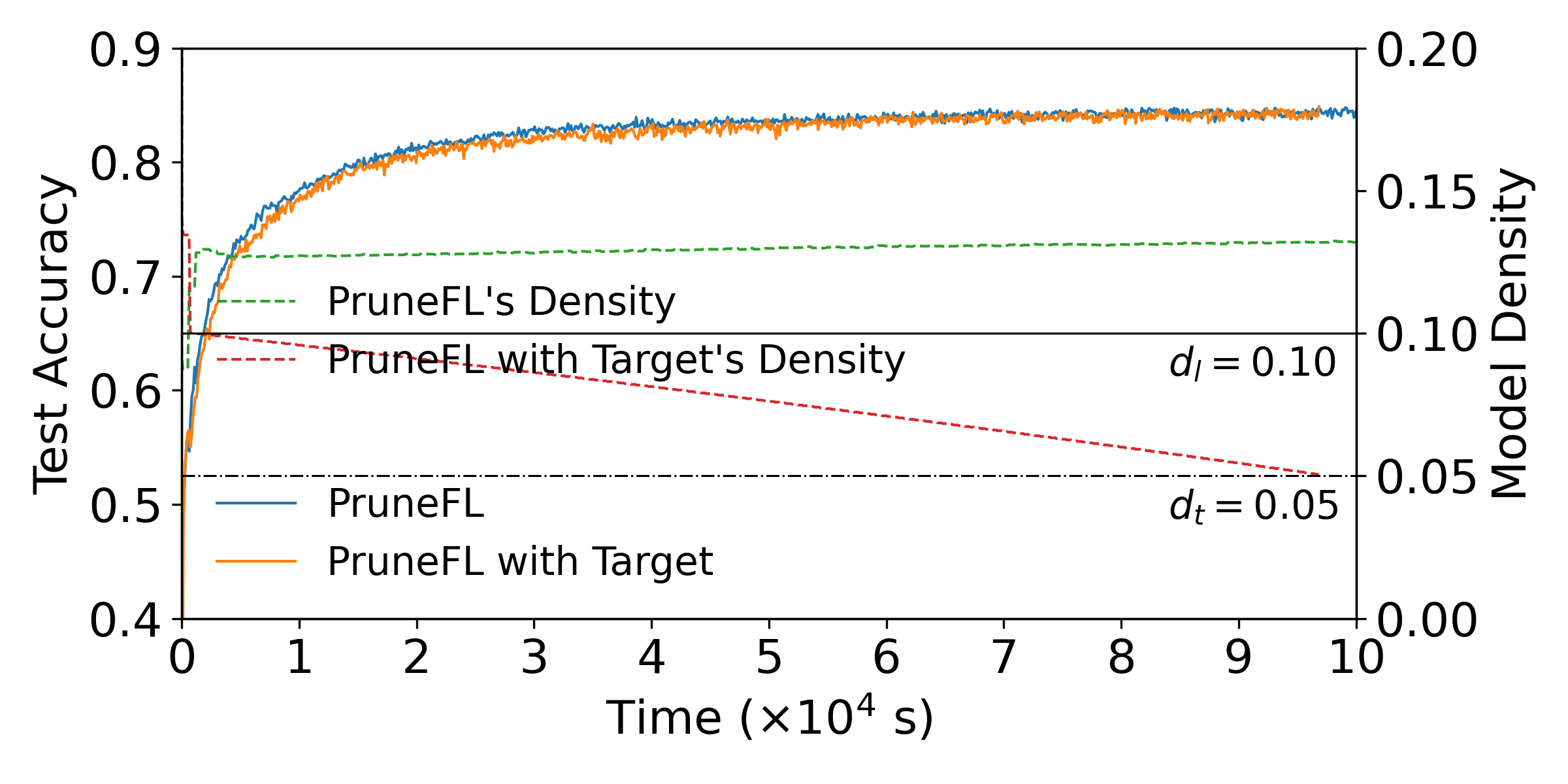}
  \vspace{-8mm}
  \caption{Training with limited and targeted size (FEMNIST).}
  \vspace{-3mm}
  \label{fig:femnist_target}
\end{figure}

\subsection{Training with Limited/Targeted Model Sizes}
There are cases where a hard limit on the maximum model size or a targeted final model size (or both) is desired. For example, if some of the client devices have limited memory or storage so that only a partial model can be loaded, then the model size must be constrained after initial pruning. Targeted model size may be needed in the case where the goal of the FL system is to obtain a model with a certain small size at the end of training.

Next, we present an extended PruneFL with limited and targeted model sizes, and show that with reasonable constraints, PruneFL still achieves good results. We use a heuristic way to limit the model size: we stop Algorithm~\ref{alg:greedySearchNewLinear} early when the number of remaining parameters reaches the maximum allowed size, and we schedule the maximum size of the model to decrease linearly. Assuming $d_l$ is the density limit, $d_t$ is the target model density at the end of the further pruning stage ($d_t \leq d_l$), and PruneFL is run for $r_\text{max}$ rounds after the initial pruning stage, then the maximum density at round $r$ is $d_{\text{max}}(r) = \big(r\cdot d_t + (r_\text{max}-r)\cdot d_l \big)/r_\text{max}$. The results of selecting $d_l = 10\%$ and $d_t = 5\%$ for Conv-2 on FEMNIST are given in Fig.~\ref{fig:femnist_target}. We see that if we do not impose these model size constraints, PruneFL exceeds the density limits $d_l = 10\%$ and $d_t = 5\%$ defined in this example, and obtains a model that is much larger than the target density $d_t = 5\%$ at the end of training. We see that PruneFL with effective size limit and target still achieves fast convergence and similar convergence accuracy, and the model size is always limited below the threshold $d_l = 10\%$ and reaches the target density $d_t = 5\%$ at the end of training.

\subsection{Relative Importance Between Layers}
Similar to SynFlow~\cite{tanaka2020synflow} and SNIP~\cite{lee2018snip}, our algorithm discovers the relative importance between layers automatically. Taking CIFAR-10 on VGG-11 as an example, the densities of convolutional layers and fully-connected layers at convergence are listed in Table~\ref{tab:layer_density} in a sequential manner. We observe that the input and output layers are not pruned to a low density, indicating that they are relatively important in the neural network architectures. This also agrees with the pruning scheme in~\cite{frankle2018the}, where the authors empirically set the pruning rate of the output layer to a small percentage or even to zero. Some large convolutional layers, such as the last two convolutional layers in VGG-11, are identified as redundant, and thus have small densities at convergence.

{\renewcommand{\arraystretch}{1.3}
\begin{table}[t]
\caption{Density of each layer at convergence (no C.S).}
\label{tab:layer_density}
\footnotesize
\centering
\begin{tabular}{ccccc}
\hline
Experiment & VGG-11 on CIFAR-10 \\
\hline \hline
Convolutional & $1.0$, $0.80$, $0.84$, $0.84$, $0.50$, $0.07$, $0.01$, $0.03$ \\
\hline
Fully-connected & $0.14$, $0.15$, $0.7$  \\
\hline
\end{tabular}
\vspace{-2mm}
\end{table}
}

\section{Conclusion}\label{sec:conclusion}
We have proposed PruneFL for FL in edge/mobile computing environments, where the goal is to effectively reduce the size of neural networks so that resource-limited clients can train them within a short time.
Our PruneFL method includes initial and further pruning stages, which improves the performance compared to only having a single stage.
PruneFL also includes a low-complexity adaptive pruning method for efficient FL, which finds a desired model size that can achieve a similar prediction accuracy as the original model but with much less time. Our experiments on Raspberry Pi devices confirm that we improve the cost-efficiency of FL while
obtaining a lottery ticket.
Our method can be applied together with other compression techniques, such as quantization, to further reduce the communication overhead.

\section*{Acknowledgment}

This work was partially supported by the U.S. Office of Naval Research under
Grant N00014-19-1-2566, the U.S. National Science Foundation AI Institute Athena under Grant CNS-2112562, and the U.S. Army Research Laboratory and the U.K. Ministry of Defence under Agreement Number W911NF-16-3-0001. The views and conclusions contained in this document are those of the authors and should not be interpreted as representing the official policies, either expressed or implied, of U.S. Office of Naval Research, the U.S. National Science Foundation, the U.S. Army Research Laboratory, the U.S. Government, the U.K. Ministry of Defence or the U.K. Government. The U.S. and U.K. Governments are authorized to reproduce and distribute reprints for Government purposes notwithstanding any copyright notation hereon.
V. Valls has also received funding from the European Union’s Horizon 2020 research and innovation programme under the Marie Skłodowska-Curie grant agreement No. 795244.

\ifCLASSOPTIONcaptionsoff
  \newpage
\fi



\clearpage

\onecolumn
\numberwithin{equation}{subsection}
\counterwithin{figure}{subsection}
\counterwithin{algocf}{subsection}
\counterwithin{table}{subsection}

\appendix

\subsection{Extension to Non-linear \texorpdfstring{$T(\mathcal{M})$}{T(M)}}
\label{appendix:non-linear-T}

For the case where $T(\mathcal{M})$ is non-linear, but a general monotone and positive set function instead, we can still find a local optimal solution to (\ref{eq:optimizationProblem}) using Algorithm~\ref{alg:greedySearchNew}. We can see that the complexity of Algorithm~\ref{alg:greedySearchNew} is $O(|\mathcal{P}|^2)$.

\begin{algorithm}[h]
\caption{Solving (\ref{eq:optimizationProblem}), general $T(\cdot)$}
\label{alg:greedySearchNew}

$\mathcal{A} \leftarrow \emptyset$;

$j^* \leftarrow$ None;

\Repeat{$\frac{g_j^2}{t_j(\mathcal{A} \cup \overline{\mathcal{P}})} < \Gamma \left( \mathcal{A} \cup \overline{\mathcal{P}} \right)$}{

\If{$j^*$ is not None}{
$\mathcal{A} \leftarrow \mathcal{A} \cup \{ j^* \}$;
}

$j^* \leftarrow {\arg\max}_{j \in \mathcal{P} \setminus \mathcal{A}} \,\,\, \frac{g_j^2}{t_j(\mathcal{A} \cup \overline{\mathcal{P}})}$;
}

\textbf{return} $\mathcal{A}$ \tcp*{final result}

\end{algorithm}

\begin{theorem}
For \emph{general} $T(\mathcal{M})$, we have $\Gamma \left( \mathcal{A} \cup \overline{\mathcal{P}} \right) \geq \Gamma \left( \mathcal{A}' \cup \overline{\mathcal{P}} \right)$, where $\mathcal{A}$ is given by Algorithm~\ref{alg:greedySearchNew} and $\mathcal{A}' = \mathcal{A} \cup \{j\}$ for any $j \in \mathcal{P} \setminus \mathcal{A}$. 
\label{theorem:generalT}
\end{theorem}

Theorem~\ref{theorem:generalT} shows that for general $T(\mathcal{M})$, adding another component to $\mathcal{A}$ cannot improve the solution to (\ref{eq:optimizationProblem}).

\textit{Remark.} Theorem~\ref{theorem:generalT} gives a weaker result for general $T(\cdot)$ compared to the global optimality result in Theorem~\ref{theorem:linearT} for linear $T(\cdot)$, because when $\mathcal{A}$ and $\mathcal{A}'$ differ by more than one element, it is non-straightforward to express the change in cost for general $T(\cdot)$. Furthermore, there may exist multiple local optimal solutions for general $T(\cdot)$.

\subsection{Proofs}

\subsubsection{Proof of Theorems~\ref{theorem:linearT} and \ref{theorem:generalT}}

Recall that $\Gamma(\mathcal{M}) := \frac{\Delta(\mathcal{M})}{T(\mathcal{M})}$, where $\Delta(\mathcal{M})$ and $T(\mathcal{M})$ are both monotone and positive functions, i.e., for any $\mathcal{M} \subseteq \mathcal{M}'$, we have $0 \leq \Delta(\mathcal{M}) \leq \Delta(\mathcal{M}')$ and $0 \leq T(\mathcal{M}) \leq T(\mathcal{M}')$.

\begin{lemma} \label{lemma:incrementCondition}
For any $\mathcal{M}$ and $\mathcal{M}'$, let $\delta_\Delta(\mathcal{M}, \mathcal{M}') := \Delta(\mathcal{M}') - \Delta(\mathcal{M})$ and $\delta_T(\mathcal{M}, \mathcal{M}') := T(\mathcal{M}') - T(\mathcal{M})$. We have $\Gamma(\mathcal{M}') \leq \Gamma(\mathcal{M})$ if and only if $\delta_\Delta(\mathcal{M}, \mathcal{M}') \leq  \Gamma(\mathcal{M}) \cdot \delta_T(\mathcal{M}, \mathcal{M}')$.
\end{lemma}
\begin{proof}
\begin{align*}
& \Gamma(\mathcal{M}') := \frac{\Delta(\mathcal{M}')}{T(\mathcal{M}')} \leq \Gamma(\mathcal{M}) \\
& \iff \Delta(\mathcal{M}') \leq \Gamma(\mathcal{M}) \cdot T(\mathcal{M}') \\
& \iff \Delta(\mathcal{M}) + \delta_\Delta(\mathcal{M}, \mathcal{M}')
\leq
\Gamma(\mathcal{M}) \cdot T(\mathcal{M}) + \Gamma(\mathcal{M}) \cdot \delta_T(\mathcal{M}, \mathcal{M}') \tag{by definition of $\delta_\Delta(\cdot, \cdot)$ and $\delta_T(\cdot, \cdot)$} \\
& \iff \Delta(\mathcal{M}) + \delta_\Delta(\mathcal{M}, \mathcal{M}')
\leq
\Delta(\mathcal{M}) + \Gamma(\mathcal{M}) \cdot \delta_T(\mathcal{M}, \mathcal{M}') \tag{by definition of $\Gamma(\mathcal{M})$} \\
& \iff \delta_\Delta(\mathcal{M}, \mathcal{M}') \leq  \Gamma(\mathcal{M}) \cdot \delta_T(\mathcal{M}, \mathcal{M}') \,\,.
\end{align*}
\end{proof}

We are now ready to prove Theorems~\ref{theorem:linearT} and \ref{theorem:generalT}.

\begin{proof}[Proof of Theorem~\ref{theorem:linearT}]
By definition, we have $\Delta(\mathcal{M}) := \sum_{j \in \mathcal{M}} g_j^2$ and $T(\mathcal{M}) := c + \sum_{j \in \mathcal{M}} t_j$ for any $\mathcal{M}$.

In the following, we let $\mathcal{M} := \mathcal{A} \cup \overline{\mathcal{P}}$ and $\mathcal{M}' := \mathcal{A}' \cup \overline{\mathcal{P}}$. We have
\begin{align}
    \delta_\Delta(\mathcal{M}, \mathcal{M}') & = \Delta(\mathcal{M}') - \Delta(\mathcal{M}) \nonumber\\ &= \sum_{j \in \mathcal{M}' \setminus \mathcal{M}} g_j^2 - \sum_{j \in \mathcal{M} \setminus \mathcal{M}'} g_j^2 \label{eq:proofLinearTStep1}\\
    \delta_T(\mathcal{M}, \mathcal{M}') & = T(\mathcal{M}') - T(\mathcal{M}) \nonumber\\ &= \sum_{j \in \mathcal{M}' \setminus \mathcal{M}} t_j - \sum_{j \in \mathcal{M} \setminus \mathcal{M}'} t_j \,\,. \label{eq:proofLinearTStep2}
\end{align}

For $\mathcal{A}$ obtained from Algorithm~\ref{alg:greedySearchNewLinear}, we can easily see that $\frac{g_j^2}{t_j} < \Gamma(\mathcal{M})$ for any $j \in \mathcal{M}' \setminus \mathcal{M}$ and $\frac{g_{j'}^2}{t_{j'}} \geq \Gamma(\mathcal{M})$ for any $j' \in \mathcal{M} \setminus \mathcal{M}'$. Hence,
\begin{align}
    \sum_{j \in \mathcal{M}' \setminus \mathcal{M}} g_j^2 & < \Gamma(\mathcal{M}) \cdot \sum_{j \in \mathcal{M}' \setminus \mathcal{M}} t_j \\
    \sum_{j \in \mathcal{M} \setminus \mathcal{M}'} g_j^2 & \geq \Gamma(\mathcal{M}) \cdot \sum_{j \in \mathcal{M} \setminus \mathcal{M}'} t_j \,\,.
\end{align}
Combining with (\ref{eq:proofLinearTStep1}) and (\ref{eq:proofLinearTStep2}), we have
$$
\delta_\Delta(\mathcal{M}, \mathcal{M}') \leq \Gamma(\mathcal{M}) \cdot \delta_T(\mathcal{M}, \mathcal{M}')\,\,.
$$
Then, the result follows from Lemma~\ref{lemma:incrementCondition}.
\end{proof}

\begin{proof}[Proof of Theorem~\ref{theorem:generalT}]
Let $\mathcal{M} := \mathcal{A} \cup \overline{\mathcal{P}}$ and $\mathcal{M}' := \mathcal{A}' \cup \overline{\mathcal{P}}$ in this proof.
As $\mathcal{A}' := \mathcal{A} \cup \{j\}$ for some $j \notin \mathcal{A}$ by definition in this theorem, we note that $\delta_\Delta(\mathcal{M}, \mathcal{M}') = g_j^2$ and $\delta_T(\mathcal{M}, \mathcal{M}') = t_j(\mathcal{M})$.

For $\mathcal{A}$ obtained from Algorithm~\ref{alg:greedySearchNew}, it is easy to see that
$\frac{g_j^2}{t_j(\mathcal{M})} < \Gamma \left( \mathcal{M} \right)$
for $j \notin \mathcal{A}$. Hence, $\delta_\Delta(\mathcal{M}, \mathcal{M}') < \Gamma \left( \mathcal{M} \right) \cdot \delta_T(\mathcal{M}, \mathcal{M}')$ and the result follows from Lemma~\ref{lemma:incrementCondition}.
\end{proof}
\vspace{0.2in}

\subsubsection{Proof of Theorem~\ref{theorem:convergence}} \label{appendix:thm2_proof}
The analysis of this section is an extension of Theorem 1 in~\cite{wang2020local} (proof given in Section A.1). Note that Assumption~\ref{assumption:convergence}(a)-(e) still hold if we apply the same mask to all gradients or function values in the LHS, since applying masks is equivalent to replacing a subset of the entries in the parameter with zeros. For the same reason, Assumption~\ref{assumption:convergence}(f) and \ref{assumption:convergence}(g) also hold if the gradients are masked. For convenience, we define three expedient notations in addition to the notations in Table~\ref{tab:main_notations} for pruned values: $\mathbf{g}'_n (\mathbf{w}) := \mathbf{g}_n(\mathbf{w}) \odot \mathbf{m}(k)$, $\nabla' F_n(\mathbf{w}) := \nabla F_n(\mathbf{w}) \odot \mathbf{m}(k)$, and $\nabla' F(\mathbf{w}) := \nabla F(\mathbf{w}) \odot \mathbf{m}(k)$.

We first present a special form of Jensen's inequality, which has the original form below:
$$
\phi \bigg( \frac{\sum_i a_i x_i}{\sum_i a_i} \bigg) \leq \frac{\sum_i a_i \phi(x_i) }{\sum_i a_i}\,\,,
$$
where $\phi(\cdot)$ is a convex function, $a_i$'s are positive weights.
\begin{lemma} Assume $\phi(\cdot)$ is a convex function, $p_n$'s are positive weights that sum to $1$, we have 
$$
\Bigg\| \sum_n p_n \mathbf{x}_n \Bigg\|^2 \leq \sum_n p_n \|\mathbf{x}_n\|^2 \,\,.
$$
\end{lemma}
\begin{proof}
$$
\Bigg\| \sum_n p_n \mathbf{x}_n \Bigg\|^2
=   \left\Vert \begin{bmatrix}
           \sum_n p_n x_{n,1} \\
           \sum_n p_n x_{n,2} \\
           \vdots
    \end{bmatrix}
    \right\Vert^2
= \sum_i \Big( \sum_n p_n x_{n,i} \Big)^2
\leq \sum_i \sum_n p_n x^2_{n,i}
= \sum_n p_n \sum_i x^2_{n,i}
= \sum_n p_n \|\mathbf{x}_n\|^2 \,\,,
$$
where $x_{n,i}$ denotes the $i$-th component of $\mathbf{x}_n$.
\end{proof}

The local updating rule is 
\begin{align*}
    \mathbf{w}'_n(k+1) &= \mathbf{w}_n(k) \odot \mathbf{m}(k) - \eta \mathbf{g}_n \big(\mathbf{w}_n(k) \odot \mathbf{m}(k) \big) \odot \mathbf{m}(k)\\
    &= \mathbf{w}'_n(k) - \eta \mathbf{g}'_n(\mathbf{w}_n'(k))\,\,,\\
\end{align*}
Consequently, the updating rule for the averaged parameters is
\begin{align*}
    \mathbf{w}(k+1) & =  \sum_{n=1}^N p_n \mathbf{w}'_n(k+1)= {\mathbf{w}}'(k) - \eta \sum_{n=1}^N p_n \mathbf{g}'_n(\mathbf{w}_n'(k))\,\,.
\end{align*}

Note that ${\mathbf{w}}(k)$, ${\mathbf{w}}'(k)$ are observable only in iterations where the clients send their local parameters to the server for aggregation (see Algorithm~\ref{alg:adaptive_pruning}), and $\mathbf{g}'(\mathbf{w}_n'(k))$ is dependent on $\mathbf{m}(k)$. We know that
\begin{align}
    &\Expect \left[ F({\mathbf{w}}(k+1)) | \{\mathbf{w}_n(k)\}, \mathbf{m}(k) \right] \nonumber\\
    &=\Expect \bigg[F \Big( {\mathbf{w}}'(k)-\eta  \sum_{n=1}^N p_n\mathbf{g}'_n(\mathbf{w}_n'(k)) \Big) \Big| \{\mathbf{w}_n(k)\}, \mathbf{m}(k) \bigg] \nonumber\\
    & \stackrel{(a)}{\leq} F ({\mathbf{w}}'(k)) - \eta \Expect \left[ \left\langle \nabla F({\mathbf{w}}'(k)) ,  \sum_{n=1}^N p_n\mathbf{g}'_n(\mathbf{w}_n'(k)) \right\rangle \Bigg| \{\mathbf{w}_n(k)\}, \mathbf{m}(k) \right] + \frac{\eta^2 \beta}{2} \Expect \left[ \left\|  \sum_{n=1}^N p_n\mathbf{g}'_n(\mathbf{w}_n'(k)) \right\|^2 \Bigg| \{\mathbf{w}_n(k)\}, \mathbf{m}(k) \right] \nonumber \\
    & \stackrel{(b)}{\leq} F ({\mathbf{w}}(k)) + L\left\| {\mathbf{w}}(k) - {\mathbf{w}}'(k) \right\| - \eta \Expect \Big[ \Big\langle \nabla F({\mathbf{w}}'(k)) ,  \sum_{n=1}^N p_n\mathbf{g}'_n(\mathbf{w}_n'(k)) \Big\rangle \Big| \{\mathbf{w}_n(k)\}, \mathbf{m}(k) \Big] \nonumber \\
    & \quad\quad\quad\quad + \frac{\eta^2 \beta}{2} \Expect \left [ \left\|  \sum_{n=1}^N p_n\mathbf{g}'_n(\mathbf{w}_n'(k)) \right\|^2 \Bigg| \{\mathbf{w}_n(k)\}, \mathbf{m}(k) \right] \nonumber \\
    & = F ({\mathbf{w}}(k)) + L\left\| {\mathbf{w}}(k) - {\mathbf{w}}'(k) \right\| - \eta \Expect \Big[ \Big\langle \nabla' F({\mathbf{w}}'(k)) ,  \sum_{n=1}^N p_n\mathbf{g}'_n(\mathbf{w}_n'(k)) \Big\rangle \Big| \{\mathbf{w}_n(k)\}, \mathbf{m}(k) \Big] \nonumber \\
    & \quad\quad\quad\quad + \frac{\eta^2 \beta}{2} \Expect \left [ \left\|  \sum_{n=1}^N p_n\mathbf{g}'_n(\mathbf{w}_n'(k)) \right\|^2 \Bigg| \{\mathbf{w}_n(k)\}, \mathbf{m}(k) \right] \label{eq:thm2_expansion1}
\end{align}
where (a) is due to Assumption~\ref{assumption:convergence}(a) ($\beta$-smoothness), (b) is due to Assumption~\ref{assumption:convergence}(b) ($L$-Lipschitzness), and (\ref{eq:thm2_expansion1}) is because
\begin{align*}
\Big\langle \nabla F({\mathbf{w}}'(k)) ,  \sum_{n=1}^N p_n\mathbf{g}'_n(\mathbf{w}_n'(k)) \Big\rangle &= \Big\langle \nabla F({\mathbf{w}}'(k)) ,  \sum_{n=1}^N p_n\mathbf{g}_n(\mathbf{w}_n'(k))\odot \mathbf{m}(k) \Big\rangle \\
& = \Big\langle \nabla F({\mathbf{w}}'(k)) \odot \mathbf{m}(k) ,  \sum_{n=1}^N p_n\mathbf{g}_n(\mathbf{w}_n'(k)) \odot \mathbf{m}(k) \Big\rangle \\
& = \Big\langle \nabla' F({\mathbf{w}}'(k)) ,  \sum_{n=1}^N p_n\mathbf{g}'_n(\mathbf{w}_n'(k)) \Big\rangle\,\,.
\end{align*}
The third term in~(\ref{eq:thm2_expansion1}) can be rewritten as

\begin{align}
    &- \eta \Expect \Big[ \Big\langle \nabla' F({\mathbf{w}}'(k)) ,  \sum_{n=1}^N p_n\mathbf{g}'_n(\mathbf{w}_n'(k)) \Big\rangle \Big| \{\mathbf{w}_n(k)\}, \mathbf{m}(k) \Big] \nonumber\\
    &\stackrel{(a)}{=} - \eta \Big\langle \nabla' F({\mathbf{w}}'(k)) ,  \sum_{n=1}^N p_n\nabla'F_n(\mathbf{w}_n'(k)) \Big\rangle \nonumber\\
    & = \frac{\eta}{2} \left(  \bigg \| \nabla' F({\mathbf{w}}'(k)) -  \sum_{n=1}^N p_n\nabla' F_n(\mathbf{w}_n'(k)) \bigg \|^2 -  \big\Vert \nabla' F({\mathbf{w}}'(k)) \big\Vert^2 -  \bigg\Vert  \sum_{n=1}^N p_n\nabla' F_n(\mathbf{w}_n'(k)) \bigg\Vert^2 \right) \,\,, \label{eq:thm2_term3_expansion}
\end{align}
where (a) uses Assumption~\ref{assumption:convergence}(c) (unbiasedness). Our fourth term in~(\ref{eq:thm2_expansion1}) is bounded by 
\begin{align}
    &\Expect \Bigg[ \bigg\|  \sum_{n=1}^N p_n\mathbf{g}'_n(\mathbf{w}_n'(k)) \bigg\|^2 \bigg| \{\mathbf{w}_n(k)\}, \mathbf{m}(k) \Bigg] \nonumber \\
    & \stackrel{(a)}{=} \Expect \Bigg[ \bigg\|  \sum_{n=1}^N p_n \Big(\mathbf{g}'_n(\mathbf{w}_n'(k)) - \nabla'{F_n}(\mathbf{w}_n'(k)) \Big) \bigg\|^2 \bigg| \{\mathbf{w}_n(k)\}, \mathbf{m}(k) \Bigg] + \Bigg( \Expect \bigg[  \bigg\|  \sum_{n=1}^N p_n\nabla'{F_n}(\mathbf{w}_n'(k)) \bigg\| \Big| \{\mathbf{w}_n(k)\}, \mathbf{m}(k) \bigg] \Bigg)^2 \nonumber \\
    & \stackrel{(b)}{=} \Expect \Bigg[  \sum_{n=1}^N \bigg\| p_n \Big(\mathbf{g}'_n(\mathbf{w}_n'(k)) - \nabla'{F_n}(\mathbf{w}_n'(k)) \Big) \bigg\|^2 \bigg| \{\mathbf{w}_n(k)\}, \mathbf{m}(k) \Bigg] + \Bigg( \Expect \bigg[  \bigg\|  \sum_{n=1}^N p_n\nabla'{F_n}(\mathbf{w}_n'(k)) \bigg\| \Big| \{\mathbf{w}_n(k)\}, \mathbf{m}(k) \bigg] \Bigg)^2 \nonumber \\
    &\stackrel{(c)}{\leq} \Big( \sum_{n=1}^N p_n^2 \Big)  \sigma^2 + \Bigg\Vert  \sum_{n=1}^N p_n\nabla'{F_n}(\mathbf{w}_n'(k)) \Bigg\Vert^2 \label{eq:thm2_term4_expansion} \,\,,
\end{align}
where (a) uses the definition of variance, i.e., $\Expect [\|\mathbf{x}\|^2] = \Expect \big[\| \mathbf{x} - \Expect [\mathbf{x}] \|^2 \big] + [\Expect\|\mathbf{x}\|]^2$; (b) uses Assumption~\ref{assumption:convergence}(g) (client independence), expands the first term and removes the zero-valued cross-product terms; (c) uses Assumption~\ref{assumption:convergence}(d) (bounded variance). Substituting (\ref{eq:thm2_term3_expansion}) and (\ref{eq:thm2_term4_expansion}) into (\ref{eq:thm2_expansion1}), (\ref{eq:thm2_expansion1}) becomes
\begin{align}
    &\Expect \left[ F({\mathbf{w}}(k+1)) | \{\mathbf{w}_n(k)\}, \mathbf{m}(k) \right] \nonumber \\
    &\leq F ({\mathbf{w}}(k)) + L\left\| {\mathbf{w}}(k) - {\mathbf{w}}'(k) \right\| + \frac{\eta^2 \beta \sum_{n=1}^N p_n^2 }{2}  \sigma^2 - \frac{\eta}{2}  \big\Vert \nabla' F({\mathbf{w}}'(k)) \big\Vert^2 + \frac{\eta}{2}  \bigg \| \nabla' F({\mathbf{w}}'(k)) - \sum_{n=1}^N p_n\nabla'{F}_n(\mathbf{w}_n'(k)) \bigg \|^2 \nonumber \\
    & \quad\quad\quad\quad - \Big(\frac{\eta}{2} - \frac{\eta^2\beta}{2} \Big) \Bigg\Vert \sum_{n=1}^N p_n \nabla'{F_n}(\mathbf{w}_n'(k)) \Bigg\Vert^2. \label{eq:thm2_expansion_4}
\end{align}
Assuming $\eta \leq \frac{1}{\beta}$, we have $\frac{\eta}{2} - \frac{\eta^2\beta}{2}\geq 0$, and the last term in (\ref{eq:thm2_expansion_4}) can be removed. Then we have the following:

\begin{align}
    &\Expect \left[ F({\mathbf{w}}(k+1)) | \{\mathbf{w}_n(k)\}, \mathbf{m}(k) \right] \nonumber \\
    &\leq F ({\mathbf{w}}(k)) + L\left\| {\mathbf{w}}(k) - {\mathbf{w}}'(k) \right\| + \frac{\eta^2 \beta \sum_{n=1}^N p_n^2 }{2}  \sigma^2 - \frac{\eta}{2}  \big\Vert \nabla' F({\mathbf{w}}'(k)) \big\Vert^2 + \frac{\eta}{2}  \bigg \| \nabla' F({\mathbf{w}}'(k)) - \sum_{n=1}^N p_n\nabla'{F}_n(\mathbf{w}_n'(k)) \bigg \|^2 \nonumber \\
    & \stackrel{(a)}{\leq} F ({\mathbf{w}}(k)) + L\left\| {\mathbf{w}}(k) - {\mathbf{w}}'(k) \right\| + \frac{\eta^2 \beta \sum_{n=1}^N p_n^2 }{2}  \sigma^2 - \frac{\eta}{2} \big\Vert \nabla' F({\mathbf{w}}'(k)) \big\Vert^2 + \frac{\eta}{2} \sum_{n=1}^N p_n \bigg \| \nabla' F_n({\mathbf{w}}'(k)) - \nabla'{F}_n(\mathbf{w}_n'(k)) \bigg \|^2 \nonumber \\
    & \stackrel{(b)}{\leq} F ({\mathbf{w}}(k)) + L\left\| {\mathbf{w}}(k) - {\mathbf{w}}'(k) \right\| + \frac{\eta^2 \beta \sum_{n=1}^N p_n^2 }{2}  \sigma^2 - \frac{\eta}{2}  \big\Vert \nabla' F({\mathbf{w}}'(k)) \big\Vert^2 + \frac{\eta \beta^2}{2} \sum_{n=1}^N p_n \big \| {\mathbf{w}}'(k) - \mathbf{w}_n'(k) \big \|^2 \label{eq:thm2_expansion_5} \,\,,
\end{align}
where (a) uses Jensen's inequality, and (b) uses Assumption~\ref{assumption:convergence}(a) (smoothness). Taking expectation on both sides of (\ref{eq:thm2_expansion_5}), we get

\begin{align}
    \Expect \left[ F({\mathbf{w}}(k+1)) \right] &\leq \Expect \left[ F ({\mathbf{w}}(k)) \right] + L \Expect\left\| {\mathbf{w}}(k) - {\mathbf{w}}'(k) \right\| + \frac{\eta^2 \beta \sum_{n=1}^N p_n^2 }{2}  \sigma^2 - \frac{\eta}{2} \Expect \big\Vert \nabla' F({\mathbf{w}}'(k)) \big\Vert^2 \nonumber \\
    & \quad\quad\quad\quad + \frac{\eta \beta^2}{2} \sum_{n=1}^N p_n \Expect \big \| {\mathbf{w}}'(k) - \mathbf{w}_n'(k) \big \|^2 \label{eq:thm2_expansion_6} \,\,.
\end{align}
Taking average over time on (\ref{eq:thm2_expansion_6}) and rearranging, we get
\begin{align}
    &\frac{1}{K} \sum_{k=0}^{K-1} \Expect \big\Vert \nabla' F({\mathbf{w}}'(k)) \big\Vert^2 \nonumber \\
    &\leq \frac{2}{\eta K} [F(\mathbf{w}(0)) - F^*] + \frac{2 L}{\eta K} \sum_{k=0}^{K-1} \Expect \left\| {\mathbf{w}}(k) - {\mathbf{w}}'(k) \right\| + \eta \beta \Big( \sum_{n=1}^N p_n^2 \Big) \sigma^2 + \frac{\beta^2}{K} \sum_{k=0}^{K-1} \sum_{n=1}^N p_n \Expect \big \| {\mathbf{w}}'(k) - \mathbf{w}_n'(k) \big \|^2 \label{eq:thm2_rearranged}\,\,.
\end{align}
Now we bound the last term of (\ref{eq:thm2_rearranged}).

\begin{align}
    &\sum_{n=1}^N p_n \Expect \big \| {\mathbf{w}}'(k) - \mathbf{w}_n'(k) \big \|^2 \nonumber\\
    &= \sum_{n=1}^N p_n \Expect \bigg \| \Big( {\mathbf{w}}'(k-1) -\eta \sum_{i=1}^N p_i \mathbf{g}'_i(\mathbf{w}_i'(k-1)) \Big) - \Big( \mathbf{w}_n'(k-1) - \eta \mathbf{g}'_n(\mathbf{w}_n'(k-1)) \Big) \bigg \|^2 \nonumber \\
    & \stackrel{(a)}{=} \eta^2 \sum_{n=1}^N p_n \Expect \Bigg\| \sum_{\tau=I \cdot \lfloor k/I \rfloor}^{k-1} \bigg( \mathbf{g}'_n(\mathbf{w}_n'(\tau)) - \sum_{i=1}^N p_i\mathbf{g}'_i(\mathbf{w}_i'(\tau)) \bigg) \Bigg\|^2 \nonumber\\
    & = \eta^2 \sum_{n=1}^N p_n \Expect \Bigg\| \sum_{\tau=I \cdot \lfloor k/I \rfloor}^{k-1} \bigg( \Big( \mathbf{g}'_n(\mathbf{w}_n'(\tau)) - \nabla' F_n({\mathbf{w}}'_n(\tau)) + \sum_{i=1}^N p_i \nabla' F_i({\mathbf{w}}'_i(\tau)) - \sum_{i=1}^N p_i \mathbf{g}'_i(\mathbf{w}_i'(\tau)) \Big) \nonumber \\
    & \quad\quad\quad\quad + \Big( \nabla' F_n({\mathbf{w}}'_n(\tau)) - \sum_{i=1}^N p_i \nabla' F_i({\mathbf{w}}'_i(\tau)) \Big) \bigg) \Bigg\|^2 \nonumber \\
    & \leq 2\eta^2 \sum_{n=1}^N p_n \Expect \Bigg\| \sum_{\tau=I \cdot \lfloor k/I \rfloor}^{k-1} \Big( \mathbf{g}'_n(\mathbf{w}_n'(\tau)) - \nabla' F_n({\mathbf{w}}'_n(\tau)) + \sum_{i=1}^N p_i \nabla' F_i({\mathbf{w}}'_i(\tau)) - \sum_{i=1}^N p_i \mathbf{g}'_i(\mathbf{w}_i'(\tau)) \Big) \Bigg\|^2 \label{eq:thm2_rearranged_last_term1} \\
    & \quad\quad\quad\quad + 2\eta^2 \sum_{n=1}^N p_n \Expect \Bigg\| \sum_{\tau=I \cdot \lfloor k/I \rfloor}^{k-1} \Big( \nabla' F_n({\mathbf{w}}'_n(\tau)) - \sum_{i=1}^N p_i \nabla' F_n({\mathbf{w}}'_i(\tau)) \Big) \Bigg\|^2 \label{eq:thm2_rearranged_last_term2} \,\,,
\end{align}
where in (a), we trace back to the nearest iteration where all local parameters are synchronized. For (\ref{eq:thm2_rearranged_last_term1}),

\begin{align}
    &2\eta^2 \sum_{n=1}^N p_n \Expect \Bigg\| \sum_{\tau=I \cdot \lfloor k/I \rfloor}^{k-1} \Big( \mathbf{g}'_n(\mathbf{w}_n'(\tau)) - \nabla' F_n({\mathbf{w}}'_n(\tau)) + \sum_{i=1}^N p_i \nabla' F_i({\mathbf{w}}'_i(\tau)) - \sum_{i=1}^N p_i \mathbf{g}'_i(\mathbf{w}_i'(\tau)) \Big) \Bigg\|^2 \nonumber \\
    & \stackrel{(a)}{=} 2\eta^2 \sum_{n=1}^N p_n \Expect \Bigg\| \sum_{\tau=I \cdot \lfloor k/I \rfloor}^{k-1} \Big( \mathbf{g}'_n(\mathbf{w}_n'(\tau)) - \nabla' F_n({\mathbf{w}}'_n(\tau)) \Big) \Bigg\|^2 - 2 \eta^2 \sum_{n=1}^N p_n \Expect \Bigg\| \sum_{\tau=I \cdot \lfloor k/I \rfloor}^{k-1} \sum_{i=1}^N p_i \Big( \mathbf{g}'_i(\mathbf{w}_i'(\tau)) - \nabla' F_i({\mathbf{w}}'_i(\tau)) \Big) \Bigg\|^2 \nonumber\\
    & \stackrel{(b)}{=} 2\eta^2 \sum_{n=1}^N p_n \sum_{\tau=I \cdot \lfloor k/I \rfloor}^{k-1} \Expect \Big\| \mathbf{g}'_n(\mathbf{w}_n'(\tau)) - \nabla' F_n({\mathbf{w}}'_n(\tau)) \Big\|^2 - 2\eta^2 \sum_{\tau=I \cdot \lfloor k/I \rfloor}^{k-1} \Expect \Bigg\| \sum_{i=1}^N p_i \Big( \mathbf{g}'_i(\mathbf{w}_i'(\tau)) - \nabla' F_i({\mathbf{w}}'_i(\tau)) \Big) \Bigg\|^2 \nonumber\\
    & \stackrel{(c)}{=} 2\eta^2 \sum_{\tau=I \cdot \lfloor k/I \rfloor}^{k-1} \sum_{n=1}^N p_n \Expect \Big\| \mathbf{g}'_n(\mathbf{w}_n'(\tau)) - \nabla' F_n({\mathbf{w}}'_n(\tau)) \Big\|^2 - 2\eta^2  \sum_{\tau=I \cdot \lfloor k/I \rfloor}^{k-1} \sum_{n=1}^N \Expect \bigg\| p_n  \Big( \mathbf{g}'_n(\mathbf{w}_n'(\tau)) - \nabla' F_n({\mathbf{w}}'_n(\tau)) \Big) \bigg\|^2 \nonumber\\
    & \leq 2\eta^2 \sum_{\tau=I \cdot \lfloor k/I \rfloor}^{k-1} \sum_{n=1}^N (p_n - p_n^2) \Expect \Big\| \mathbf{g}'_n(\mathbf{w}_n'(\tau)) - \nabla' F_n({\mathbf{w}}'_n(\tau)) \Big\|^2 \nonumber\\
    & \leq 2 \Big( 1- \sum_{n=1}^N p_n^2 \Big) I \eta^2 \sigma^2 \label{eq:thm2_rearranged_last_term1_exp} \,\,, 
\end{align}
where (a) is due to the definition of variance; (b) is due to Assumption~\ref{assumption:convergence}(f) (time independence) and $\sum_n p_n = 1$; and (c) is due to Assumption~\ref{assumption:convergence}(g) (client independence). For (\ref{eq:thm2_rearranged_last_term2}),

\begin{align}
    &2\eta^2 \sum_{n=1}^N p_n \Expect \Bigg\| \sum_{\tau=I \cdot \lfloor k/I \rfloor}^{k-1} \Big( \nabla' F_n({\mathbf{w}}'_n(\tau)) - \sum_{i=1}^N p_i \nabla' F_i({\mathbf{w}}'_i(\tau)) \Big) \Bigg\|^2 \nonumber \\
    & = 2\eta^2 \sum_{n=1}^N p_n \Expect \Bigg\| \sum_{\tau=I \cdot \lfloor k/I \rfloor}^{k-1} \bigg( \Big( \nabla' F_n({\mathbf{w}}'_n(\tau)) - \nabla' F_n({\mathbf{w}}'(\tau)) \Big) + \Big( \nabla' F_n({\mathbf{w}}'(\tau)) - \sum_{i=1}^N p_i \nabla' F_i({\mathbf{w}'}(\tau)) \Big) \nonumber \\
    &\quad\quad\quad\quad + \Big(\sum_{i=1}^N p_i \nabla' F_i({\mathbf{w}'}(\tau)) - \sum_{i=1}^N p_i \nabla' F_i({\mathbf{w}}'_i(\tau)) \Big) \bigg) \Bigg\|^2 \nonumber \\
    &\leq 6\eta^2 \sum_{n=1}^N \Bigg( p_n \Expect \bigg\| \sum_{\tau=I \cdot \lfloor k/I \rfloor}^{k-1} \Big( \nabla' F_n({\mathbf{w}}'_n(\tau)) - \nabla' F_n({\mathbf{w}}'(\tau)) \Big) \bigg\|^2 + p_n \Expect \bigg\| \sum_{\tau=I \cdot \lfloor k/I \rfloor}^{k-1} \sum_{i=1}^N p_i \Big( \nabla' F_i({\mathbf{w}'}(\tau)) -  \nabla' F_i({\mathbf{w}}'_i(\tau)) \Big) \bigg\|^2 \nonumber \\
    &\quad\quad\quad\quad + p_n \Expect \bigg\| \sum_{\tau=I \cdot \lfloor k/I \rfloor}^{k-1} \Big( \nabla' F_n({\mathbf{w}}'(\tau)) - \sum_{i=1}^N p_i \nabla' F_i({\mathbf{w}'}(\tau)) \Big) \bigg\|^2 \Bigg) \nonumber \\
    &\stackrel{(a)}{\leq} 6\eta^2 I \sum_{\tau=I \cdot \lfloor k/I \rfloor}^{k-1} \sum_{n=1}^N p_n \Bigg(\Expect \bigg\| \nabla' F_n({\mathbf{w}}'_n(\tau)) - \nabla' F_n({\mathbf{w}}'(\tau)) \bigg\|^2 + \Expect \bigg\| \sum_{i=1}^N p_i \Big( \nabla' F_i({\mathbf{w}'}(\tau)) -  \nabla' F_i({\mathbf{w}}'_i(\tau)) \Big) \bigg\|^2 \Bigg) \nonumber \\
    &\quad\quad\quad\quad + 6\eta^2 I \Expect \bigg[ \sum_{n=1}^N p_n \sum_{\tau=I \cdot \lfloor k/I \rfloor}^{k-1} \bigg\| \nabla' F_n({\mathbf{w}}'(\tau)) - \sum_{i=1}^N p_i \nabla' F_i({\mathbf{w}'}(\tau)) \bigg\|^2 \bigg] \nonumber \\
    &\stackrel{(b)}{\leq} 6\eta^2 I \sum_{\tau=I \cdot \lfloor k/I \rfloor}^{k-1} \sum_{n=1}^N p_n \bigg( \Expect \bigg\| \nabla' F_n({\mathbf{w}}'_n(\tau)) - \nabla' F_n({\mathbf{w}}'(\tau)) \bigg\|^2 + \Expect \bigg\| \sum_{i=1}^N p_i \Big( \nabla' F_i({\mathbf{w}'}(\tau)) -  \nabla' F_i({\mathbf{w}}'_i(\tau)) \Big) \bigg\|^2 \bigg) + 6 I^2 \eta^2 \epsilon^2 \nonumber \\
    &\stackrel{(c)}{\leq} 12\eta^2 I \sum_{\tau=I \cdot \lfloor k/I \rfloor}^{k-1} \Expect  \bigg\|  \nabla' F_n({\mathbf{w}}'_n(\tau)) - \nabla' F_n({\mathbf{w}}'(\tau)) \bigg\|^2 + 6 I^2 \eta^2 \epsilon^2 \nonumber \\
    & \leq 12 I \eta^2 \beta^2 \sum_{n=1}^N \sum_{\tau=I \cdot \lfloor k/I \rfloor}^{k-1} \Expect \Big\| \mathbf{w}'_n(\tau) - \mathbf{w}'(\tau) \Big\|^2 + 6 I^2 \eta^2 \epsilon^2 \label{eq:thm2_rearranged_last_term2_exp} \,\,,
\end{align} 
where in (a), we use the fact that, $\forall \mathbf{x}_{\tau}$,
\begin{align*}
\bigg\| \sum_{\tau = I \cdot \lfloor k/I \rfloor}^{k-1} \mathbf{x}_{\tau} \bigg\|^2 & \leq \Big((k-1) - I \cdot\lfloor k/I \rfloor \Big) \cdot \sum_{\tau = I\cdot\lfloor k/I \rfloor}^{k-1} \| \mathbf{x}_{\tau} \|^2 \leq I \cdot \sum_{\tau = I\cdot\lfloor k/I \rfloor}^{k-1} \| \mathbf{x}_{\tau} \|^2 \,\,.
\end{align*}
In (b), we use Assumption~\ref{assumption:convergence}(e) (bounded divergence), and in (c), we use Jensen's inequality. Substituting (\ref{eq:thm2_rearranged_last_term1_exp}) and (\ref{eq:thm2_rearranged_last_term2_exp}) into (\ref{eq:thm2_rearranged_last_term1}) and (\ref{eq:thm2_rearranged_last_term2}), respectively, we get
\begin{align}
    &\frac{1}{K}\sum_{k=0}^{K-1}\sum_{n=1}^N p_n \Expect \big \| {\mathbf{w}}'(k) - \mathbf{w}_n'(k) \big \|^2 \nonumber\\
    & \leq 2 \Big(1- \sum_{n=1}^N p_n^2 \Big) I \eta^2 \sigma^2 + 6\eta^2 I^2\epsilon^2 + \frac{12 I \eta^2 \beta^2}{K} \sum_{n=1}^N \sum_{k=0}^{K-1} \sum_{\tau=I \cdot \lfloor k/I \rfloor}^{k-1} \Expect \Big\| \mathbf{w}'_n(\tau) - \mathbf{w}'(\tau) \Big\|^2  \nonumber\\
    & \stackrel{(a)}{\leq} 2 \Big(1- \sum_{n=1}^N p_n^2 \Big) \eta^2 I \sigma^2 + 6\eta^2 I^2\epsilon^2 + \frac{12 I \eta^2 \beta^2}{K} \sum_{n=1}^N \sum_{k=0}^{K-1} \sum_{\tau=I \cdot \lfloor k/I \rfloor}^{\min\{I \cdot \lceil k/I \rceil, K-1\}} \Expect \Big\| \mathbf{w}'_n(\tau) - \mathbf{w}'(\tau) \Big\|^2  \nonumber\\
    & \stackrel{(b)}{\leq} 2 \Big( 1- \sum_{n=1}^N p_n^2 \Big) \eta^2 I \sigma^2 + 6\eta^2 I^2\epsilon^2 + \frac{12 I^2 \eta^2 \beta^2}{K} \sum_{k=0}^{K-1} \sum_{n=1}^N \Expect \Big\| \mathbf{w}'_n(\tau) - \mathbf{w}'(\tau) \Big\|^2\,\,,
\end{align}
where (a) holds because $k-1\leq \min\{I\cdot \lceil k/I \rceil, K-1 \}$ is always true, and (b) holds because in the innermost summation in its last term, we always have
$$
\min\{I \cdot \lceil k/I \rceil, K-1\} -  I \cdot \lfloor k/I \rfloor \leq I\,\,.
$$
Rearranging yields
\begin{align}
    &\frac{1}{K}\sum_{k=0}^{K-1}\sum_{n=1}^N p_n \Expect \big \| {\mathbf{w}}'(k) - \mathbf{w}_n'(k) \big \|^2 \leq \frac{2 \Big( 1- \sum_{n=1}^N p_n^2 \Big) \eta^2 I \sigma^2 + 6\eta^2 I^2\epsilon^2}{1- 12 I^2 \eta^2 \beta^2} \label{eq:thm2_rearranged_last_exp} \,\,.
\end{align}
Applying (\ref{eq:thm2_rearranged_last_exp}) into (\ref{eq:thm2_rearranged})'s last term, (\ref{eq:thm2_rearranged}) becomes
\begin{align}
    \frac{1}{K} \sum_{k=0}^{K-1} \Expect \big\Vert \nabla' F({\mathbf{w}}'(k)) \big\Vert^2 & \leq \frac{2}{\eta K} [F(\mathbf{w}(0)) - F^*] + \eta \beta \Big( \sum_{n=1}^N p_n^2 \Big) \sigma^2 + \frac{2 \Big( 1- \sum_{n=1}^N p_n^2 \Big) I \sigma^2 + 6 I^2\epsilon^2}{1- 12 I^2 \eta^2 \beta^2} \eta^2\beta^2 \nonumber \\
    & \quad\quad\quad\quad + \frac{2 L}{\eta K} \sum_{k=0}^{K-1} \Expect \left\| {\mathbf{w}}(k) - {\mathbf{w}}'(k) \right\|\,\,.
\end{align}
Assume $\eta \leq \frac{1}{2\sqrt{6} I \beta}$, we have $1- 12 \eta^2 I^2 \eta^2 \beta^2 \geq \frac{1}{2}$, and
\begin{equation}\label{eq:convergence_final}
    \frac{1}{K} \sum_{k=0}^{K-1} \Expect \big\Vert \nabla' F({\mathbf{w}}'(k)) \big\Vert^2
    \leq \frac{2(F_0 - F^*)}{\eta K} + \alpha \eta \beta \sigma^2 + 4\beta^2 \big( (1- \alpha) I\sigma^2 + 3I^2\epsilon^2 \big) \eta^2 + \frac{2 L}{\eta K} \sum_{k=0}^{K-1} \Expect \left\| {\mathbf{w}}(k) - {\mathbf{w}}'(k) \right\|\,,
\end{equation}
where $\alpha:= \sum_{n=1}^N p_n^2$, $F_0 := F(\mathbf{w}(0))$, $F^* := \min_\mathbf{w}F(\mathbf{w})$. Note that under the assumption that $\eta \leq \frac{1}{2\sqrt{6} I \beta}$, the previous assumption that $\eta \leq \frac{1}{\beta}$ used for (\ref{eq:thm2_expansion_4}) is automatically satisfied.

\textbf{Discussion.} Consider the situation where all clients have equal weight, i.e. $p_n = \frac{1}{N}$, $\forall n$, we have $\alpha = \frac{1}{N}$. Letting $\eta = \frac{1}{\sqrt{\alpha K}} = \sqrt{\frac{N}{K}}$, (\ref{eq:convergence_final}) becomes
\begin{equation}\label{eq:convergence_dis1}
    \frac{1}{K} \sum_{k=0}^{K-1} \Expect \big\Vert \nabla' F({\mathbf{w}}'(k)) \big\Vert^2
    \leq \frac{2(F_0 - F^*) + \beta \sigma^2}{\sqrt{N K}} + \frac{4\beta^2 \big( (N- 1) I\sigma^2 + 3NI^2\epsilon^2 \big)}{K} + \frac{2L}{\sqrt{N K}} \sum_{k=0}^{K-1} \Expect \left\| {\mathbf{w}}(k) - {\mathbf{w}}'(k) \right\|\,\,.
\end{equation}
For the last term in (\ref{eq:convergence_dis1}), when $k$ is not a reconfiguration iteration, $\mathbf{w}(k) = \mathbf{w}'(k)$, and when $k$ is a reconfiguration iteration, the difference between $\mathbf{w}(k)$ and $\mathbf{w}'(k)$ is the subset of parameters that get pruned from $\mathbf{w}(k)$. Assuming the norm of $\mathbf{w}(k)$ is bounded by $B$, i.e., $ \Vert \mathbf{w}(k) \Vert \leq B$, $\forall k$, the initial fraction of non-zero prunable parameters is $r_0$ and this fraction halves every $h$ iterations (as we do in our experiments, see Table~\ref{tab:hyperparam}), then the last term in (\ref{eq:convergence_dis1}) is bounded by
\begin{equation}\label{eq:convergence_dis2}
    \frac{2L}{\sqrt{N K}} \sum_{k=0}^{K-1} \Expect \left\| {\mathbf{w}}(k) - {\mathbf{w}}'(k) \right\|
    \stackrel{(a)}{\leq} \frac{2L}{\sqrt{N K}} \sum_{k=0}^{K-1} B\cdot r_0 \cdot 2^{-k/h}
    \leq \frac{2L B r_0}{\sqrt{N K}} \sum_{k=0}^{\infty} 2^{-k/h}
    \leq \frac{2^{\frac{1}{h}+1}}{2^{\frac{1}{h}} - 1} \cdot \frac{LB r_0} {\sqrt{N K}}\,\,.
\end{equation}
In (\ref{eq:convergence_dis2}), (a) holds for the following reason: the reconfiguration on the parameter vector $\mathbf{w}(k)$ includes both adding back parameters and removing parameters, resulting in a new parameter vector $\mathbf{w}'(k)$. However, parameters that are added back in $\mathbf{w}'(k)$ in reconfigurations are assigned zero values\footnote{In practice, we use a small perturbation instead of zero values for parameters that are added back, but we use zero values in our analysis for ease of exposition.}, which are equal to their corresponding values in $\mathbf{w}(k)$. Thus, the difference between $\mathbf{w}(k)$ and $\mathbf{w}'(k)$ is the part that is \textit{removed} from $\mathbf{w}(k)$, whose maximum fraction is bounded by $r_0 \cdot 2^{-k/h}$. In consequence, $\left\| {\mathbf{w}}(k) - {\mathbf{w}}'(k) \right\|$ is bounded by $B\cdot r_0 \cdot 2^{-k/h}$. Plugging (\ref{eq:convergence_dis2}) into (\ref{eq:convergence_dis1}), we get
\begin{align}
    \frac{1}{K} \sum_{k=0}^{K-1} \Expect \big\Vert \nabla' F({\mathbf{w}}'(k)) \big\Vert^2
    &\leq \left(2(F_0 - F^*) + \beta \sigma^2 + \frac{2^{\frac{1}{h}+1}}{2^{\frac{1}{h}} - 1} LB r_0 \right)\frac{1}{\sqrt{N K}} + \frac{4\beta^2 \big( (N - 1) I\sigma^2 + 3N I^2\epsilon^2 \big)}{K} \nonumber \\
    & = \mathcal{O} \left(\frac{1}{\sqrt{N K}} \right) + \mathcal{O} \left( \frac{1}{K} \right) \label{eq:convergence_dis3} \,\,.
\end{align}
Thus, when the additional assumptions (i) $K \geq 24 N I^2 \beta^2$, (ii) $\eta = \sqrt{\frac{N}{K}}$, (iii) $p_n = \frac{1}{N}$, $\forall n$, (iv) $ \Vert \mathbf{w}(k) \Vert \leq B$, $\forall k$, and (v) the fraction of non-zero prunable parameters decreases exponentially all hold, we have a convergence bound provided in (\ref{eq:convergence_dis3}) that is dominated by $\mathcal{O}\left(\frac{1}{\sqrt{N K}}\right)$. This means using more clients can accelerate the convergence (by a factor of $\frac{1}{\sqrt{N}}$).

\subsection{PruneFL details}\label{appendix:implementation_details}
\subsubsection{Model architecture details}
The details of the model architectures are listed in Table~\ref{tab:model_arch}.
{\renewcommand{\arraystretch}{1.1}
\begin{table*}[t]
\caption{Model architectures. \vspace{-0.05in}}
\label{tab:model_arch}
\centering
\footnotesize
\vskip 1mm
\begin{tabular}{ccccc}
\hline
Architecture &Conv-2& VGG-11 & ResNet-18 & MobileNetV3-Small\\
\hline \hline
Convolutional & \begin{tabular}[c]{@{}c@{}} $32$, pool,\\$64$, pool\end{tabular} & \begin{tabular}[c]{@{}c@{}}$64$, pool,\\$128$, pool,\\$2\times256$, pool,\\$2\times512$, pool,\\$2\times512$, pool\end{tabular} & \begin{tabular}[c]{@{}c@{}}$64$, pool,\\$2\times[64, 64],$\\$2\times[128, 128]$,\\$2\times[256, 256]$,\\$2\times[512, 512]$\end{tabular} & \begin{tabular}[c]{@{}c@{}}$16$, $16$, $8$, $16$, $16$, $72$, $72$, $24$, $88$, $88$, $24$, $96$, $96$, $24$,\\ $96$, $40$, $240$, $240$, $64$, $240$, $40$, $240$, $240$, $64$, $240$, $40$,\\ $120$, $120$, $32$, $120$, $48$, $144$, $144$, $40$, $144$, $48$, $288$,\\ $288$, $72$, $288$, $96$, $576$, $576$, $144$, $576$, $96$, $576$, $576$, \\$144$, $576$, $96$, $576$ \end{tabular} \\
\hline
Fully-connected & \begin{tabular}[c]{@{}c@{}}$2048$, $62$\\(input: $3136$)\end{tabular} & \begin{tabular}[c]{@{}c@{}}$512$, $512$, $10$\\(input: $512$) \end{tabular} & \begin{tabular}[c]{@{}c@{}}avgpool, $100$\\(input: $512$)\end{tabular} & \begin{tabular}[c]{@{}c@{}} avgpool, $1280$, dropout ($0.2$), $2$  \\(input: $960$)\end{tabular} \\
\hline
Conv/FC/all params & 52.1K/6.6M/6.6M & 9.2M/530.4K/9.8M & 11.2M/102.6K/11.3M & 927.0K/592.9K/1.5M\\

\hline
\end{tabular}
\end{table*}
}
\subsubsection{Gradient Computation}\label{appendix:gradient_comp}
The forward pass in neural networks with sparse matrices is straightforward. Taking an FC layer as an example (convolutional layers are more complex but similar in principle), the input data is multiplied by a sparse weight, and produces a dense output to be passed to the next layer. Let $\mathbf{u}\in\mathbb{R}^{n_\text{in}\times n_\text{out}}$ be its (sparse) weight, $\mathbf{x} \in \mathbb{R}^{N\times n_\text{in}}$ be the (dense) input and $\mathbf{y}\in\mathbb{R}^{N\times n_\text{out}}$ be the (dense) output, where $n_\text{in}$, $n_\text{out}$, and $N$ are the number of input neurons, output neurons of the FC layer, and the mini-batch size for SGD, respectively. Assuming there is no bias, the forward pass is given by $\mathbf{y} = \mathbf{x}\cdot \mathbf{u}\,$. The backward pass is slightly more complex. By simple calculations, the gradient of $\mathbf{u}$ is given by $\mathbf{g}_\mathbf{u} = \mathbf{x}^\mathrm{T} \mathbf{g}_\mathbf{y}$, and the gradient of $\mathbf{x}$ (when required) is given by $\mathbf{g}_\mathbf{x} = \mathbf{g}_\mathbf{y}\mathbf{u}^T$. Here, $\mathbf{g}_\mathbf{y}$ is the (dense) gradient in backpropagation fed by the next layer. 

For the computation of $\mathbf{g}_\mathbf{x}$, since the weight $\mathbf{u}$ is sparse, we can accelerate the computation using our sparse matrix implementation.

The computation of $\mathbf{g}_\mathbf{u}$ is however different. Note that both $\mathbf{x}$ and $\mathbf{g}_\mathbf{y}$ are dense, and thus current implementations (e.g., PyTorch) first compute the dense gradient with $\mathbf{u}$'s dense form that has all zero values included, and then select values from the dense gradient according to $\mathbf{u}$'s sparse pattern. There is currently no better way to accelerate this process as far as we know. Therefore, this implementation does not improve the backward pass's speed of weights' gradient computations. For the above reason, in our implementation we collect the gradients of zero-valued components of $\mathbf{u}$ at the same time with no extra overhead (although those zero-valued components themselves are not updated). This characteristic is useful in our adaptive pruning procedure.

\subsubsection{FLOPs Computation}\label{appendix:flops_comp}
Following the discussion in Section~\ref{appendix:gradient_comp}, we now explain the computation of FLOPs in both forward and backward passes in our models. Using convention in the literature, we consider that one addition and one multiplication each counts as a FLOP~\cite{tang2018flops}. Taking the same notations and assumptions from Section~\ref{appendix:gradient_comp}, the FLOPs for the forward pass is $2 N n_\text{in} n_\text{out} \times d$, where $d$ is the density of this FC layer. In the backward pass, the FLOPs for the gradient computation of weight $\mathbf{u}$ is $2 N n_\text{in} n_\text{out}$ since the computation does not involve sparse matrices, while for the gradient of input $\mathbf{x}$, the FLOPs is $2 N n_\text{in} n_\text{out} \times d$. Therefore, the total FLOPs for the backward pass is $2 N n_\text{in} n_\text{out} \times (1+d)$, and the FLOPs for both forward and backward passes is $2 N n_\text{in} n_\text{out} \times (1+2d)$. FLOPs computation for convolutional layers is similar, so we skip this discussion.

\subsubsection{Starting Reconfiguration Round}
We start the reconfiguration in the initial pruning stage when the \textit{training accuracy} on \textit{local sample data} of the selected client exceeds $1.5$ times the random guess accuracy. There are two advantages: (i) if the task is easy, reconfiguration starts early, which results in early model size decrease that saves training time; and (ii) it also avoids pruning the model too early when the prediction is still close to random guess, meaning the parameter values are still in the random initialization stage. In the further pruning stage, reconfiguration happens periodically with a fixed interval.

\subsubsection{Further Details of Adaptive Pruning}
In the following, we provide detailed information for the adaptive pruning procedure described in Section~\ref{sec:adaptive_pruning}. In both reconfiguration and non-reconfiguration rounds, the importance measure is summed locally after every local update, until the sum is sent to the server in the next reconfiguration round. Note that in non-reconfiguration rounds, only the remaining model parameters are used in computation and exchanged between server and clients. Since the parameter set is fixed in non-reconfiguration rounds, only the \textit{values} of the parameters need to be exchanged between the server and clients, which incurs no extra communication cost. An illustration of the two types of rounds is shown in Fig.~\ref{fig:illustration}.

\begin{figure}[t]
\centering
\subfigure[Non-reconfiguration round]{
    \includegraphics[width=0.35\linewidth]{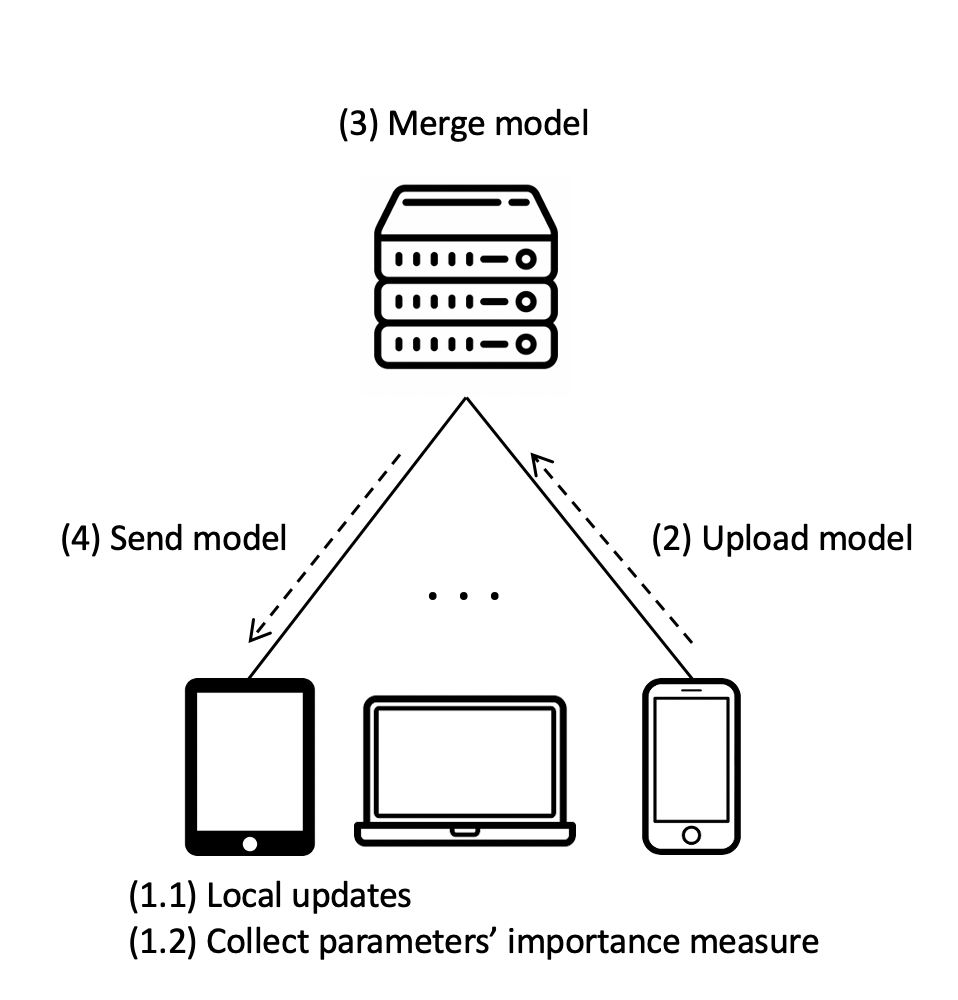}
}
\subfigure[Reconfiguration round]{
    \includegraphics[width=0.35\linewidth]{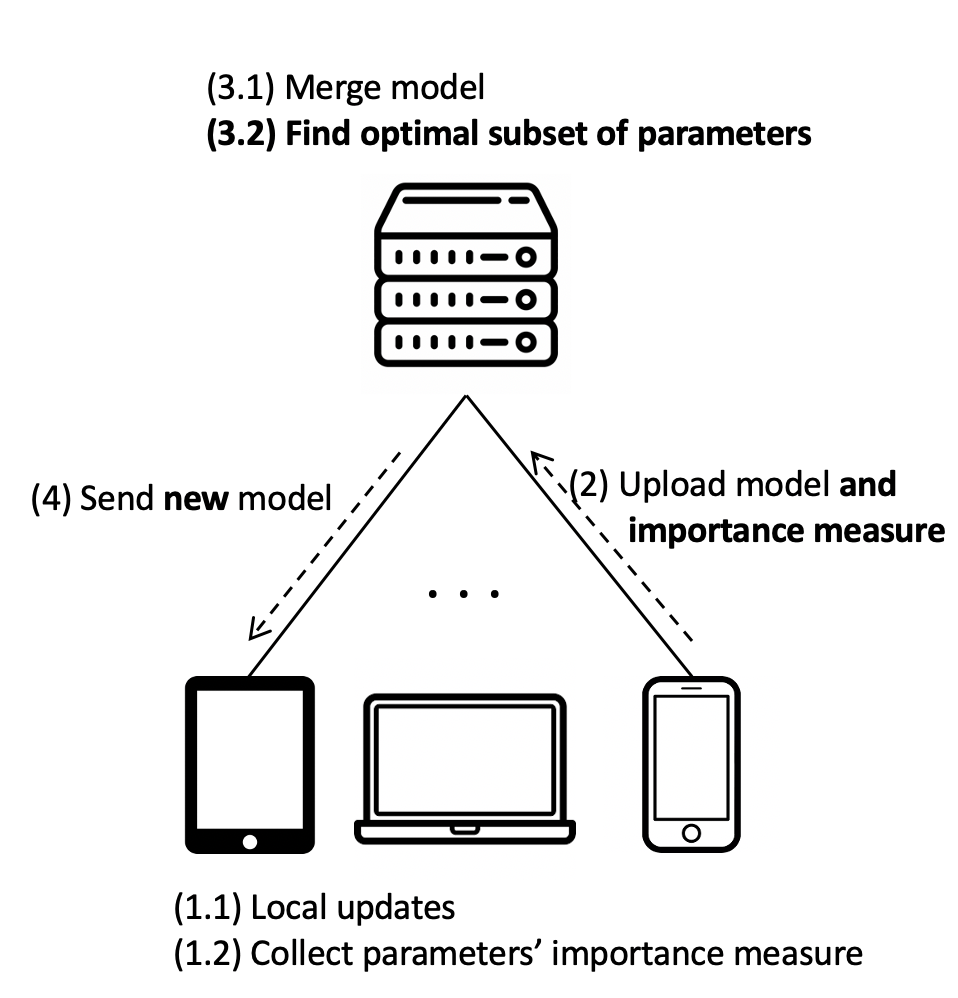}
}
\caption{Illustration of adaptive pruning as part of further pruning during FL.}
\vspace{-3mm}
\label{fig:illustration}
\end{figure}

\subsubsection{PyTorch on Raspberry Pi devices}
To install PyTorch on Raspberry Pi devices, we follow the instructions described at \url{https://bit.ly/3e6I7tG}, where acceleration packages such as MLKDNN and NNPACK are disabled due to possible compatibility issues and their lack of support of sparse computation. We compare our implementation with the \textit{plain} implementation by PyTorch without accelerations. We expect that similar results can be obtained if acceleration packages could support sparse computation. This is an active area of research on its own where methods for efficient sparse computation on both CPU and GPU have been developed in recent years. Integrating such methods into our experiments is left for future work.

\subsection{Additional Experimental Results}
\begin{figure}[t]
\centering
\subfigure[Conv-2's first fully-connected layer]{
\includegraphics[width=0.35\linewidth]{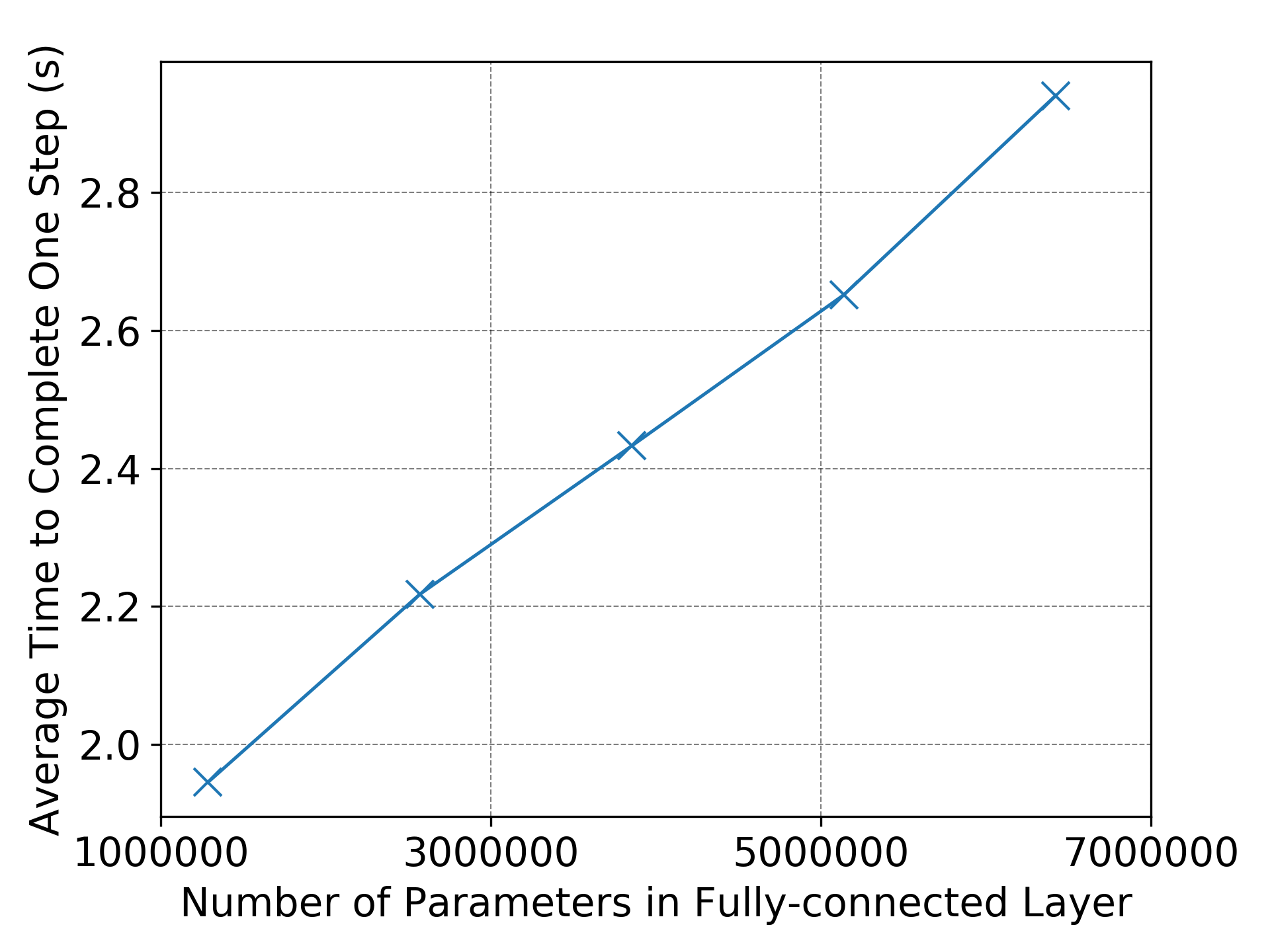}
\label{fig:conv2_linearity}
}
\quad
\subfigure[VGG-11's last convolutional layer]{
\includegraphics[width=0.35\linewidth]{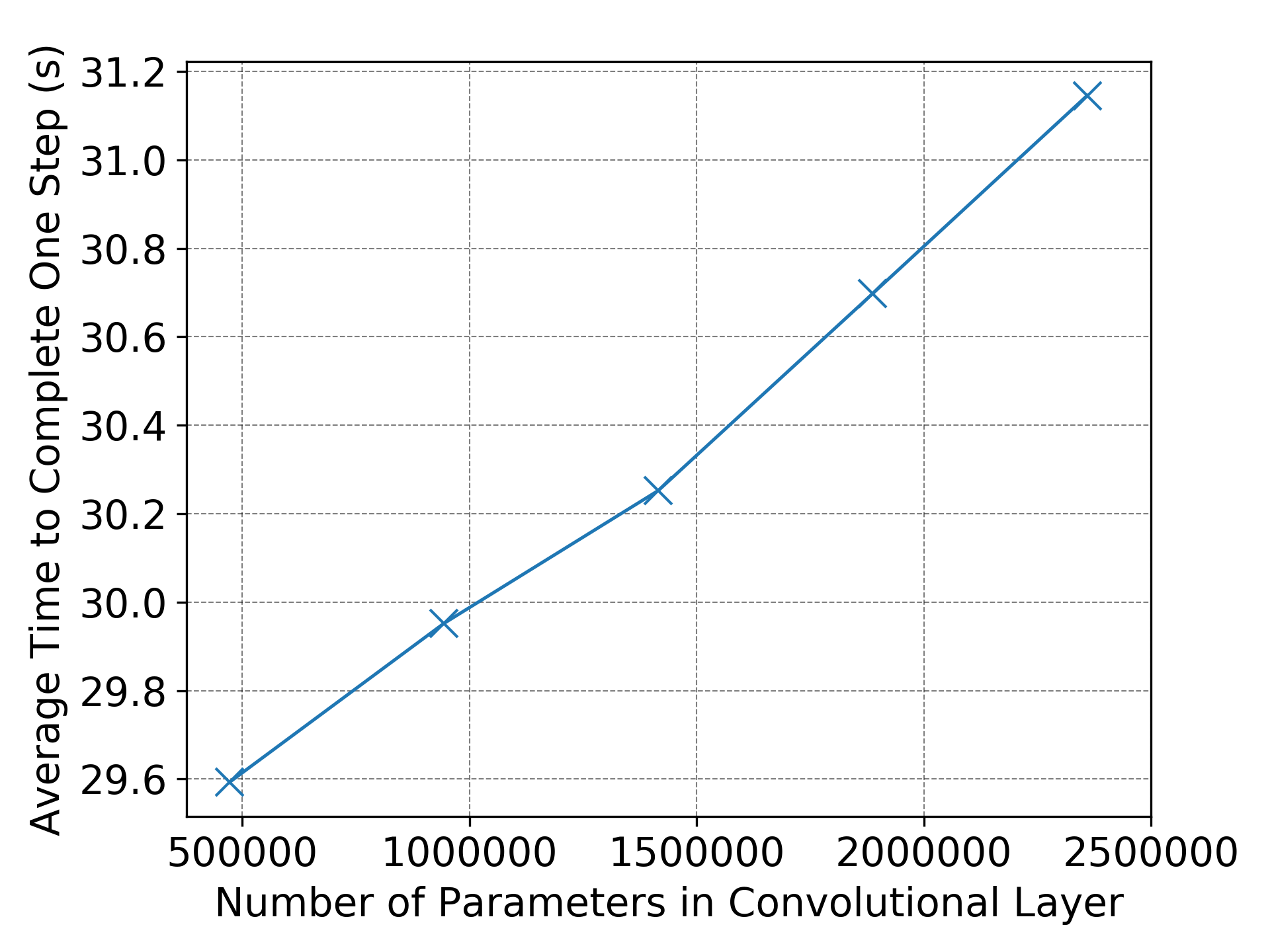}
\label{fig:vgg11_linearity}
}
\caption{Linearity of average computation time vs. number of parameters.}
\label{fig:exp_sim_linearity}
\end{figure}

\subsubsection{Validation of Assumptions}\label{appendix:validate}
Agreeing with our assumptions and analysis in Sections~\ref{sec:adaptive_pruning} and \ref{sec:complexity_analysis}, we observe that the training time in each layer is generally independent of the other layers. Within each layer, the time is approximately linear with the number of parameters in the layer with sparse implementation. In Fig.~\ref{fig:exp_sim_linearity}, we fix the parameters in other layers and increase the number of parameters in Conv-2's largest (first) FC layer and VGG-11's largest convolutional layer (last convolutional layer with $512$ channels), respectively, and measure for 50 times. The $R^2$ values of linear regression for Conv-2 and VGG-11 are $0.997$ and $0.994$, respectively.

\subsubsection{Client Selection Results for Section~\ref{sec:experiments}}\label{appendix:additional_results}
We present results under same settings as in Section~\ref{sec:experiments}, but with random client selection.
We partition the IID CIFAR-10 and ImageNet-100 datasets uniformly into $100$ equal-sized, non-overlapping clients. FEMNIST and CelebA are intrinsically non-IID datasets. For the FEMNIST dataset, we use its original $193$-user partition; and for the CelebA dataset, since some of its users in the original partition have too few images (e.g. $4$ images), we merge the $9343$ persons' images into $934$ clients (first $933$ clients have $10$ persons' images and the last client has $13$ persons' images). In each round, we sample $10$ clients randomly from the aforementioned partitions when using client selection.
Fig.~\ref{fig:training_cs} (corresponding to Fig.~\ref{fig:training}) shows the training time reduction; Fig.~\ref{fig:lottery_cs} (corresponding to Fig.~\ref{fig:lottery}) shows the lottery ticket result; and Fig.~\ref{fig:model_size_cs} (corresponding to Fig.~\ref{fig:model_size}) shows the model size adaptation. We observe similar behaviors as with full client participation that is described in the main paper, thus we omit further discussion in this section.
\begin{figure*}[t]
\centering
\subfigure{
\hspace*{2.5mm}\includegraphics[height=4.9mm]{figs/legend1.png}
}
\addtocounter{subfigure}{-1}
\vspace{-4.5mm}
\\
\subfigure[Conv-2 on FEMNIST]{
    \includegraphics[height=3.075cm]{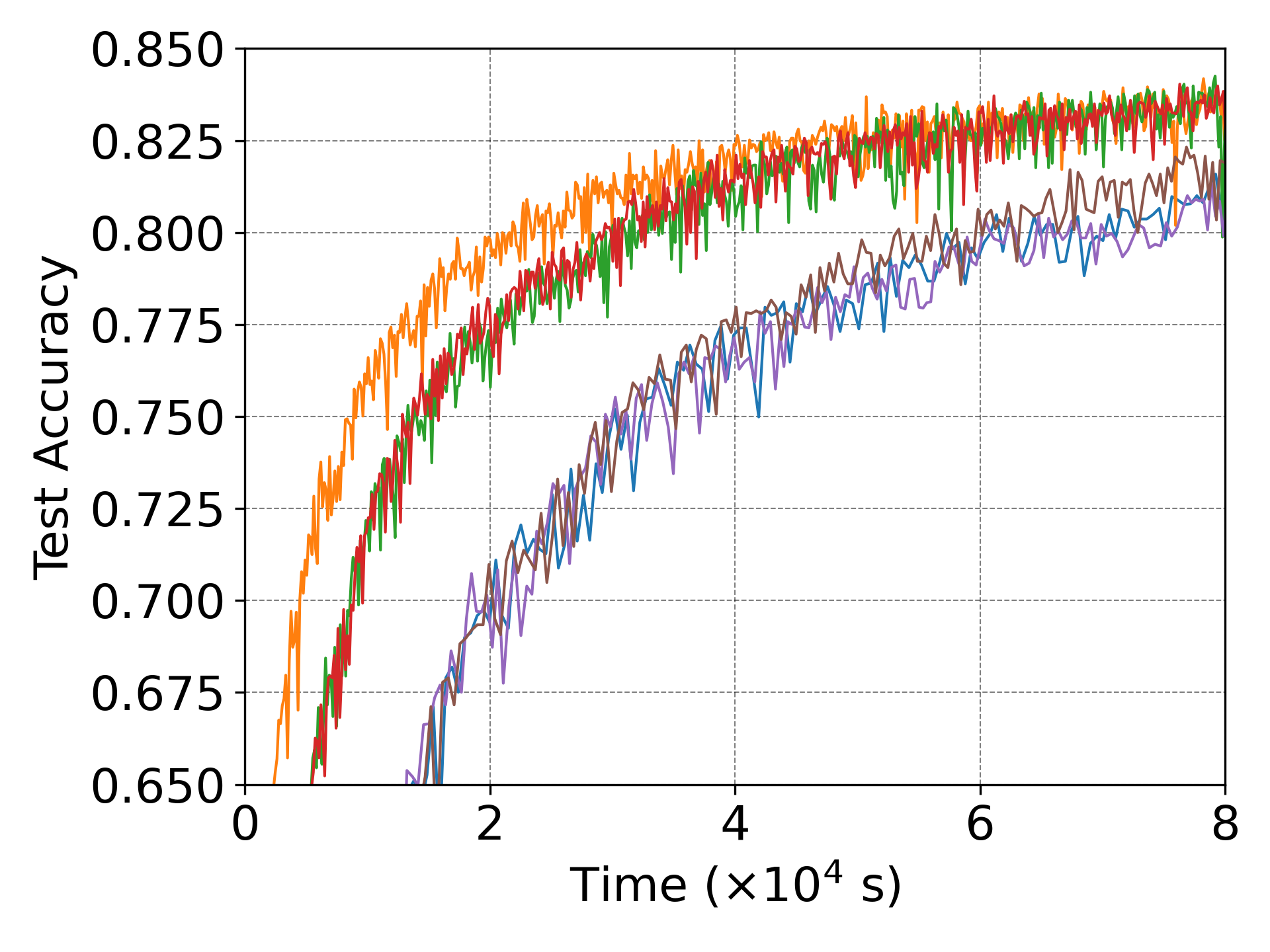}
    \label{fig:femnist_training_cs}
}
\subfigure[VGG-11 on CIFAR-10]{
    \includegraphics[height=3.075cm]{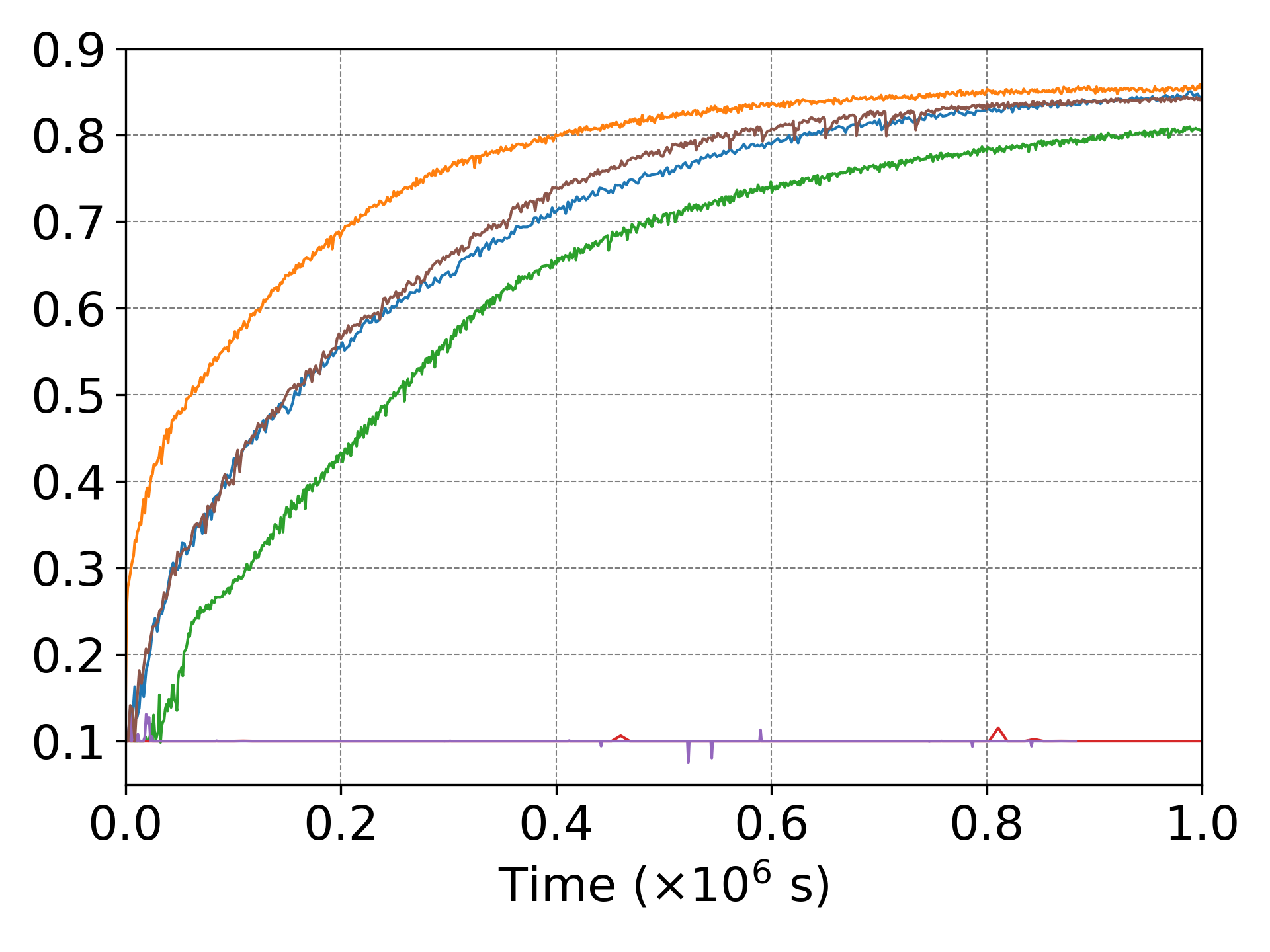}
    \label{fig:cifar10_training_cs}
}
\subfigure[ResNet-18 on ImageNet-100]{
    \includegraphics[height=3.075cm]{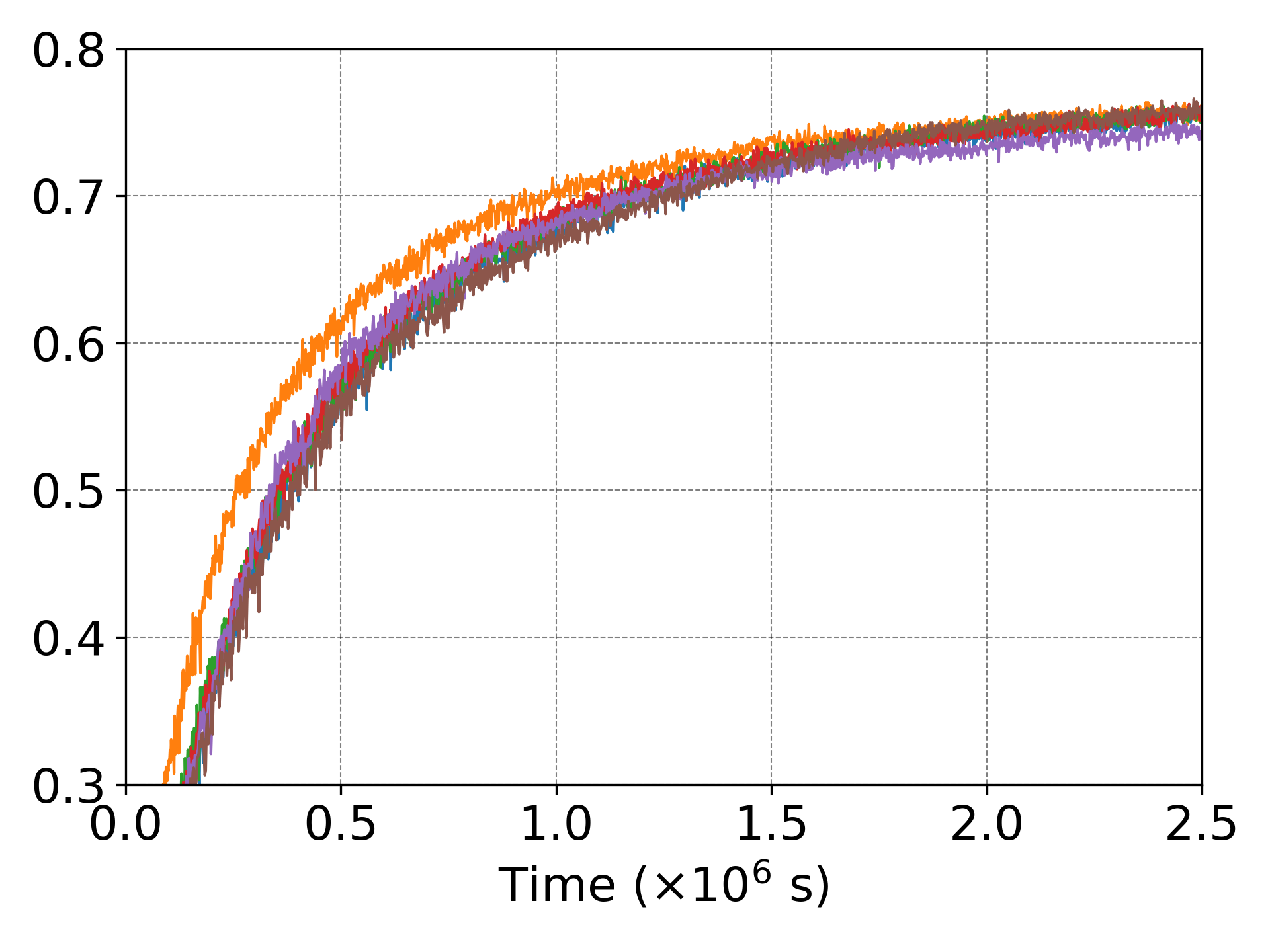}%
    \label{fig:imagenet_training_cs}
}
\subfigure[MobileNetV3-Small on CelebA]{
    \includegraphics[height=3.075cm]{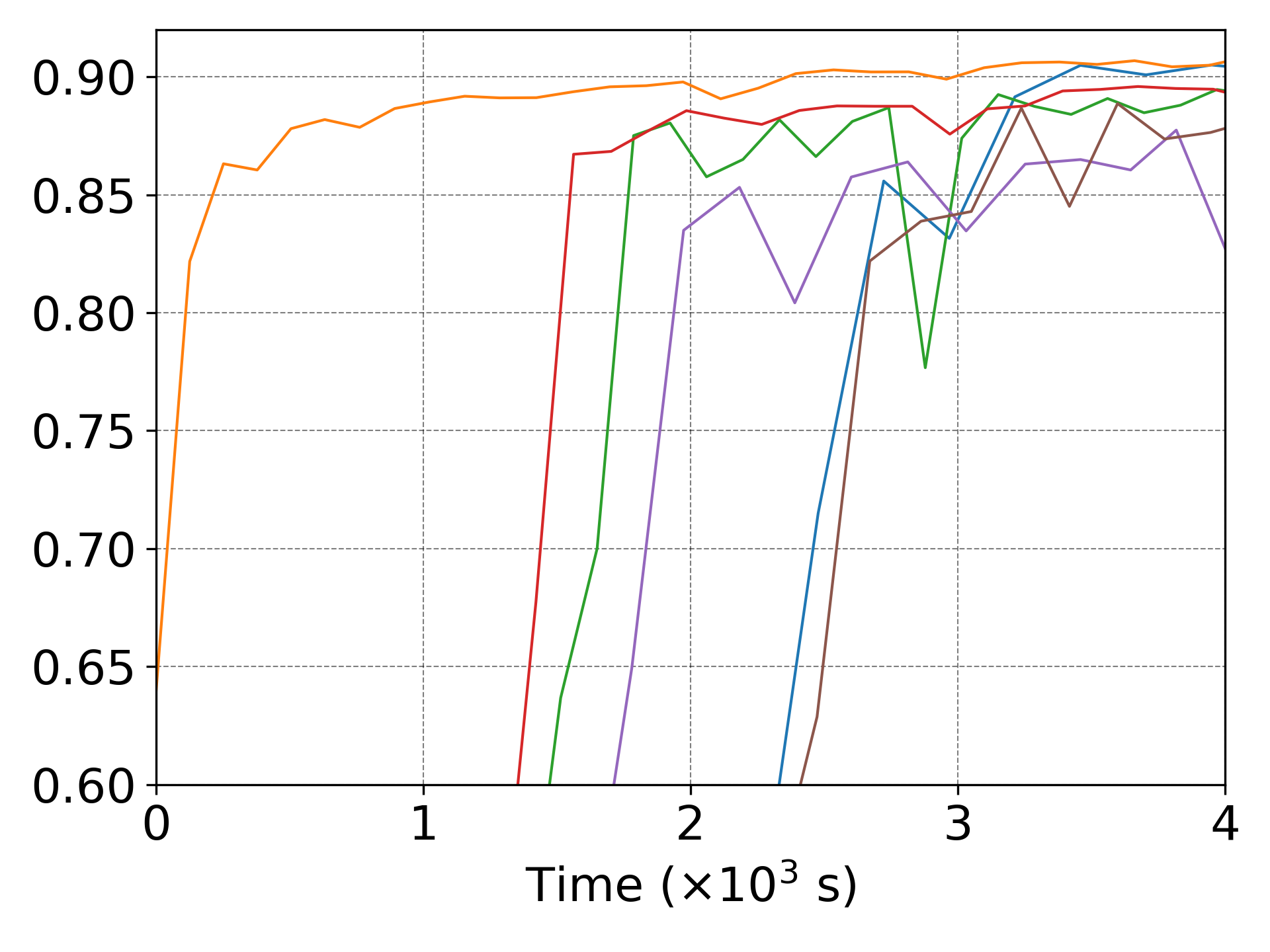}%
    \label{fig:celeba_training_cs}
}
\caption{Test accuracy vs. time results of four datasets (client selection).}
\label{fig:training_cs}
\end{figure*}

\begin{figure*}[t]
\centering
\subfigure{
\hspace*{2.5mm}\includegraphics[height=4.9mm]{figs/legend2.png}
}
\addtocounter{subfigure}{-1}
\vspace{-4.5mm}
\\
\subfigure[Conv-2 on FEMNIST]{
    \includegraphics[height=3.075cm]{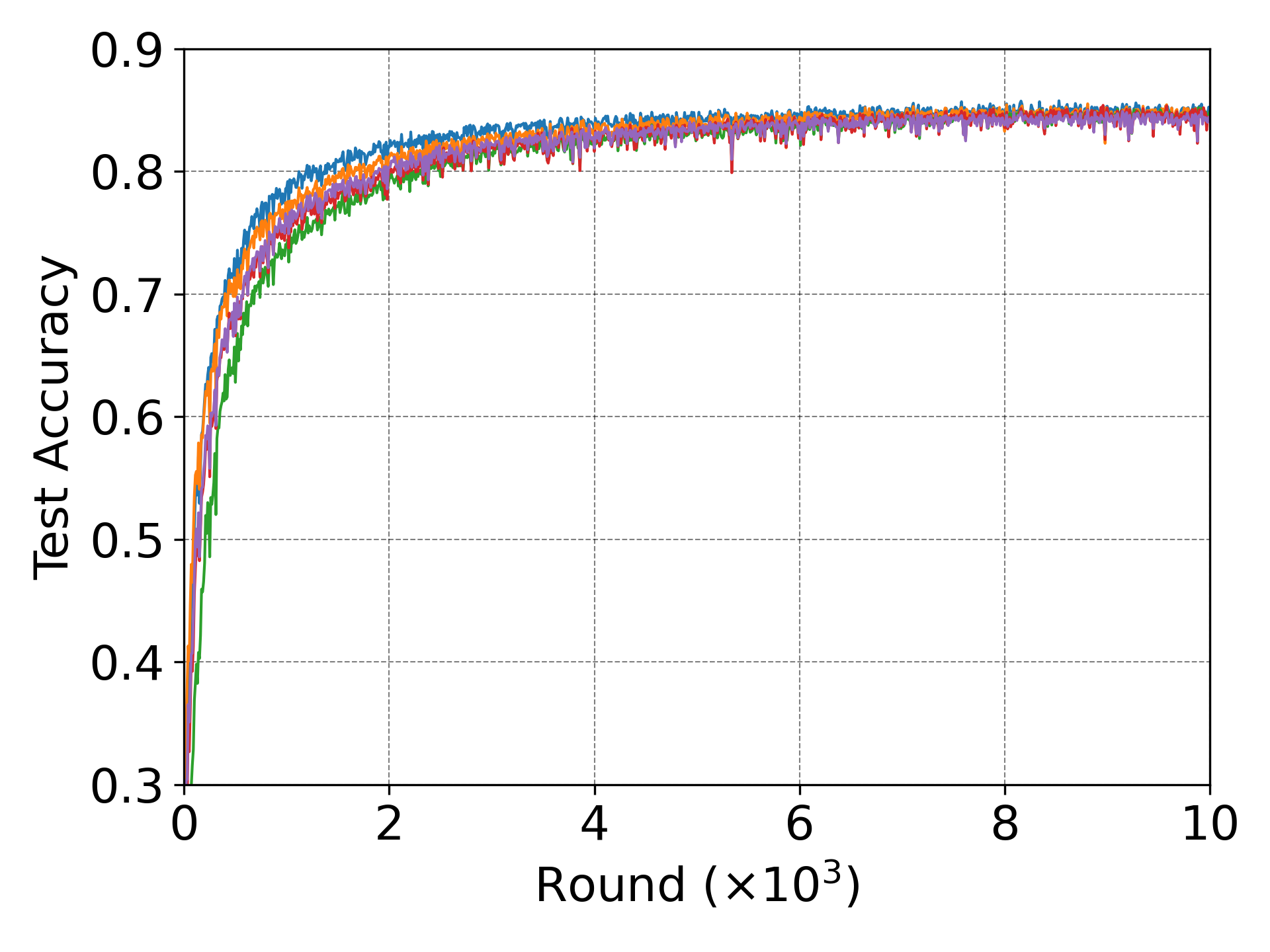}
    \label{fig:femnist_lt_cs}
}
\subfigure[VGG-11 on CIFAR-10]{
    \includegraphics[height=3.075cm]{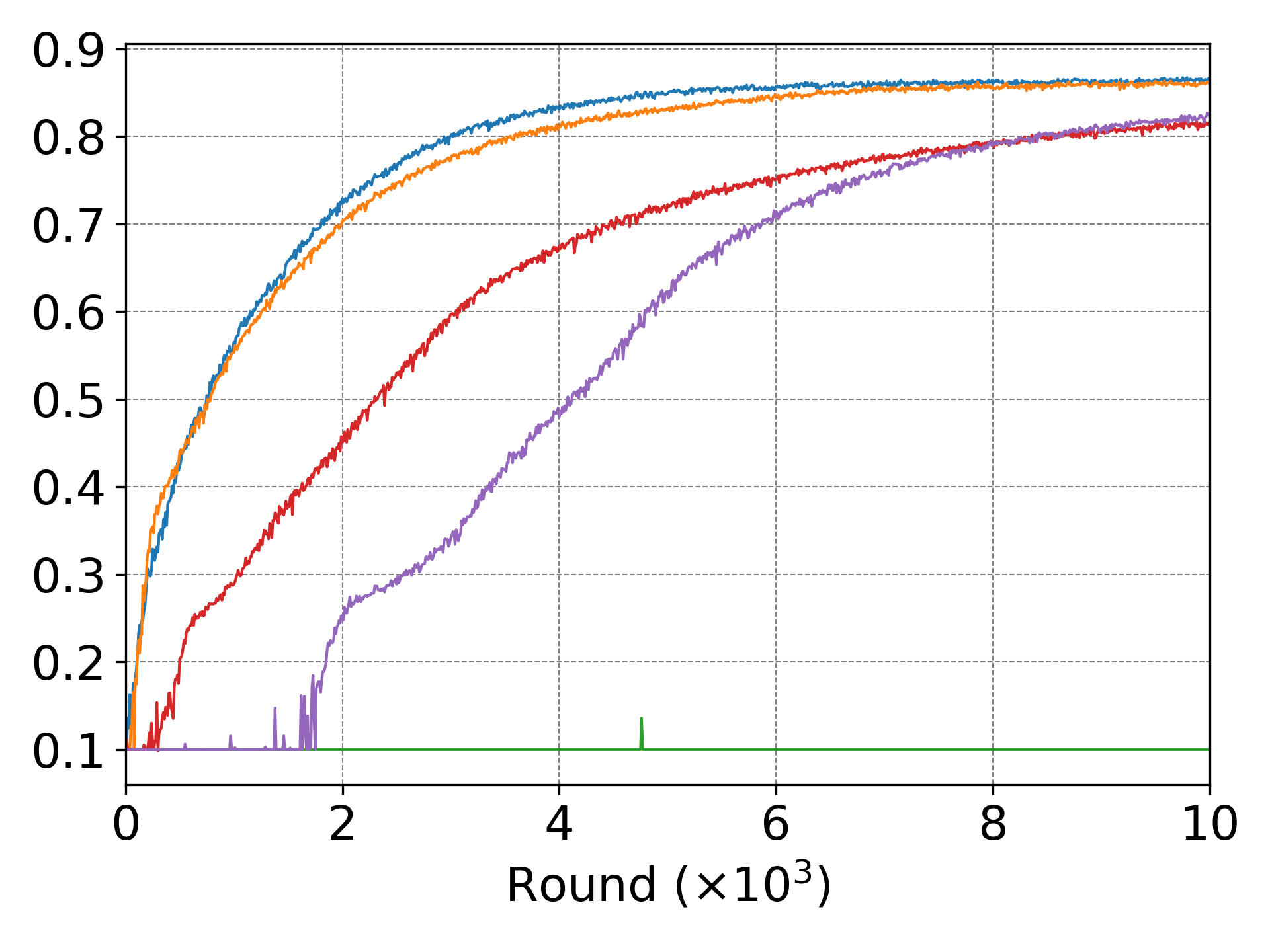}
    \label{fig:cifar10_lt_cs}
}
\subfigure[ResNet-18 on ImageNet-100]{
    \includegraphics[height=3.075cm]{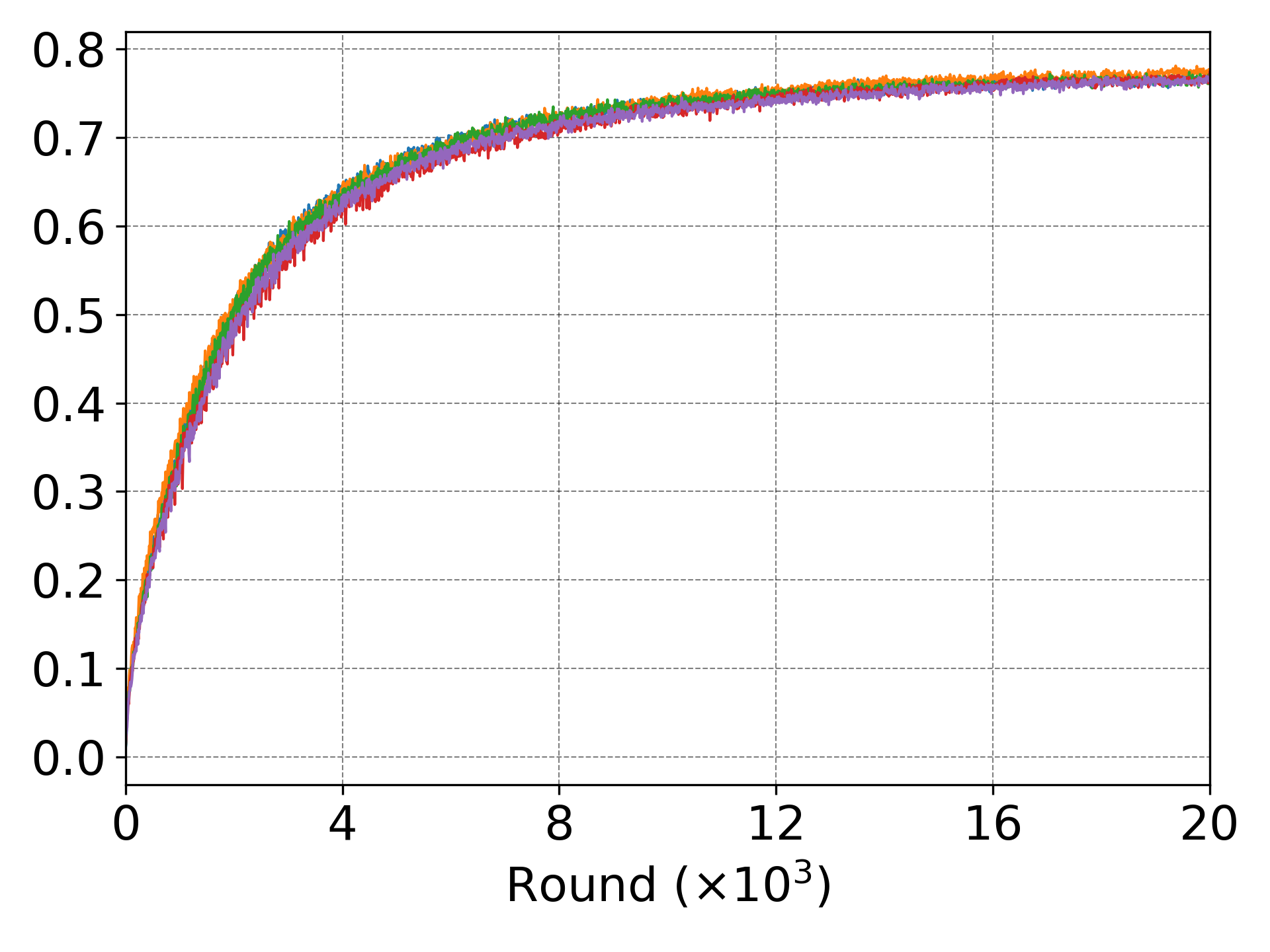}%
    \label{fig:imagenet_lt_cs}
}
\subfigure[MobileNetV3-Small on CelebA]{
    \includegraphics[height=3.075cm]{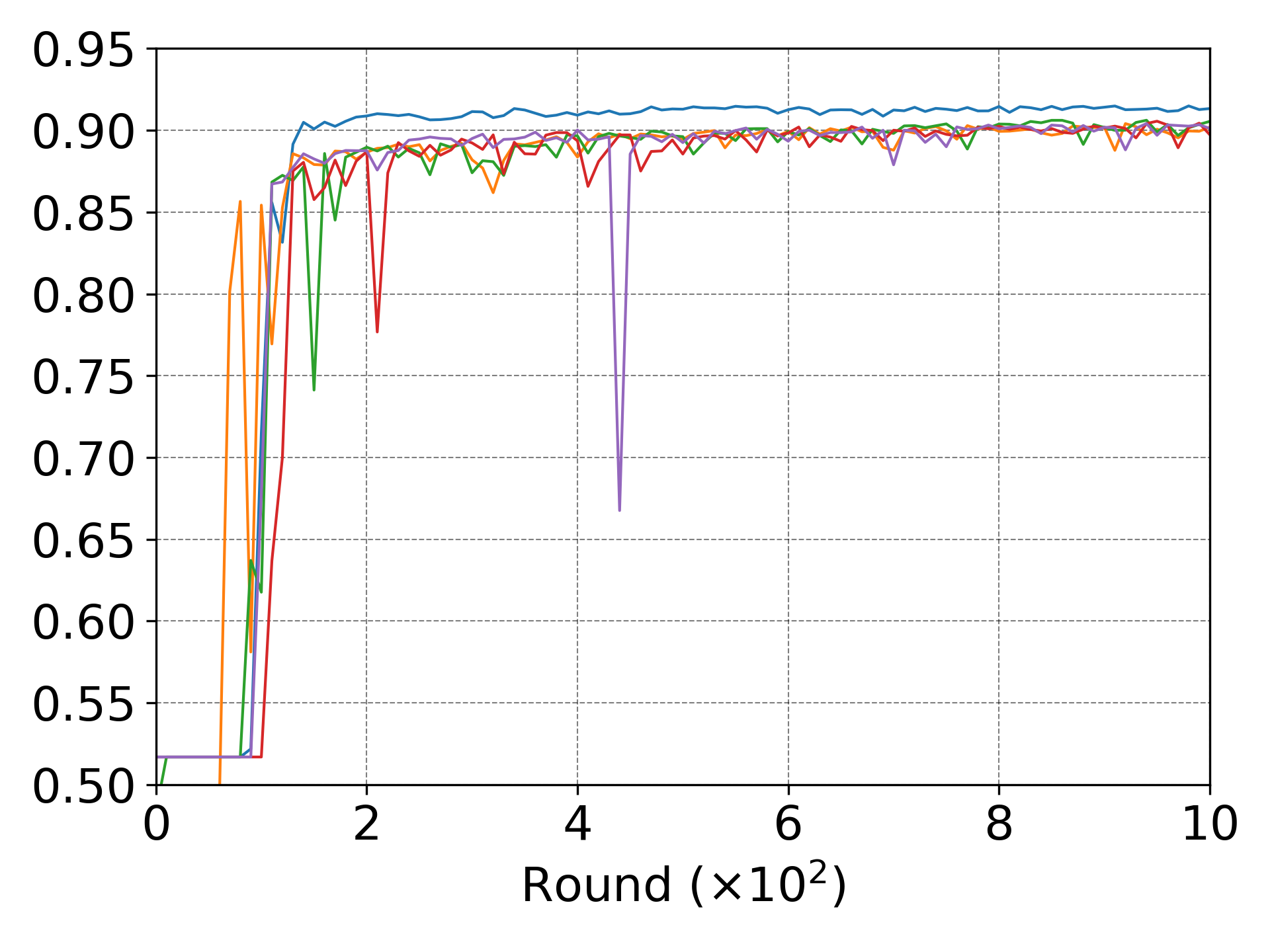}%
    \label{fig:celeba_lt_cs}
}
\caption{Lottery ticket results of four datasets (client selection).}
\label{fig:lottery_cs}
\end{figure*}

\begin{figure*}[t]
\centering
\subfigure{
\hspace*{2.5mm}\includegraphics[height=4.9mm]{figs/legend1.png}
}
\addtocounter{subfigure}{-1}
\vspace{-4.5mm}
\\
\subfigure[Conv-2 on FEMNIST]{
    \includegraphics[height=3.075cm]{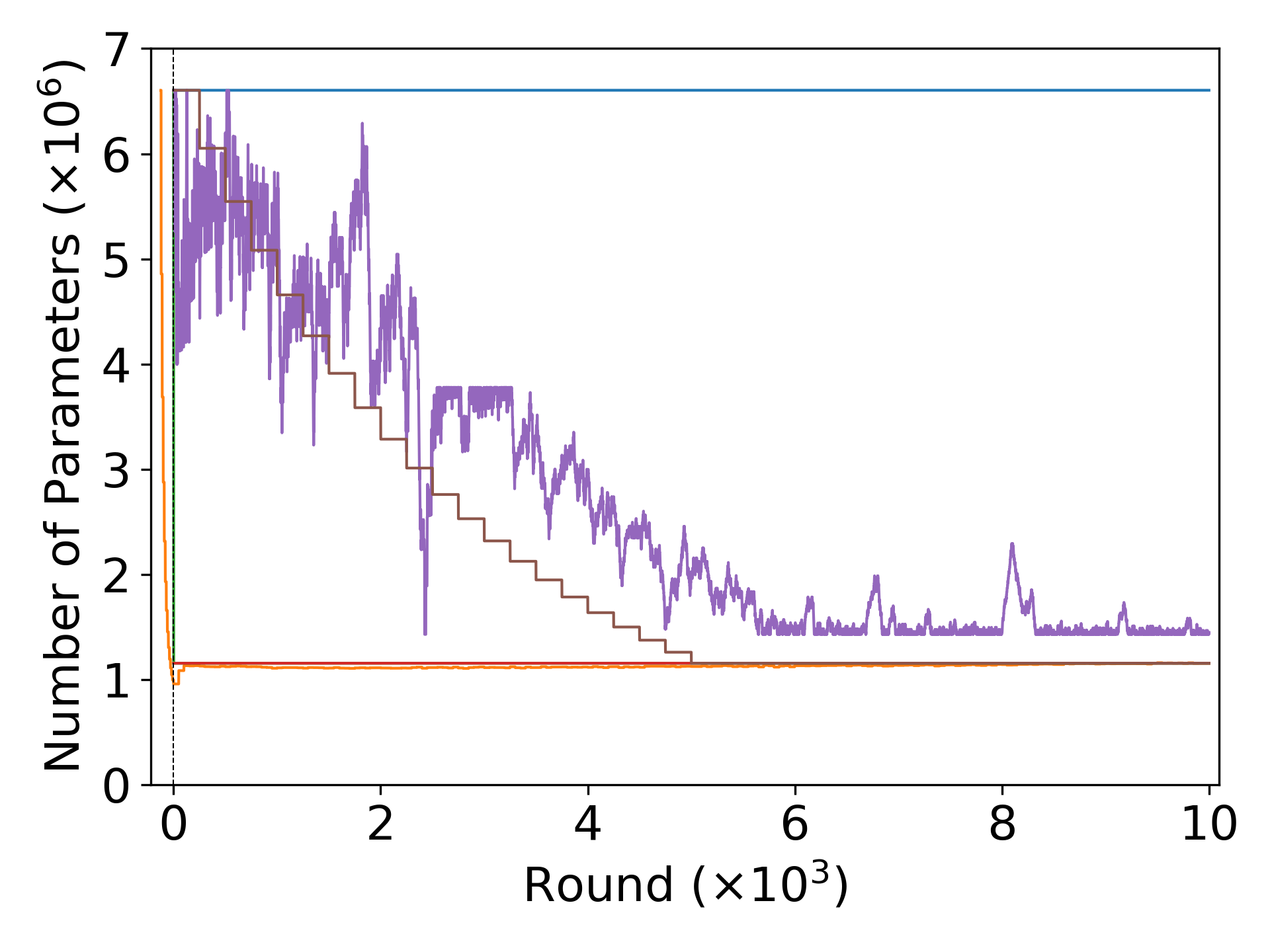}
    \label{fig:femnist_ms_cs}
}
\subfigure[VGG-11 on CIFAR-10]{
    \includegraphics[height=3.075cm]{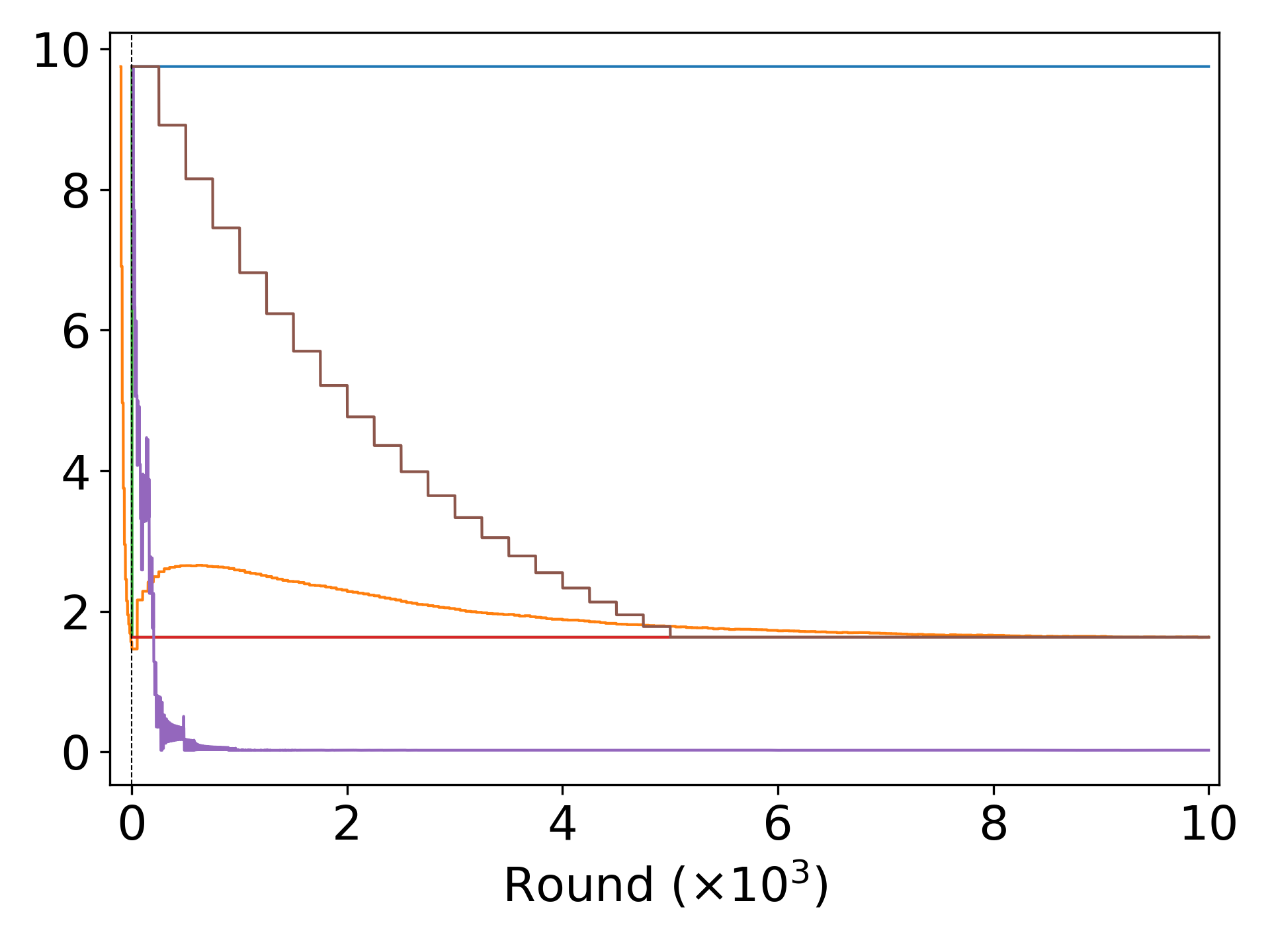}
    \label{fig:cifar10_ms_cs}
}
\subfigure[ResNet-18 on ImageNet-100]{
    \includegraphics[height=3.075cm]{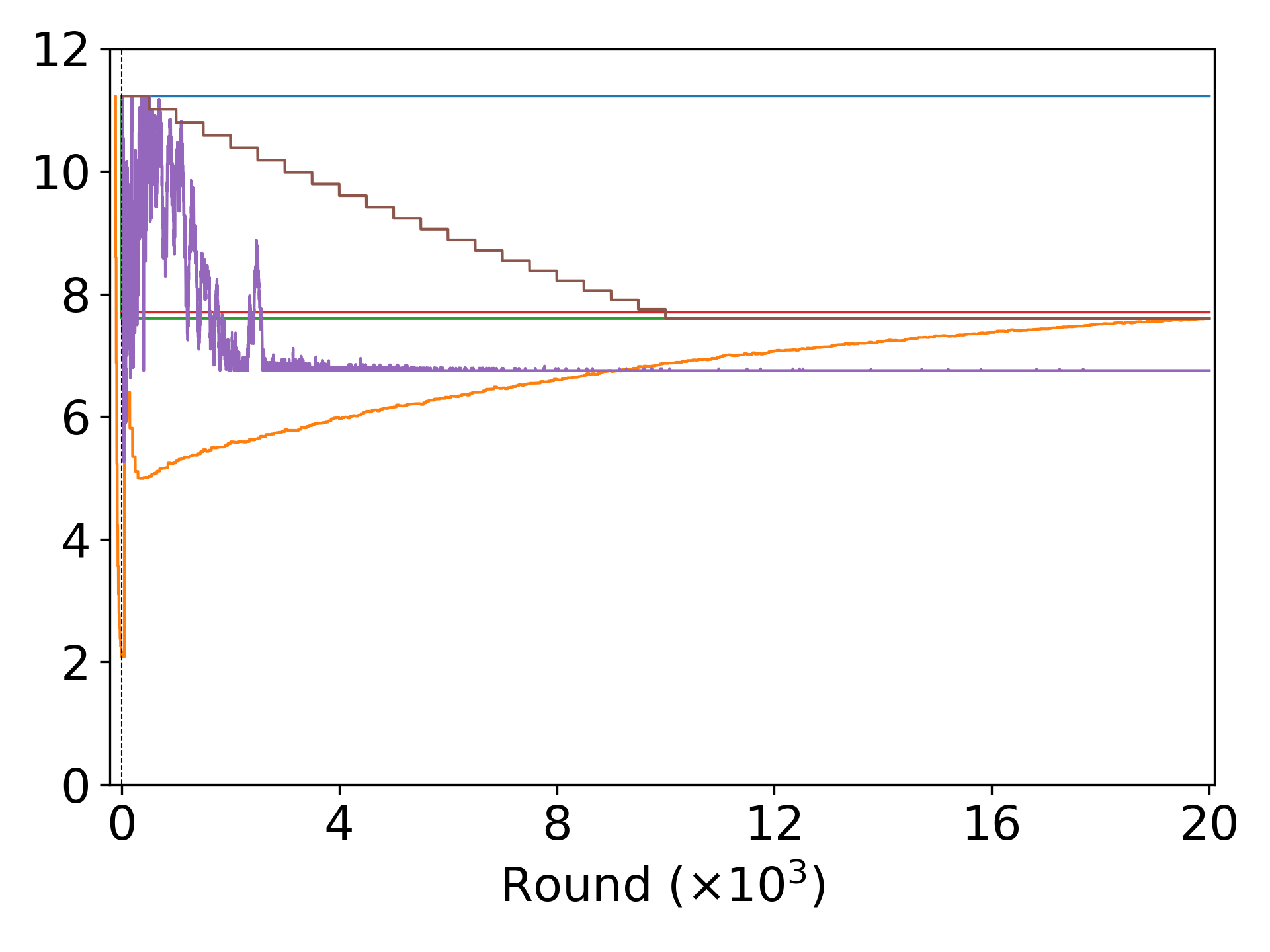}%
    \label{fig:imagenet_ms_cs}
}
\subfigure[MobileNetV3-Small on CelebA]{
    \includegraphics[height=3.075cm]{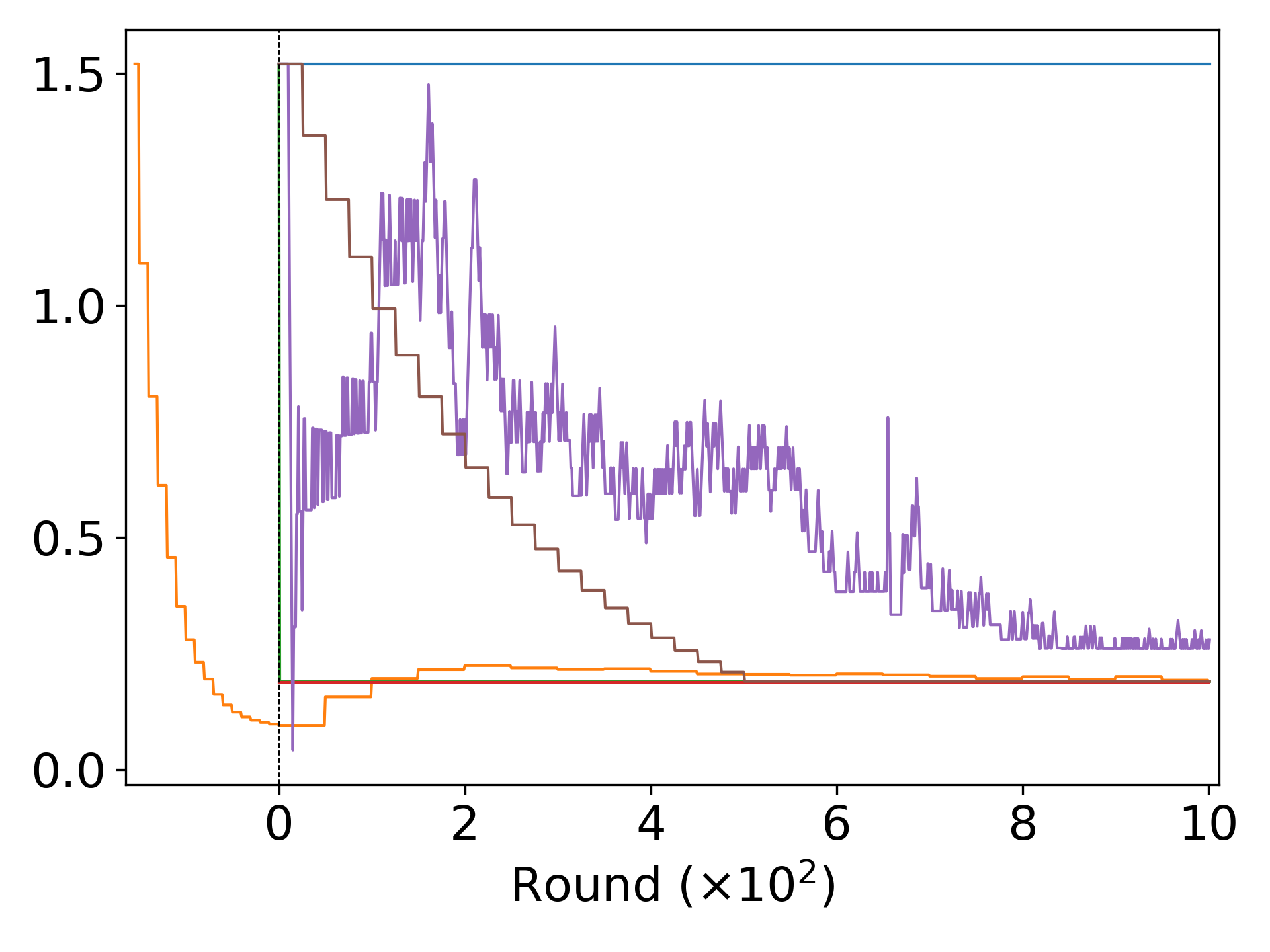}%
    \label{fig:celeba_ms_cs}
}
\caption{Number of parameters vs. round for four datasets (client selection).}
\label{fig:model_size_cs}
\end{figure*}

\subsubsection{Convergence Accuracy Results for All Experiments}\label{appendix:convergence_acc}
The convergence accuracies with/without client selection are shown in Table~\ref{tab:accuracy}. The results are taken from the test accuracies in the last five evaluations of each simulation. We see that the convergence accuracy of PruneFL is similar to that of conventional FL. Conventional FL  sometimes shows  a slight advantage because all methods run for the same number of \textit{rounds}, and full-sized models in conventional FL learn faster when the accuracy is measured in \textit{rounds} instead of time.

{\renewcommand{\arraystretch}{1.1}
\begin{table*}[t]
\caption{Average of the last $5$ measured accuracies (\%). C.S. stands for client selection.}
\label{tab:accuracy}
\centering
\footnotesize
\begin{tabular}{ccccccccc}
\hline
&\multicolumn{2}{c}{FEMNIST}& \multicolumn{2}{c}{CIFAR-10} & \multicolumn{2}{c}{ImageNet-100} & \multicolumn{2}{c}{CelebA} \\
\cline{2-3}
\cline{4-5}
\cline{6-7}
\cline{8-9}
&No C.S. & C.S. & No C.S. & C.S. & No C.S. & C.S. & No C.S. & C.S. \\
\hline \hline
\!\!\!\! \begin{tabular}[c]{@{}c@{}} Conventional \\FL\end{tabular} & $85.33 \pm 0.17$ & $84.61 \pm 0.62$ & $86.30 \pm 0.20$ & $86.51 \pm 0.12$ & $76.36 \pm 0.56$ & $76.62 \pm 0.34$ & $91.41 \pm 0.15$ & $91.33 \pm 0.17$\\ \hline

\begin{tabular}[c]{@{}c@{}} PruneFL\\(ours)\end{tabular} & $85.07 \pm 0.31$ & $83.90 \pm 0.48$ & $85.47 \pm 0.19$ & $85.50 \pm 0.18$ & $77.23 \pm 0.25$ & $76.07 \pm 0.31$ & $91.38 \pm 0.14$ & $91.2 \pm 0.17$\\ \hline

SNIP & $85.16 \pm 0.22$ & $84.05 \pm 0.70$ & $81.16 \pm 0.22$ & $81.44 \pm 0.13$ & $76.89 \pm 0.27$ & $76.76 \pm 0.60$ & $91.48 \pm 0.07$ & $89.69 \pm 0.75$\\ \hline

SynFlow &$84.77 \pm 0.17$ & $84.19 \pm 0.21$ &$82.41 \pm 0.42$ &$82.36 \pm 0.28$ &$76.66 \pm 0.28$ &$76.60 \pm 0.24$ &$91.24 \pm 0.23$ & $90.23 \pm 0.11$ \\ \hline

\begin{tabular}[c]{@{}c@{}} Online\\learning\end{tabular} & $85.31 \pm 0.37$ & $84.30 \pm 0.27$ & $10.00 \pm 0.00$ & $10.00 \pm 0.00$ & $75.32 \pm 0.18$ & $74.97 \pm 0.15$ & $90.39 \pm 0.13$ & $88.61 \pm 0.53$\\ \hline
\!\!\!\! 
\begin{tabular}[c]{@{}c@{}} Iterative\\pruning\end{tabular} & $84.87 \pm 0.22$ & $84.08 \pm 0.39$ & $84.45 \pm 0.14$ & $84.64 \pm 0.14$ & $76.18 \pm 0.30$ & $76.96 \pm 0.26$ & $90.14 \pm 0.37$ & $88.65 \pm 0.33$\\
\hline
\end{tabular}
\end{table*}
}

\end{document}